\documentclass[11pt]{article}

\usepackage[T1]{fontenc}
\usepackage[utf8]{inputenc}

\DeclareUnicodeCharacter{2265}{\ensuremath{\geq}}
\DeclareUnicodeCharacter{2264}{\ensuremath{\leq}}
\DeclareUnicodeCharacter{2208}{\ensuremath{\in}}
\DeclareUnicodeCharacter{2260}{\ensuremath{\neq}}
\DeclareUnicodeCharacter{2248}{\ensuremath{\approx}}

\usepackage{amsmath,amsthm}
\usepackage{mathtools}

\usepackage{XCharter}
\usepackage[xcharter]{newtxmath}

\usepackage{bm}

\usepackage[round]{natbib}
\usepackage{xcolor}
\usepackage{graphicx}
\graphicspath{{../}{../../}{../../../}{./}}
\usepackage{booktabs}
\usepackage{tabularx}
\newcolumntype{Y}{>{\raggedright\arraybackslash}X}
\usepackage{adjustbox}
\usepackage{algorithm}
\usepackage{algorithmic}
\usepackage{subcaption}
\usepackage{multirow}
\usepackage{nicefrac}
\usepackage{enumitem}

\providecommand{\DDCTextWidth}{6.0in}
\providecommand{\DDCTextHeight}{9.0in}
\usepackage[letterpaper]{geometry}
\usepackage{microtype}
\theoremstyle{plain}
\newtheorem{theorem}{Theorem}
\newtheorem{proposition}[theorem]{Proposition}
\newtheorem{lemma}[theorem]{Lemma}
\newtheorem{corollary}[theorem]{Corollary}

\theoremstyle{definition}

\newtheorem{remark}[theorem]{Remark}

\usepackage{fancyhdr}
\usepackage{titlesec}
\titleformat{\section}
  {\normalfont\Large\bfseries\scshape}{\thesection}{0.7em}{}
\titlespacing*{\section}{0pt}{2.4ex plus 1ex minus .2ex}{1.2ex plus .2ex}
\titleformat{\subsection}
  {\normalfont\large\bfseries}{\thesubsection}{0.6em}{}
\titlespacing*{\subsection}{0pt}{1.9ex plus .8ex minus .2ex}{0.8ex plus .2ex}

\definecolor{linknavy}{HTML}{1A4D8F}
\usepackage{hyperref}
\hypersetup{
  colorlinks=true,
  linkcolor=linknavy,
  citecolor=linknavy,
  urlcolor=linknavy,
  pdftitle={Dead-Direction Conditioners: Gauge-Equivariant Preconditioning for Deep Networks},
}
\usepackage{orcidlink}  \usepackage{etoc}       

\usepackage{titling}
\pretitle{\begin{center}\LARGE\bfseries\scshape}
\posttitle{\par\end{center}\vskip 0.45em {\centering\rule{\linewidth}{0.8pt}\par}\vskip 0.7em}
\date{}

\newif\ifanonymise
\anonymisetrue
\newif\ifanonymousfriendly
\anonymousfriendlytrue

\newcommand{\reals}{\mathbb{R}}
\newcommand{\fisher}{F}

\newcommand{\kl}{D_{\mathrm{KL}}}
\newcommand{\eqadam}{\textsc{DDCAdam}}
\DeclareMathOperator{\tr}{tr}
 
\anonymisefalse
\anonymousfriendlyfalse

\title{Dead-Direction Conditioners: Gauge-Equivariant Preconditioning for Deep Networks}

  \author{
    Tejas Pradeep Shirodkar\thanks{Correspondence: \texttt{tejas.shirodkar@research.iiit.ac.in} \quad
      \orcidlink{0009-0001-3034-0087}\,\href{https://orcid.org/0009-0001-3034-0087}{0009-0001-3034-0087}} \\
    IIIT, Hyderabad
  }

\begin{document}

\maketitle

\begin{abstract}
A deep network's loss is invariant to continuous symmetries of its parameters: the logit shift, the ReLU rescaling, the LayerNorm scale, the per-head attention rotation. Adam's per-coordinate preconditioner drifts along each symmetry orbit, which pulls the trajectory off the symmetry quotient where the optimization lives and blurs the singular-learning rate the quotient makes readable. We build DDC, a Dead-Direction Conditioner that lifts a base optimizer into a $G$-equivariant one: it conditions the optimizer's state in the orbit decomposition of a $G$-invariant metric, so the trajectory stays a preconditioned gradient flow on the quotient $\bar\Theta = \Theta/G$. The construction carries four architectural gauges (cross-entropy shift, ReLU and SwiGLU rescaling, LayerNorm and RMSNorm scale, and a per-head $O(d_{\rm head})$ attention rotation matched to RoPE), proves exactly equivariant on an Adam base, and composes with a Muon base through a gauge-equivariant orthogonaliser. Respecting the symmetry changes both the minimum the optimizer reaches and what it leaves measurable there. On a language model trained past the point of fit, \eqadam{} resists the over-training collapse AdamW falls into, holding a validation-train loss gap of $0.67$ against $5.88$, and reads the dead-direction rate in 32 of 65 layer-by-observable cells where AdamW reads it in 7. A vision transformer trained from scratch reaches lower validation loss ($1.71$ against $2.12$) while compressing spare feed-forward capacity a matched AdamW leaves intact. On a Muon base, where the rotation gauge composes exactly, \textsc{DDCMuon} groks ten of eleven seeds at depth 24 that a plain Muon never reaches. Built into the optimizer, a network's gauge symmetry sharpens the minimum it finds and turns that minimum's geometry into something the trajectory can measure.
 \end{abstract}

\etocsettocdepth.toc{none}

\section{Introduction}
\label{sec:intro}

\begin{figure}[t]
\centering
  \includegraphics[width=0.98\linewidth]{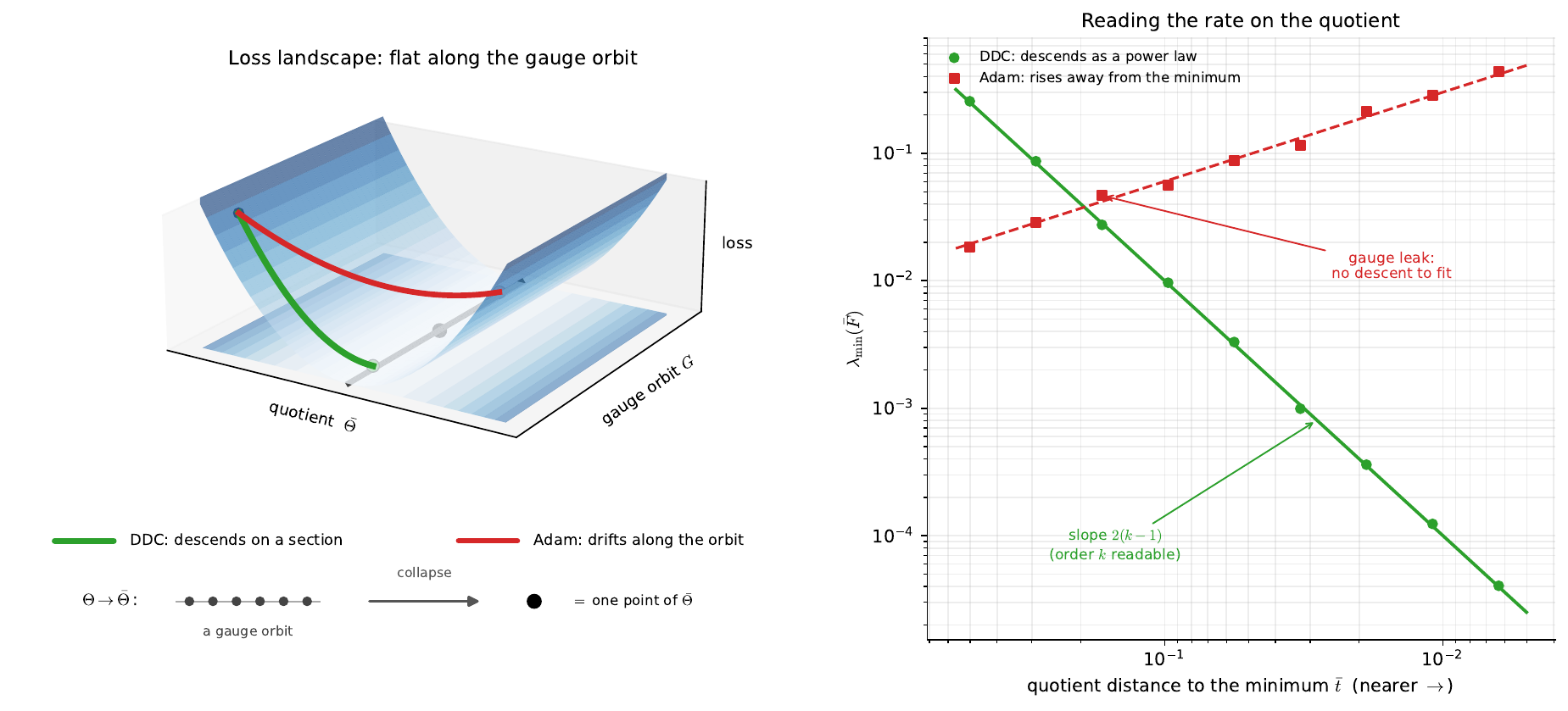}

\caption{The construction in one picture. Left: the loss is curved across the quotient $\bar\Theta = \Theta/G$ but flat along each gauge orbit $G$, so its minima form the valley floor, the orbit itself, a flat direction the Fisher metric cannot see (a dead direction). Adam descends the wall but drifts along that flat floor as the gauge mode leaks; \eqadam{} holds the trajectory on a fixed section, straight to the floor. The strip below is the quotient construction: a whole gauge orbit collapses to a single point of $\bar\Theta$, the space on which the rate is read. Right: the leak is what blurs the singular-learning rate read off the trajectory. Approaching the minimum, $\lambda_{\min}(\bar\fisher)$ follows the predicted power law $\bar t^{\,2(k-1)}$ under DDC, a straight line on log axes whose slope gives the order $k$; under Adam it rises away from the minimum with no slope to fit. The schematic states the mechanism; \S\ref{ssec:rate_readability} reads it on a depth-12 language model.}
\label{fig:overview}
\end{figure}
 
A deep network's loss is blind to its own architectural symmetries. Shift every logit by a per-example constant, rescale a ReLU pair by inverse factors, rescale a LayerNorm channel against its downstream weight, or rotate an attention head's $QK$ basis, and the function the network computes does not move. Each symmetry is continuous, so it traces an orbit through parameter space along which the loss stays flat (the parameter-space symmetries surveyed by \citealp{ZhaoWaltersYu25}); the tangent to such an orbit is a \emph{dead direction}, flat to second order ($u^\top \fisher u = 0$ for the Fisher metric $\fisher$). A model's complexity near a singular minimum is set by its local learning coefficient \citep{Watanabe09, LauFurmanWangMurfetWei25, WeiMurfet22}, and a recent rate theorem makes that coefficient trajectory-readable: for a clean optimizer descending the symmetry quotient $\bar\Theta = \Theta/G$, where $G$ is the Lie group of the loss symmetry, the smallest quotient-Fisher eigenvalue decays as $\bar t^{2(k-1)}$ in the singular order $k$, so the order falls out of the trajectory with no Hironaka resolution (\S\ref{ssec:quotient_rate})~\citep{GSDTheoryAnon}. An adaptive optimizer drifts off that quotient: Adam's per-coordinate $1/\sqrt{\hat v}$ preconditioner commutes with $G$ only when $G$ permutes coordinates, and the architectural groups mix them, so its update leaks into the orbit and the projection picks up a bias that both moves the minimum the optimizer settles into and blurs the rate (\S\ref{ssec:adam_breaks}).

We build DDC, a Dead-Direction Conditioner that makes an adaptive optimizer respect these symmetries. It replaces the per-coordinate second moment with a split taken in the orbit decomposition of a $G$-invariant metric: per-coordinate adaptivity in a $G$-equivariant horizontal frame, and the vertical (orbit) component collapsed to one scalar per orbit dimension. The split is independent of the base optimizer, so the same components lift any preconditioner; we instantiate \eqadam{} and \textsc{DDCMuon} here. Theorem~\ref{thm:ddcadam_rate} shows the result descends a preconditioned gradient flow on $\bar\Theta$ and inherits the trajectory-rate readability of \citet{GSDTheoryAnon}, with discrete-step bias bounded by $O(\eta^2\kappa(P))$.

The construction covers four architectural gauges. Three are abelian, their group operation commutative: the cross-entropy shift (a per-logit translation), the ReLU and SwiGLU rescaling, and the LayerNorm and RMSNorm scale, each a one-dimensional scaling or translation per orbit dimension. A per-channel scalar second moment conditions them exactly, verified $G$-equivariant to fp32 precision in code. The per-head $O(d_{\rm head})$ rotation of attention QK and VO weights is non-abelian: two rotations of a head compose in an order-dependent way and mix the head's coordinates, so a standard-basis per-coordinate $\hat v$ picks up basis-dependent components. We characterize a spectrum of $\hat v$ constructions, from one scalar per head (strictly equivariant) through a per-head PSD matrix to per-coordinate adaptivity in the body frame, the eigenbasis of $W^\top W$; the default \texttt{body\_frame\_topk} keeps full adaptivity on the live eigendirections and conditions through the dead ones.

How far the family reaches depends on how the base optimizer meets each gauge. Muon~\citep{JordanBernsteinKaramcheti24} orthogonalizes each momentum update, and a rotated weight orthogonalizes to the rotation of the orthogonalized weight, so \textsc{DDCMuon} carries the rotation gauges exactly. Orthogonalization removes the per-channel scale outright, so the rescaling gauges have nothing left to condition on a Muon base and ride the Adam base instead.

The gauge an optimizer respects decides the minimum it reaches and what it leaves measurable there. On a depth-12 language model trained into the over-training regime, \eqadam{} resists the over-training collapse AdamW falls into and reads the singular-approach rate that AdamW cannot. On a vision transformer trained from scratch it reaches lower validation loss and compresses spare feed-forward capacity under weight decay where a matched AdamW leaves it intact. On the grokking transformer the construction holds the gauge mode orders of magnitude below AdamW, and on a Muon base it groks ten of eleven seeds at depth 24 where a plain Muon groks none, settling into a measurably less degenerate minimum; matching a torus gauge to a rotary model's geometry groks faster at the same weight decay.

\paragraph{Scope.} A closed-form rate under standard non-equivariant Adam (\S\ref{sec:discussion}) and the multi-direction $(k_i, h_i)$ generalization of the rate theorem~\citep{GSDTheoryAnon} remain open. Reading the order and multiplicity of a formed dead direction off canonical alignment needs machinery we develop separately.
 \section{Background}
\label{sec:setup}

\subsection{Quotient Fisher Rate}
\label{ssec:quotient_rate}

Let $\Theta^\circ \subset \reals^p$ be the parameter open set, $\{p_\theta\}$ a statistical model, $L : \Theta^\circ \to \reals$ a loss whose minimum coincides with $p_\theta = p^*$ for some target distribution $p^*$. Let $G$ be a Lie group acting smoothly, freely, and properly on $\Theta^\circ$ with $L$ and $\{p_\theta\}$ both $G$-invariant. The quotient $\bar\Theta := \Theta^\circ / G$ is a smooth manifold of dimension $\dim \Theta - \dim G$, with quotient submersion $\pi : \Theta^\circ \to \bar\Theta$. The group $G$ relabels parameters that compute the same function. For the ReLU pair $(W_\ell, W_{\ell+1}) \mapsto (cW_\ell, c^{-1}W_{\ell+1})$, the product $\|W_\ell\|\,\|W_{\ell+1}\|$ moves the network while the ratio $\|W_\ell\| / \|W_{\ell+1}\|$ is pure gauge freedom; $\bar\Theta$ keeps the product and discards the ratio, carrying one point per distinct function.

For a horizontal vector $u \in T_{\theta_0} \Theta^\circ$ at a singular minimum $\theta_0$ (i.e., $p_{\theta_0} = p^*$), the \emph{KL order} along $u$ is the smallest integer $k \ge 2$ such that
\[
\partial_t^k \bigl[ \kl(p^* \| p_{\theta_0 + tu}) \bigr]\bigl|_{t=0} \;\ne\; 0.
\]
Order-$k$ singular minima generalise the classical $k=2$ regular case: the KL divergence vanishes to higher order along horizontal directions, so the Fisher metric degenerates faster than quadratically. A gauge symmetry is the limiting $k = \infty$ case, where the Fisher vanishes identically along the orbit everywhere, not only at the minimum, and so carries no rate. The quotient divides these gauge directions out, leaving the finite-order singularity to govern the rate; gauge and singularity are the same dead-direction primitive at opposite ends of the order $k$ \citep{GSDTheoryAnon}.

\citet{GSDTheoryAnon} establish the following quotient rate. For a $G$-equivariant horizontal distribution $H_\theta$ chosen via a $G$-invariant Riemannian metric $\rho$ on $\Theta^\circ$, and any horizontal direction $\bar u \in T_{\bar\theta_0} \bar\Theta$ of KL order $k$,
\[
\bar u^\top \bar\fisher(\bar\theta_0 + \bar t\, \bar u)\, \bar u \;=\; \Theta\bigl(\bar t^{2(k-1)}\bigr), \qquad \lambda_{\min}\bigl(\bar\fisher(\bar\theta_0 + \bar t\, \bar u)\bigr) \;=\; \Theta\bigl(\bar t^{2(k-1)}\bigr),
\]
where $\bar\fisher$ is the quotient Fisher information metric. Under continuous-time SGD on a $G$-invariant metric, the projected trajectory realises this rate along its canonical-aligned approach (the approach direction lies along the local resolution coordinates in which the singularity is monomial, so a single order $k$ governs the decay; an off-canonical approach mixes orders and blurs the exponent). We call this property \emph{rate-readability}: the exponent $k$ is recoverable from the trajectory of $\lambda_{\min}(\bar\fisher)$ versus the quotient arc-length to $\bar\theta_0$.

\subsection{Where Adam breaks the projection}
\label{ssec:adam_breaks}

Adam's update rule~\citep{KingmaBa15} preconditions the gradient by a per-coordinate inverse second-moment estimator:
\[
\theta_{t+1} = \theta_t - \eta \cdot \frac{\hat m_t}{\sqrt{\hat v_t} + \epsilon}, \qquad \hat v_t \in \reals^p \text{ tracked per coordinate.}
\]
The preconditioner $P_{\rm Adam}(\theta, g) = -\eta\, g \,/\, (\sqrt{\hat v(\theta, g)} + \epsilon)$ is $G$-equivariant if and only if the action of $G$ on $T_\theta \Theta^\circ$ is by coordinate permutations: any non-permutation action makes $\sqrt{\hat v}$ vary across coordinates that share an orbit, breaking the equivariance condition $(dh)_\theta P(\theta, g) = P(h\cdot\theta, (dh^{-1})^* g)$.

The architectural Lie groups in deep networks act non-trivially on coordinates:
\begin{itemize}\itemsep=2pt
\item \textbf{Cross-entropy shift gauge.} For a classifier head $z(x) = W_L h(x) + b_L$ with softmax-CE loss, $G = \reals$ acts by $b_L \mapsto b_L + c \mathbf{1}_C$ (when bias is present) and $G = \reals^d$ acts by $W_L \mapsto W_L + \mathbf{1}_C u^\top$ for $u \in \reals^d$ (always; with no bias, this is the only CE-shift action). Both are translation actions; Euclidean is $G$-invariant; Adam fails because the per-coordinate $\hat v$ does not commute with summation across logit rows.

\item \textbf{ReLU rescaling gauge.} For adjacent weight layers $(W_\ell, W_{\ell+1})$ with a ReLU between, $c \in \reals^+$ acts as $(W_\ell, W_{\ell+1}) \mapsto (c W_\ell, c^{-1} W_{\ell+1})$. The Euclidean metric is not $G$-invariant; the natural $G$-invariant metric is the log-norm metric. Adam's per-coordinate $\hat v$ does not commute with the multiplicative paired-block action.

\item \textbf{LayerNorm scale gauge.} For a LayerNorm with scale $\gamma \in \reals^d$ paired with the downstream weight $W_{\rm next}$, $c \in (\reals^+)^d$ acts per channel as $(\gamma, W_{\rm next}) \mapsto (\gamma \odot c, W_{\rm next} \odot 1/c)$. Same multiplicative structure as ReLU rescaling.

\item \textbf{Attention head rotation gauge.} For per-head attention weights, $G = O(d_{\rm head})^{n_{\rm heads}}$ acts by right-multiplying each head's $W_Q, W_K$ (and $W_V, W_O$) by an orthogonal matrix, leaving the attention scores invariant. This action is non-abelian and mixes coordinates within each head, where per-coordinate $\hat v$ is furthest from equivariant: the hardest of the four gauges (\S\ref{ssec:gauges}, \S\ref{ssec:rope}).
\end{itemize}

\begin{figure}[t]
\centering
  \includegraphics[width=0.94\linewidth]{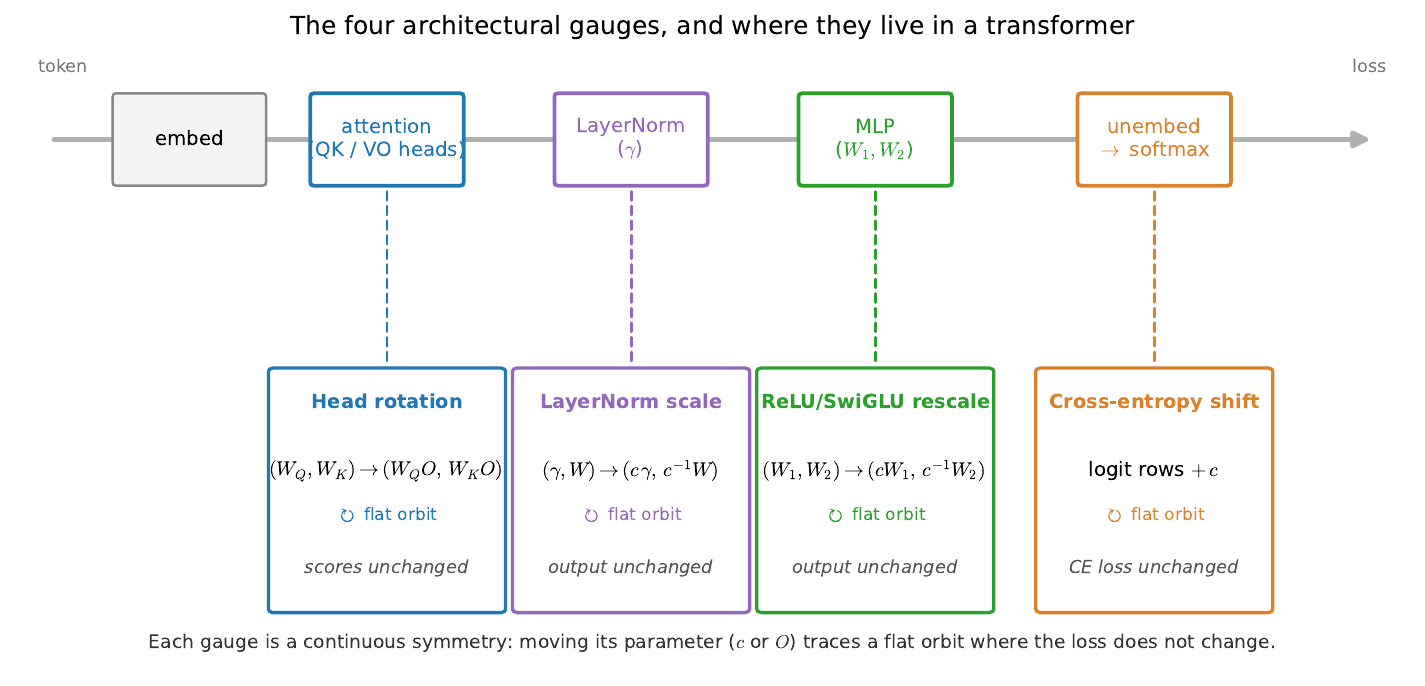}

\caption{The four architectural gauge groups. Each acts on parameter space and leaves the function the network computes unchanged: the cross-entropy row-shift, the ReLU rescaling of an adjacent weight pair, the per-channel LayerNorm scale, and the per-head $O(d_\mathrm{head})$ rotation of attention QK (and VO) weights. The first three are abelian, where per-coordinate $\hat v$ is automatically equivariant; the rotation is non-abelian, where it is not (\S\ref{ssec:gauges}).}
\label{fig:gauges}
\end{figure}
 
\citet{GSDTheoryAnon} characterise the resulting non-equivariance: the projected Adam trajectory is not gradient flow on $\bar\Theta$, the gauge mode of the parameter accumulates noise non-uniformly along the orbit, and the quotient rate is not directly readable from the trajectory. The two routes to restore rate-readability under adaptive optimisers are (i) gauge-fixing the loss at training time (Z-loss~\citep{deBrebissonVincent16} addresses the CE shift gauge only) and (ii) constructing a $G$-equivariant preconditioner. Section~\ref{sec:construction} executes route (ii) for all four architectural gauges.

\subsection{Related work}
\label{ssec:related}

The CE-shift gauge has a known soft fix, the Z-loss of \citet{deBrebissonVincent16}, and \S\ref{ssec:containment} shows why it fails under Adam: the per-coordinate normaliser turns the Z-loss gradient into a steady drift along the orbit. The architectural gauge groups of standard networks were identified and tied to Noether-type conservation laws by \citet{KuninSagastuyBrenaGanguli21}, with the Lagrangian treatment under non-symmetry-respecting metrics in \citet{TanakaKunin21}. \citet{DePaviaCharisopoulosWillett25} characterise the implicit-bias change of Adam under orthogonal feature-space rotations through the same per-coordinate $\hat v$ mechanism, for a different action. Rotation-quantization work (QuaRot~\citep{AshkboosCroci24QuaRot}, SpinQuant~\citep{LiuSpinQuant24}, OSTQuant~\citep{LiuOSTQuant25}) exploits the same per-head $O(d_{\rm head})$ freedom post-training, to flatten weight distributions for low-bit inference: a basis-dependent operation (per-channel quantization there, the per-coordinate Adam $\hat v$ here) composed with a rotation that diagonalises the gauge-orbit quotient. DDC's \texttt{body\_frame\_topk} truncates the eigenbasis of $W_Q^\top W_Q + W_K^\top W_K$ at $\lambda_i / \lambda_{\max} \ge \tau$, the same small-eigenvalue projection OSTQuant applies post-hoc, here during training. The two communities reach the same symmetry from opposite ends of the model lifecycle. On the optimizer side, \citet{YenLoRARITE24} establish the deflationary half of the no-go, that a coordinate-wise preconditioner fails transformation-invariance, while their RITE map equilibrates the LoRA-factor group rather than the per-channel gauge; \citet{FilatovKesselheim25} find the output-layer operator norm transfers optimal scaling across model and data sizes, matching the per-block norm reading of \S\ref{sec:discussion}. Whole-net second-order preconditioners, the Kronecker-factored family of Shampoo~\citep{GuptaKorenSinger18} and K-FAC~\citep{MartensGrosse15}, compose with the construction: SOAP~\citep{VyasSOAP24} runs Adam in Shampoo's Kronecker eigenbasis on the non-attention parameters while the body-frame rotation handles the attention heads, at a marginal per-head eigendecomposition cost. Riemannian adaptive methods~\citep{BecigneulGanea19} precondition Adam for a fixed product-manifold geometry, a different target from the symmetry quotient here. \citet{Pesme2021}'s diagonal-linear-network analysis is the natural starting point for a closed-form Adam rate.
 \section{Construction}
\label{sec:construction}

\subsection{The Equivariant Preconditioner}
\label{ssec:preconditioner}

Standard Adam tracks a per-coordinate second-moment estimator $\hat v_i \approx \mathbb{E}[g_i^2]$ and steps in direction $-\hat m_i / \sqrt{\hat v_i}$. The diagonal structure makes the preconditioner $G$-equivariant only when the action is by coordinate permutations. We replace the per-coordinate estimator with a vertical-horizontal split, computed in the orbit decomposition induced by a $G$-invariant Riemannian metric $\rho$.

Let $\Pi_V(\theta), \Pi_H(\theta)$ be the $\rho$-orthogonal projectors onto the vertical (orbit-tangent) and horizontal subspaces at $\theta$. Decompose the gradient at step $t$ as $g_t = \Pi_V g_t + \Pi_H g_t$. Maintain two second-moment estimators:
\begin{align*}
\hat v^V_t &\;=\; \beta_2 \hat v^V_{t-1} + (1-\beta_2) \cdot \mathrm{vert\text{-}to\text{-}scalar}(\Pi_V g_t)^{\odot 2} \in \reals^{\dim G}, \\
\hat v^H_t &\;=\; \beta_2 \hat v^H_{t-1} + (1-\beta_2) \cdot (\Pi_H g_t)^{\odot 2} \in \reals^p.
\end{align*}
The vertical estimator is collapsed to a single scalar per orbit dimension via a parameter-independent reduction (e.g., row-mean for the CE row-shift action; details below). The horizontal estimator is per-coordinate but receives only the horizontal gradient. The \eqadam{} update is
\begin{equation}
\theta_{t+1} \;=\; \theta_t \;-\; \eta \biggl[\, \Pi_H \biggl( \frac{\Pi_H \hat m_t}{\sqrt{\hat v^H_t} + \epsilon} \biggr) \;+\; \mathrm{lift}_V \biggl( \frac{\mathrm{vert\text{-}to\text{-}scalar}(\Pi_V \hat m_t)}{\sqrt{\hat v^V_t} + \epsilon} \biggr) \,\biggr],
\label{eq:ddcadam}
\end{equation}
where $\hat m_t$ is the bias-corrected first moment and $\mathrm{lift}_V$ inverts the orbit-collapsing reduction. The outer $\Pi_H$ on the horizontal summand re-projects after the per-coordinate division. This re-projection is required: per-coordinate division of a horizontal vector by $\sqrt{\hat v^H}$ does not preserve horizontality unless $\hat v^H$ is constant on orbit cells, and the re-projection is exactly equivalent to running Adam directly in a $G$-equivariant horizontal frame (a one-line lifting computation makes this explicit).

\begin{figure}[t]
\centering
  \includegraphics[width=0.96\linewidth]{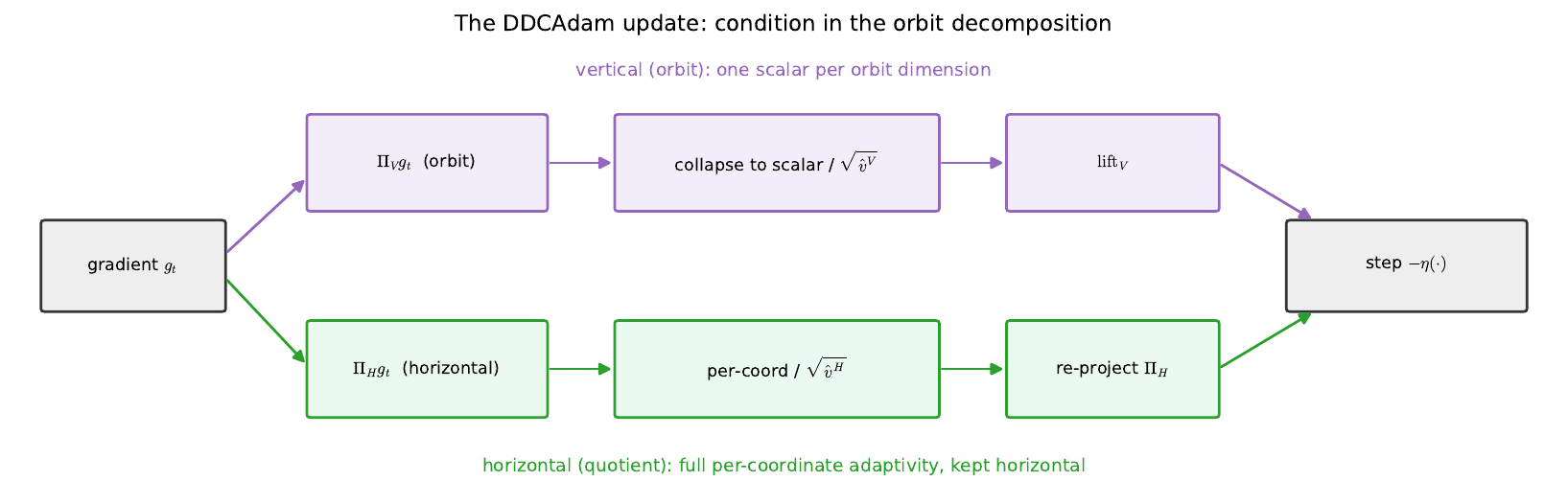}

\caption{The \eqadam{} update of Equation~\eqref{eq:ddcadam}. The gradient splits along the $\rho$-orthogonal projectors into an orbit (vertical) and a horizontal component. The vertical is collapsed to one scalar per orbit dimension, normalised, and lifted back; the horizontal takes per-coordinate Adam and is re-projected to stay horizontal. The re-projection is what keeps the trajectory a gradient flow on the quotient.}
\label{fig:update_flow}
\end{figure}
 
\paragraph{Vertical mode.} Equation~\eqref{eq:ddcadam} uses orbit-collapsed Adam on the vertical, matching the strict statement of Theorem~\ref{thm:ddcadam_rate}. Two practical alternatives preserve $G$-equivariance: \eqadam-SGD (no second-moment normalisation on the vertical: replace $\hat v^V_t$ by 1 and skip the EMA) and \eqadam-frozen (zero vertical update: replace the entire vertical summand by 0). Both alternatives leave the horizontal Adam adaptivity intact, which is the essential piece for rate-readability. The frozen variant has the tightest gauge-mode containment in practice.

\paragraph{Base optimizers.} The split lifts any base preconditioner. We run two: \eqadam{} over Adam throughout, and \textsc{DDCMuon} over Muon~\citep{JordanBernsteinKaramcheti24}, which orthogonalises the momentum update in place of a second-moment estimate. \textsc{DDCMuon} deploys the scaled-polar orthogonaliser, Muon's polar factor with a single gauge-invariant scale (Appendix~\ref{app:pseudocode}), distinct from the \texttt{body\_frame} second-moment construction of \S\ref{ssec:gauges}. How each base composes with the gauges, exactly for the rotation gauges on Muon and not the rescaling ones, is the boundary of \S\ref{ssec:muon_boundary}. Table~\ref{tab:selection} states which base each gauge selects.

\begin{table}[h]
\centering
\small
\setlength{\tabcolsep}{5pt}\renewcommand{\arraystretch}{1.3}\begin{tabularx}{\textwidth}{@{}>{\raggedright\arraybackslash}X >{\raggedright\arraybackslash}X >{\raggedright\arraybackslash}X l c@{}}
\toprule
Gauge & Quotient metric & Induced norm & Base & Composes \\
\midrule
Per-head attention rotation & rotation-invariant & spectral (RMS-to-RMS) & \textsc{DDCMuon} & yes \\
ReLU/SwiGLU rescale, LayerNorm scale & per-channel & per-coordinate & \eqadam{} & no \\
Cross-entropy row-shift & Euclidean (additive) & additive & either & yes \\
\bottomrule
\end{tabularx}
\caption{Which base optimiser each gauge selects, and whether it composes with an orthogonalising (Muon) base. The quotient-Fisher metric on a block fixes the modular norm of its update~\citep{BernsteinNewhouse24}: a rotation-invariant quotient metric induces the spectral norm and selects Muon, a per-channel quotient metric induces a per-coordinate norm and selects the diagonal update \eqadam{} applies. A single dualiser cannot be equivariant to both, so the orthogonalising base admits the rotation gauges and the additive cross-entropy row-shift, while the per-channel rescaling gauges ride \eqadam{}; the per-channel scale is the freedom orthogonalisation removes (\S\ref{ssec:muon_boundary}, \S\ref{sec:discussion}).}
\label{tab:selection}
\end{table}
 
\subsection{Architectural Gauges}
\label{ssec:gauges}

\paragraph{Cross-entropy shift.} Two architectural variants. With output bias $b_L \in \reals^C$, $G = \reals$ acts as $b_L \mapsto b_L + c \mathbf{1}_C$. Vertical: $V = \reals \cdot \mathbf{1}_C$, one-dimensional. The vert-to-scalar reduction is $g \mapsto \mathbf{1}_C^\top g / C$, lift is $s \mapsto s \cdot \mathbf{1}_C$. With no bias and unembed weight $W_L \in \reals^{C \times d}$, $G = \reals^d$ acts as $W_L \mapsto W_L + \mathbf{1}_C u^\top$. Vertical: $d$-dimensional, the constant-row matrices. The reduction is the column-mean $G \mapsto \bar G^\top \in \reals^d$ with $\bar G_j = (1/C) \sum_i G_{ij}$; lift is $s \mapsto \mathbf{1}_C s^\top$. Both variants have closed-form linear projectors; the construction adds $O(C d)$ overhead per unembed step relative to standard Adam.

\paragraph{ReLU rescaling.} For adjacent weight layers $(W_\ell, W_{\ell+1})$ with ReLU between, $G = \reals^+$ acts per pair as $(c W_\ell, c^{-1} W_{\ell+1})$. The Euclidean metric is not $G$-invariant; the natural $G$-invariant metric is the log-norm metric. Reparameterise $W_\ell = \rho_\ell U_\ell$ with $\rho_\ell = \|W_\ell\|_F$ and $U_\ell \in S^{p_\ell - 1}$. In coordinates $(\log \rho_\ell, U_\ell, \log \rho_{\ell+1}, U_{\ell+1})$, the action is translation by $(+\log c, -\log c)$ on the first and third coordinates and trivial on the rest. Vertical: span of $(\partial / \partial \log \rho_\ell - \partial / \partial \log \rho_{\ell+1})$. Per-coordinate Adam on $U_\ell, U_{\ell+1}$ is automatically $G$-equivariant. Scalar Adam on the joint-norm direction $(\log \rho_\ell + \log \rho_{\ell+1})/\sqrt{2}$ is the horizontal-radial mode; orbit-collapsed Adam on the gauge mode $(\log \rho_\ell - \log \rho_{\ell+1})/\sqrt{2}$ is the vertical update. Total cost: one polar split per pair per step, $O(p_\ell + p_{\ell+1})$.

\paragraph{LayerNorm scale.} For a LayerNorm with per-channel scale $\gamma \in \reals^d$ paired with the downstream weight column $W_{\rm next}$, $G = (\reals^+)^d$ acts per channel as $(\gamma_i, W_{\rm next}[:, i]) \mapsto (c_i \gamma_i, c_i^{-1} W_{\rm next}[:, i])$. Same multiplicative structure as ReLU rescaling, applied per channel. The construction is the per-channel log-coordinate analogue of the ReLU case.

\paragraph{Non-abelian: attention head rotation.} The three gauges above are abelian; the action on the horizontal subspace is by translation or coordinate-rescaling, and per-coordinate $\hat v$ in the standard basis is automatically $G$-equivariant (a coordinate squaring commutes with translation and with diagonal rescaling). The first non-abelian case in the paper is the per-head $O(d_{\rm head})$ rotation gauge on attention QK and VO weights: with $W_Q, W_K \in \reals^{n_{\rm heads} \times d_{\rm model} \times d_{\rm head}}$, the action $W_Q[h] \mapsto W_Q[h] \cdot O_h, W_K[h] \mapsto W_K[h] \cdot O_h$ for $O_h \in O(d_{\rm head})$ leaves the attention scores $(W_Q x)(W_K x')^\top$ invariant. The horizontality condition reduces per-head to a continuous Lyapunov equation $M_h A_h + A_h M_h = 2 \, \mathrm{skew}(S_h)$ with $M_h = W_Q[h]^\top W_Q[h] + W_K[h]^\top W_K[h]$ and $S_h = W_Q[h]^\top g_Q[h] + W_K[h]^\top g_K[h]$, solved via eigendecomposition with truncated-eigenvalue pseudoinverse for ill-conditioned heads; the projector is exactly $G$-equivariant. The VO gauge is the parallel construction on $(W_V, W_O)$ with both right-multiplied by the same $P_h$; the $.T$ on the canonical $\mathtt{out\,@\,W\_O.T}$ output projection means both transform by $P_h$ rather than by $P_h$ and $P_h^\top$. Closure: $d_{\rm head}(d_{\rm head}-1)/2$ horizontal directions per head per gauge.

The non-abelian action surfaces an issue invisible in the abelian cases: per-coordinate $\hat v$ in the standard basis is \emph{not} $O(d_{\rm head})$-equivariant. Under $g \mapsto g \cdot O$, per-coordinate squaring $\hat v_i = \mathbb E[g_i^2]$ does not commute with the rotation, so $\hat v$ accumulates basis-dependent components and the Adam update $g/\sqrt{\hat v}$ is no longer covariant (the same per-coordinate-squaring obstruction \citet{LingSharpJacobson22} identify for rotation-equivariant geometry optimisation, there resolved by a per-vector second moment). The projection step is exact; the per-coordinate $\hat v$ is what breaks. We therefore replace the standard-basis per-coordinate $\hat v$ on the rotation gauges with one of four $O(d_{\rm head})$-aware second-moment constructions (the \texttt{v\_mode} options in code; distinct from the vertical-mode choice of \S\ref{ssec:preconditioner}), in increasing expressive power:
\begin{itemize}[leftmargin=*]
\item \texttt{per\_head\_scalar}: $\hat v_h = \|g_h\|_F^2 / \dim h$, one scalar per head. Strict equivariance by Frobenius-norm invariance under right-orthogonal multiplication ($\|g \cdot O\|_F = \|g\|_F$) plus cyclic trace.
\item \texttt{per\_head\_matrix}: $\hat v_h \in \reals^{d_{\rm head}\times d_{\rm head}}$ accumulating the outer product $g_h^\top g_h$ per head. Update is $m_h \cdot (\hat v_h + \epsilon I)^{-1/2}$ via eigendecomposition. Under $g \mapsto g \cdot O$: $\hat v \mapsto O^\top \hat v O$, $\hat v^{-1/2} \mapsto O^\top \hat v^{-1/2} O$, so the update is strictly equivariant by conjugation.
\item \texttt{body\_frame}: per-coordinate $\hat v$ stored in the eigenbasis of $M = W_Q^\top W_Q + W_K^\top W_K$ (the body frame), with $m, \hat v$ rotated between consecutive frames each step. The body-frame Adam update is approximately $G$-equivariant modulo the eigenvector sign ambiguity that LAPACK resolves independently per matrix.
\item \texttt{none}: drop $\hat v$ entirely, momentum-only Adam. Strictly equivariant by trivial construction.
\end{itemize}
Dropping $\hat v$ entirely removes the adaptivity the optimizer needs to navigate the rotation orbit, so the strictest fix is not the one we deploy. \eqadam{} uses \texttt{body\_frame\_topk}, the \texttt{body\_frame} construction restricted to the live eigendirections of $M$: full per-coordinate adaptivity on the directions with $\lambda_i / \lambda_{\max} \ge \tau$ (default $\tau = 10^{-2}$), conditioning the dead ones below the cut through momentum alone. The cut is set by the eigenvalues of $M$, which are $G$-invariant ($M \mapsto O^\top M O$ leaves the spectrum fixed), so \texttt{body\_frame\_topk} inherits the \texttt{body\_frame} equivariance of Corollary~\ref{cor:vmode_equivariance}. Section~\ref{ssec:rope} matches it to the rotary attention geometry.

\begin{table}[h]
\centering
\small
\setlength{\tabcolsep}{6pt}\renewcommand{\arraystretch}{1.25}\begin{tabular}{@{}lllc@{}}
\toprule
\texttt{v\_mode} & Second moment & Equivariance & Deployed \\
\midrule
\texttt{none}            & dropped (momentum only)                  & exact (trivial); loses adaptivity & \\
\texttt{per\_head\_scalar} & one scalar per head $\|g_h\|_F^2/\dim h$ & exact (Frobenius invariance)      & \\
\texttt{per\_head\_matrix} & per-head PSD $g_h^\top g_h$              & exact (conjugation)               & \\
\texttt{body\_frame}     & per-coordinate in eigenbasis of $M$      & leading-order ($\le 10^{-3}$/20 steps) & \\
\texttt{body\_frame\_topk} & \texttt{body\_frame} on live eigendirections & inherits \texttt{body\_frame}  & $\checkmark$ \\
\bottomrule
\end{tabular}
\caption{The $O(d_\mathrm{head})$-aware second-moment constructions for the rotation gauge, in increasing expressive power (\S\ref{ssec:gauges}). Dropping $\hat v$ is strictly equivariant but removes the adaptivity the optimizer needs to navigate the rotation orbit; the per-head scalar and matrix forms are exactly equivariant but coarser; \texttt{body\_frame} keeps full per-coordinate adaptivity in the eigenbasis $M = W_Q^\top W_Q + W_K^\top W_K$ at the cost of a leading-order drift from the eigenvector sign ambiguity. \eqadam{} deploys \texttt{body\_frame\_topk}, the \texttt{body\_frame} construction restricted to the live eigendirections ($\lambda_i/\lambda_{\max} \ge \tau$), inheriting its equivariance class (Corollary~\ref{cor:vmode_equivariance}).}
\label{tab:vmode}
\end{table}
 
\subsection{Equivariance verification}
\label{ssec:implementation_status}

Section~\ref{sec:experiments} tests these constructions: the cross-entropy row-shift on the grokking transformer (\S\ref{ssec:containment}), the ReLU-rescale and LayerNorm-scale multiplicative gauges on synthetic teacher-student tasks (\S\ref{ssec:containment}, Appendix~\ref{app:relu}, Appendix~\ref{app:ln}), the per-coordinate gauge drift on a diagonal linear network (Appendix~\ref{app:dln_drift}), and the non-abelian rotation gauge on grokking and at language-model scale (\S\ref{ssec:rope}, \S\ref{ssec:rate_readability}). For chained ReLU rescale gauges (a sequence of layer pairs $(W_1, W_2), (W_2, W_3), \ldots$ sharing intermediate tensors) we implement a joint $(\reals^+)^{L-1}$ product-group construction with an $(L-1)$-dimensional vertical subspace in joint log-norm coordinates (orthogonal basis via the discrete cosine basis on the zero-mean subspace of $\reals^L$); this is jointly $G$-equivariant for the product group. All five abelian constructions (CE-shift bias, CE-shift row-shift, ReLU rescale, LN scale, chained ReLU rescale) verify $G$-equivariance of the implementation to fp32 precision (relative error $\le 5 \times 10^{-7}$) via paired-trajectory self-tests on bit-identical $G$-related gradients. The non-abelian rotation gauges (QK and VO) verify strict equivariance to fp64 machine precision ($\le 10^{-12}$) under the \texttt{per\_head\_scalar} and \texttt{per\_head\_matrix} variants, and approximate equivariance ($\le 10^{-3}$ over 20 steps) under \texttt{body\_frame} (Table~\ref{tab:equivariance}).

\begin{table}[h]
\centering
\small
\setlength{\tabcolsep}{6pt}\renewcommand{\arraystretch}{1.2}\begin{tabular}{@{}llll@{}}
\toprule
Construction & Gauge & Equivariance & Verified precision \\
\midrule
CE-shift (bias)        & abelian, additive       & exact & $\le 5\times10^{-7}$ (fp32) \\
CE-shift (row-shift)   & abelian, additive       & exact & $\le 5\times10^{-7}$ (fp32) \\
ReLU rescale           & abelian, multiplicative & exact & $\le 5\times10^{-7}$ (fp32) \\
LayerNorm scale        & abelian, multiplicative & exact & $\le 5\times10^{-7}$ (fp32) \\
Chained ReLU rescale   & $(\reals^+)^{L-1}$ product & exact & $\le 5\times10^{-7}$ (fp32) \\
\midrule
Rotation, \texttt{per\_head\_scalar} & non-abelian $O(d_\mathrm{head})$ & exact (Frobenius)   & $\le 10^{-12}$ (fp64) \\
Rotation, \texttt{per\_head\_matrix} & non-abelian $O(d_\mathrm{head})$ & exact (conjugation) & $\le 10^{-12}$ (fp64) \\
Rotation, \texttt{body\_frame}       & non-abelian $O(d_\mathrm{head})$ & leading-order       & $\le 10^{-3}$ over 20 steps (fp64) \\
\bottomrule
\end{tabular}
\caption{Implementation equivariance, verified by paired-trajectory self-tests that initialise two copies related by a random $G$-action, feed bit-identical $G$-related gradients, and measure the relative drift after stepping (50 steps for the abelian constructions, 20 for the rotation gauges; Appendix~\ref{app:equiv_tests}). The five abelian constructions are exact to fp32. The two strict rotation variants are exact to fp64 machine precision; \texttt{body\_frame} is equivariant only to leading order, its drift set by the eigenvector sign ambiguity of $\mathrm{eigh}(M)$ and bounded by Lemma~\ref{lem:body_frame_drift}. The deployed \texttt{body\_frame\_topk} inherits the \texttt{body\_frame} class.}
\label{tab:equivariance}
\end{table}
  \section{Trajectory-Rate Readability under \eqadam{}}
\label{sec:theory}

The original motivation to build \eqadam{} was so that one theorem holds. Reading the singular-learning rate off a trajectory (Section~\ref{sec:intro}) requires the optimizer to descend cleanly on the quotient $\bar\Theta = \Theta^\circ / G$, and Adam fails that requirement: its diagonal preconditioner injects a component along the gauge orbit that the projection onto $\bar\Theta$ cannot absorb, so the rate blurs. \eqadam{} removes the failure at its source, conditioning the optimizer state in the orbit decomposition, where Adam conditions in the coordinate basis.

The theorem states what that construction buys, in four parts. Equivariance of the preconditioner (i) makes the continuous-time flow commute with $G$, so it descends to a well-defined flow on $\bar\Theta$ (ii). The descended flow is preconditioned gradient descent of the quotient loss $\bar L$, which carries the trajectory-rate readability of \citet{GSDTheoryAnon} unchanged (iii). The discrete update inherits the rate up to an Euler bias that the preconditioner conditioning $\kappa(P)$ controls (iv). The construction asks nothing of the loss that SGD does not already meet; it holds the adaptive preconditioner to the equivariance SGD has for free. The theory paper states this construction's equivariance and rate readout (\citealp{GSDTheoryAnon}, Prop.~83 and Cor.~86) and defers the full construction here; Theorem~\ref{thm:ddcadam_rate} supplies it, adding the discrete-step bias (iv) and, through Corollary~\ref{cor:vmode_equivariance}, the non-abelian rotation gauge the theory paper's abelian gauge classes do not treat.

\begin{theorem}[Trajectory-rate readability under \eqadam]
\label{thm:ddcadam_rate}
Let $\Theta^\circ \subset \reals^p$ be a parameter open set carrying a smooth, free, proper action of a finite-dimensional Lie group $G$, with $G$-invariant loss $L : \Theta^\circ \to \reals$ and $G$-invariant statistical model $\{p_\theta\}_{\theta \in \Theta^\circ}$. Let $\rho$ be a $G$-invariant Riemannian metric on $\Theta^\circ$, and let $V_\theta \subset T_\theta \Theta^\circ$ denote the vertical (orbit-tangent) subspace at $\theta$, with horizontal complement $H_\theta := V_\theta^{\perp_\rho}$. Let $\pi : \Theta^\circ \to \bar\Theta := \Theta^\circ / G$ be the quotient submersion.

Let $\Pi_V(\theta), \Pi_H(\theta) : T_\theta \Theta^\circ \to T_\theta \Theta^\circ$ be the vertical and horizontal $\rho$-orthogonal projectors, and define the \emph{orbit-averaged second-moment estimator} $\hat v_t \in T_\theta \Theta^\circ \otimes T_\theta \Theta^\circ$ by separate exponential moving averages of $\Pi_V g_t \otimes \Pi_V g_t$ (orbit-collapsed: a single scalar per orbit dimension) and $\Pi_H g_t \otimes \Pi_H g_t$ (per-coordinate in a $G$-equivariant horizontal frame), where $g_t = \nabla_\theta L(\theta_t)$. Let $\eqadam$ denote the optimizer
\begin{equation}
\theta_{t+1} = \theta_t - \eta \cdot P(\theta_t, g_t), \qquad P(\theta, g) := \frac{\Pi_V g}{\sqrt{\hat v^V} + \epsilon} + \Pi_H \biggl( \frac{\Pi_H g}{\sqrt{\hat v^H} + \epsilon} \biggr),
\label{eq:ddcadam_update}
\end{equation}
with orbit-collapsed division on the vertical and per-coordinate-then-re-projected division on the horizontal. The outer $\Pi_H$ on the horizontal summand is exactly equivalent to running Adam directly in a $G$-equivariant horizontal frame: per-frame-coordinate $m / \sqrt{v}$ lifts back to the parameter representation as the per-coordinate ratio with the column-mean (in the orbit-cell sense) subtracted. The lifted form makes the implementation a one-line projection rather than a frame change.

Then:
\begin{enumerate}[label=(\roman*)]
\item \textbf{$G$-equivariance.} For all $h \in G$, $\theta \in \Theta^\circ$, $g \in T_\theta^* \Theta^\circ$,
\[
(dh)_\theta P(\theta, g) = P(h \cdot \theta, (dh^{-1})^* g),
\]
so the discrete \eqadam{} update commutes with the $G$-action.
\item \textbf{Projection.} The continuous-time \eqadam{} flow $\dot\theta = -P(\theta, \nabla L(\theta))$ projects under $\pi$ to a well-defined dynamics on $\bar\Theta$ that coincides with preconditioned gradient descent of $\bar L := L / G$ in the quotient metric $\bar\rho$.
\item \textbf{Rate readability.} Suppose a continuous-time \eqadam{} trajectory $\theta(t)$ approaches a singular minimum $\theta_0 \in \pi^{-1}(\bar\theta_0)$ along a horizontal canonical-aligned direction $u \in H_{\theta_0}$ of KL order $k \ge 2$. Let $\bar\theta(t) := \pi(\theta(t))$ and $\bar t(t)$ be the quotient arc-length to $\bar\theta_0$. Then
\[
\bar u^\top \bar\fisher(\bar\theta(t))\, \bar u \;=\; \Theta\bigl(\bar t(t)^{2(k-1)}\bigr), \qquad \lambda_{\min}\bigl(\bar\fisher(\bar\theta(t))\bigr) \;=\; \Theta\bigl(\bar t(t)^{2(k-1)}\bigr),
\]
where $\bar u = \pi_*(u)$ and $\bar\fisher$ is the quotient Fisher metric.
\item \textbf{Discrete-step bias.} The discrete \eqadam{} update~\eqref{eq:ddcadam_update} with step $\eta$ inherits (iii) up to an Euler-Maruyama bias of order $O(\eta^2 \cdot \kappa(P))$, where $\kappa(P) = \sup_{\theta, g, \|g\|=1} \|P(\theta,g)\| / \inf_{\theta, g, \|g\|=1} \|P(\theta,g)\|$ is the preconditioner conditioning along the trajectory.
\end{enumerate}
\end{theorem}

\begin{remark}[Vertical-mode choice]
\label{rem:vertical_mode}
Items (ii)--(iii) depend only on the horizontal summand of $P$: the vertical summand projects to zero under $\pi$ and is irrelevant to the projected dynamics on $\bar\Theta$. Replacing the orbit-collapsed Adam on the vertical by any other $G$-equivariant choice, SGD ($\Pi_V \hat m_t$, no normalisation) or zero (no vertical update), preserves rate-readability. The zero-vertical variant has the tightest empirical gauge-mode containment because it does not amplify finite-precision noise in the orbit-collapsed second moment.
\end{remark}
 
\begin{proof}[Proof sketch]
$G$-equivariance of $P$ follows because the orbit-collapsed $\hat v^V$ depends only on the orbit (not on the orbit representative) and the per-coordinate $\hat v^H$ is computed in a $G$-equivariant horizontal frame. Equivariance of $P$ combined with $G$-equivariance of $\nabla L$ (from $G$-invariance of $L$) makes the continuous flow $G$-equivariant; it projects under $\pi$ to a flow on $\bar\Theta$, with the vertical summand of $P$ projecting to zero. The projected flow is preconditioned gradient flow of $\bar L$, and the rate transfers from~\citep{GSDTheoryAnon} after a constant time-rescaling. Discrete bias follows the standard Euler-Maruyama bound, with $\kappa(P)$ controlling the magnitude. Full proof in Appendix~\ref{app:proofs}.
\end{proof}

The theorem covers the abelian gauges directly. The rotation gauge needs the second-moment construction of \S\ref{ssec:gauges}, and the corollary below verifies that each of the four choices keeps the preconditioner $G$-equivariant, so the rate transfers to the attention heads.

\begin{corollary}[Non-abelian second-moment equivariance]
\label{cor:vmode_equivariance}
Take Theorem~\ref{thm:ddcadam_rate} with $G = O(d_{\rm head})^{n_{\rm heads}}$ acting on per-head attention weights $W \in \reals^{n_{\rm heads} \times d_{\rm model} \times d_{\rm head}}$ by right-multiplication, $W[h] \mapsto W[h] \cdot O_h$. Replace the per-coordinate horizontal preconditioner $\Pi_H g / (\sqrt{\hat v^H} + \epsilon)$ in eq.~\eqref{eq:ddcadam_update} by any of the four operators $P_H(g_H, \mathrm{state})$ defined in \S\ref{ssec:gauges} and Appendix~\ref{app:pseudocode} (\texttt{per\_head\_scalar}, \texttt{per\_head\_matrix}, \texttt{body\_frame}, \texttt{none}). Then:
\begin{enumerate}[label=(\roman*)]
\item Each operator $P_H$ is $G$-equivariant: $P_H(g_H \cdot O, \mathrm{state}) = P_H(g_H, \mathrm{state}) \cdot O$ where the state's transformation under $g \mapsto g \cdot O$ is $\hat v \mapsto \hat v$ (\texttt{per\_head\_scalar}, \texttt{none}), $\hat v \mapsto O^\top \hat v O$ (\texttt{per\_head\_matrix}), or $g_{\rm body} \mapsto g_{\rm body}$ (\texttt{body\_frame}, modulo \texttt{eigh} sign convention).
\item Items (i)--(iv) of Theorem~\ref{thm:ddcadam_rate} hold with each operator substituted.
\end{enumerate}
\end{corollary}

\begin{proof}
Item (i) for each operator is Propositions~\ref{prop:phs_equiv}, \ref{prop:phm_equiv}, \ref{prop:bf_equiv}, and \ref{prop:none_equiv} in Appendix~\ref{app:proofs}. Item (ii) follows because Theorem~\ref{thm:ddcadam_rate}'s items (ii)--(iv) only use that $P_H$ is $G$-equivariant and that the resulting projected dynamics on $\bar\Theta$ is preconditioned gradient descent of $\bar L$ in some $G$-invariant operator-norm. The first hypothesis is item (i); the second is verified for each operator: \texttt{per\_head\_scalar} preconditions by a scalar (a $G$-invariant rescaling of the quotient metric), \texttt{per\_head\_matrix} preconditions by a $G$-invariant operator on the orbit cell (the conjugation action $\hat v \mapsto O^\top \hat v O$ collapses on the quotient), and \texttt{body\_frame} preconditions by a frame-aligned per-direction scaling that is exact in the body frame (and approximately exact under the world-frame interpretation modulo the sign convention). \texttt{none} preconditions by the identity.
\end{proof}

\begin{remark}[Body-frame approximate equivariance]
\label{rem:bf_approx}
\texttt{body\_frame} achieves equivariance modulo the eigenvector sign ambiguity that LAPACK's \texttt{eigh} resolves independently per matrix. For two configs $W$ and $W \cdot O$ related by the gauge action, the eigenvectors of $M = W^\top W$ and $M' = O^\top M O$ are related by $U' = O^\top U \cdot S$ for some sign-flip diagonal $S \in \{\pm 1\}^{d_{\rm head}}$. The per-direction Adam state in the body frame is invariant up to multiplication by $S$, and the rotation step between consecutive frames is exact for momentum (linear) but absolute-valued for the second moment (quadratic), so the sign drift accumulates only in $\hat v$ and decays as $|S|^2 = 1$ on the diagonal. Empirical drift is $\sim 10^{-3}$ over 20 fp64 steps; a sign-fixing pass on consecutive frames lifts this to machine precision. The rate-readability statement holds for the exact body-frame Adam (no sign drift); the practical implementation introduces a $O(\eta \cdot s)$ bias where $s$ is the per-step sign-drift rate.
\end{remark}
  \section{Experiments}
\label{sec:experiments}

This section is findings-first: each subsection names what the result lets a reader use, then the minimal evidence that earns it; the full setup, seeds, devices, and result-file provenance live in Appendix~\ref{app:experiments}. A few practices hold throughout. Quantitative claims run multiple seeds under deterministic seeding, three by default and five for the diagonal-network and over-parametrisation reads; single-seed mechanism and sweep studies are flagged where they appear. Each DDC arm is read against a matched baseline at the same weight decay, with the base optimizers tuned over a learning-rate by weight-decay grid (Appendix~\ref{app:baseline_sweep}) so the comparison reads each base at its best. Geometry comes from the residual-free true-MC Fisher at $n/d \ge 100$, and a singular-approach rate is fit only after an asymptoticity gate confirms the trajectory has reached its power-law regime. Table~\ref{tab:results_master} collects the findings.

\begin{table}[t]
\centering
\small
\setlength{\tabcolsep}{5pt}\renewcommand{\arraystretch}{1.25}\begin{tabularx}{\textwidth}{@{}cl X ccc@{}}
\toprule
 & Setting (base optimiser) & What DDC changes & DDC & Baseline & $n$ \\
\midrule
F1 & depth-12 LM, over-training (Adam)      & rate-readable cells, of 65          & \textbf{32}        & 7         & 3  \\
F2 & 1-block grok, $\bmod 113$ (Muon)        & dead-basis axis alignment           & \textbf{0.889}     & 0.718     & 5  \\
F3 & depth-8 RoPE grok (Muon)                & post-grok accuracy                  & \textbf{0.977}     & 0.867     & 3  \\
F4 & depth-24 arithmetic grok (Muon)         & seeds reaching grok                 & \textbf{10/11}     & 0/11      & 11 \\
F5 & depth-8 grok, true-MC (Muon)            & $\lambda_{\min}(\bar\fisher)$ off the floor & \textbf{$5.8\!\times\!10^{-11}$} & $4.6\!\times\!10^{-12}$ & 3  \\
F7 & ViT from scratch, ImageNet-100 (Adam)   & validation loss (nats)              & \textbf{1.712}     & 2.116     & 3  \\
\bottomrule
\end{tabularx}
\caption{The findings at a glance. Each row reports the DDC arm against its matched baseline on one setting, with the base optimiser in parentheses; lower is better for F5 in the sense of distance off the singular floor (higher $\lambda_{\min}$ is less degenerate) and for F7 (loss). Baselines are AdamW (F1, F7), a weight-decay-matched Muon (F2), and vanilla Muon (F3--F5). F1 instantiates the rate guarantee of Theorem~\ref{thm:ddcadam_rate}; the Muon composition boundary (F6) is stated in \S\ref{ssec:muon_boundary}. The remaining rows are observed effects of the DDC versions across the settings run. Per-cohort seeds, gate status, and provenance are in Appendix~\ref{app:experiments} (Table~\ref{tab:cohort_provenance}).}
\label{tab:results_master}
\end{table}
 
\subsection{Reading the rate at language-model scale}
\label{ssec:rate_readability}

Theorem~\ref{thm:ddcadam_rate} guarantees the quotient rate along a \eqadam{} trajectory, and we read it directly on a modern language model. We train a small transformer (RMSNorm, SwiGLU, grouped-query attention, RoPE; Appendix~\ref{app:otr}) on a fixed 10M-token FineWeb-edu slice held far below its capacity, for many epochs into the over-training regime where the model memorises the slice as validation degrades, and fit the dead-direction approach across 65 layer-by-observable cells. AdamW fails to read it in 58: the relevant singular value rises away from the minimum, so no rate can be fit. \eqadam{} reaches the asymptotic regime in 32, its singular value collapsing by nearly four orders of magnitude (Figure~\ref{fig:otr_dds}). No cell reads a finite rate under both optimizers, so the projection bias of Lemma~\ref{lem:adam_projection_bias} appears here in its strongest form: the two readable profiles separate completely.

We read the rate from the quotient-Fisher decay, fitting $\lambda_{\min}(\bar\fisher)$ against the quotient arc-length; on this network the order of vanishing takes its generic value $k = 2$, and the readable cells recover the predicted $\bar t^{2}$.

\begin{figure}[t]
\centering
  \includegraphics[width=0.96\linewidth]{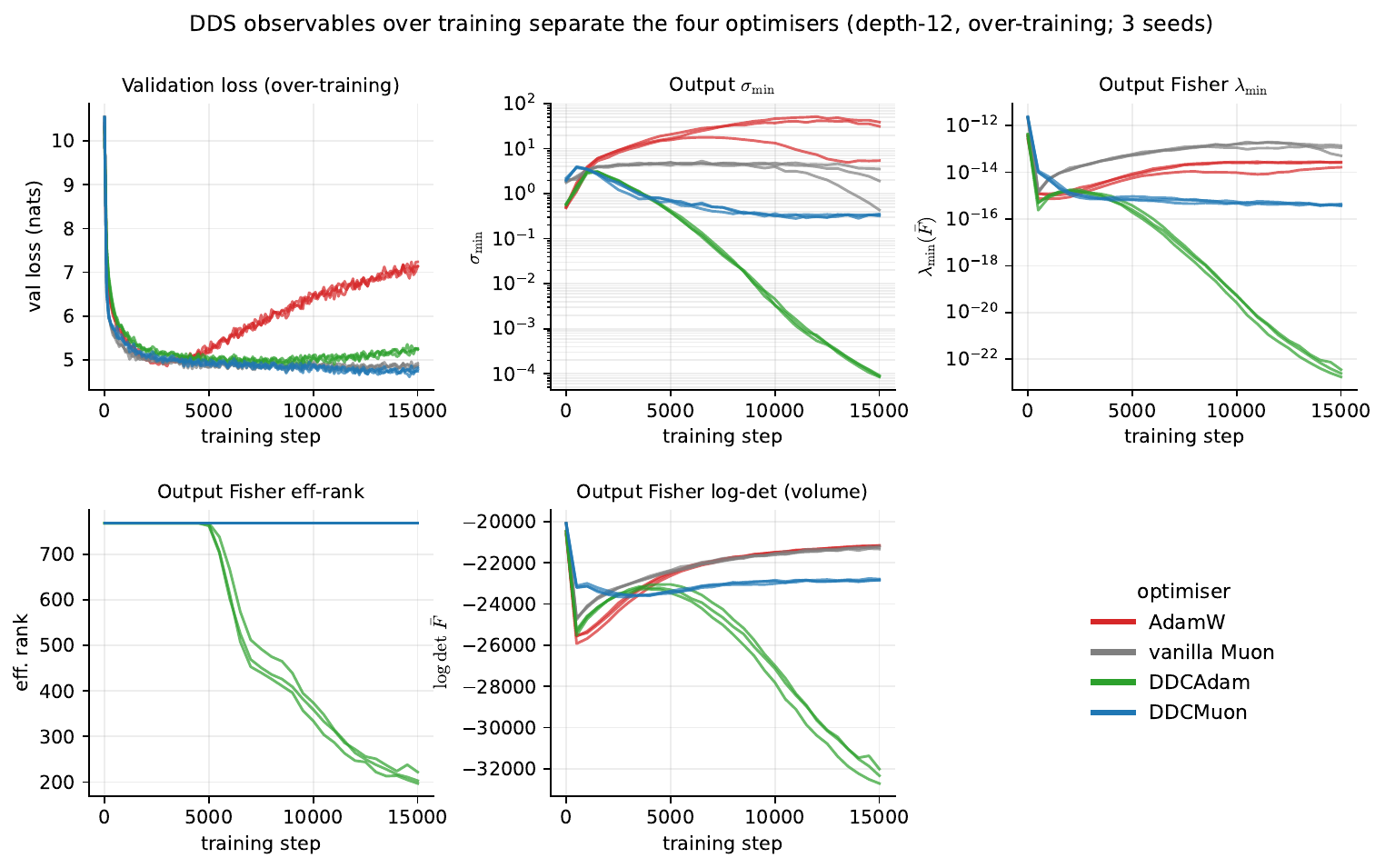}

\caption{The dead-direction observable suite over training separates the four optimisers, on the depth-12 over-trained language model (3 seeds; the Muon arm is a textbook Newton-Schulz vanilla baseline). All panels span the full 15{,}000-step run; the per-layer true-MC geometry is read every $500$ steps. Validation loss orders all four (AdamW over-trains and climbs; \textsc{DDCMuon} settles lowest). The dead-direction geometry separates them further: under \eqadam{} the output activation $\sigma_{\min}$ collapses by nearly four orders, the Fisher $\lambda_{\min}$ and effective rank fall with it, and the Fisher log-determinant (the volume of the output singular structure) collapses, the strongest dead-direction formation of the four, while vanilla Muon and AdamW hold their output geometry. Each optimiser leaves a distinct signature across the suite.}
\label{fig:otr_dds}
\end{figure}
 
\subsection{Containing the gauge mode}
\label{ssec:containment}

The construction must do one thing first: hold the gauge mode that Adam moves. We watch it on the grokking transformer of \citet{NandaChanLieberum23}, a one-block residual network trained on $(a+b)\bmod 113$, where the readout's row-shift gauge $G=\reals^d$ (\S\ref{ssec:gauges}) gives the cleanest abelian case. SGD holds the gauge mode at its initial value, because cross-entropy has zero column-mean gradient and the orbit lies in the gradient kernel. AdamW drifts it $1.19\times$ over 8000 steps as the per-coordinate second moment breaks that invariance. \eqadam-frozen holds it at the floor; AdamW and \eqadam{} leave the quotient identical (ratio $1.000$), so the construction removes the Adam-specific drift and nothing else (Figure~\ref{fig:ce_shift_gauge}).

Z-loss, the standard fix for the same readout instability, drifts the gauge mode worse than no fix at all. Adding $\alpha(\log Z)^2$ gives the gauge direction a non-zero gradient (Eq.~\ref{eq:zloss_gauge_grad}), which Adam's normalisation turns into an $O(\eta)$ push every step, so the gauge mode accumulates as $\Theta(T\eta)$ against the $\Theta(\sqrt T \eta)$ random walk of vanilla AdamW. PaLM-style Z-loss grows the mode to $7\times$ initialisation where vanilla AdamW reaches $3\times$, and strong Z-loss contains it only by failing to grok. Z-loss acts as a soft section of the gauge, and under Adam that section fights the per-coordinate normalisation.

The containment holds across the other abelian gauges. On synthetic teacher-student tasks, \eqadam{} tightens the cross-seed drift of the ReLU-rescale and LayerNorm-scale modes by $14\times$ and $65\times$ over AdamW while matching its final loss (Table~\ref{tab:multiplicative}); the small residual \eqadam{} drift comes from the $O(\eta^2 T)$ Pythagorean correction of combining a multiplicative radial update with an additive tangential one in finite precision, with the per-step equivariance exact. The construction complements the orthogonal-gradient projection of \citet{PrietoBarsbey25}, which projects the radial logit-scaling descent direction out of the gradient: \eqadam{} fixes the gauge mode, the orthogonal complement. The two act on different readout directions.

The gauge-direction gradient that Eq.~\ref{eq:zloss_gauge_grad} evaluates is a column-mean calculation. Cross-entropy has zero column-mean in its unembed-weight gradient ($\sum_c (p_c - \mathbf{1}[c{=}y]) = 0$), so the row-shift orbit lies in its kernel, while the $\alpha(\log Z)^2$ term carries a per-row $2\alpha(\log Z)\,p_i\,h^\top$ whose column-mean
\begin{equation}
\frac{1}{C}\sum_i 2\alpha (\log Z) p_i h^\top \;=\; \frac{2\alpha}{C} (\log Z) \cdot h^\top \;\ne\; 0,
\label{eq:zloss_gauge_grad}
\end{equation}
stays nonzero since $\sum_i p_i = 1$.

\begin{figure}[t]
\centering
  \includegraphics[width=0.95\linewidth]{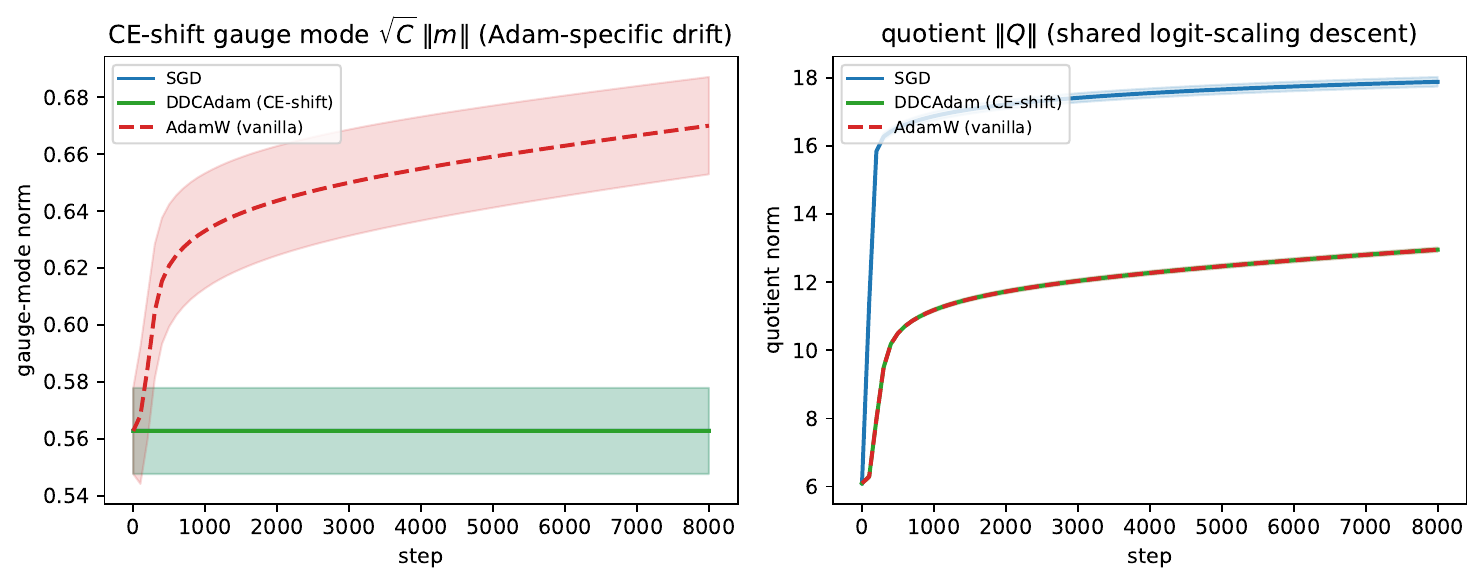}

\caption{Descent versus gauge on the readout, three seeds, $\mathrm{wd} = 0$ (the no-regularisation regime where grokking is naive-loss-minimisation-driven). The unembed weight splits orthogonally as $W = \mathbf{1}_C m^\top + Q$, $m$ the class-mean row, $\|W\|^2 = C\|m\|^2 + \|Q\|^2$. \emph{Left, the CE-shift gauge mode $\sqrt{C}\|m\|$:} vanilla AdamW (dashed) drifts it ($1.19\times$ over 8000 steps); SGD holds it exactly at its initial value (the cross-entropy gradient has zero column-mean, so the orbit is in the gradient kernel); \eqadam(CE-shift) suppresses it back to that floor. \emph{Right, the quotient $\|Q\|$ (the radial logit-scaling descent, the naive-loss-minimisation direction of \citet{PrietoBarsbey25}):} AdamW (dashed) and \eqadam{} coincide exactly (ratio $1.000$), so the construction leaves the descent untouched; SGD's larger quotient is its larger effective step, not the gauge. The Adam-vs-\eqadam{} comparison isolates the fix: \eqadam{} removes the gauge drift and changes nothing else.}
\label{fig:ce_shift_gauge}
\end{figure}
 
\begin{table}[h]
\centering
\small
\begin{tabular}{llcc}
\toprule
Gauge & Optimiser & Gauge-mode drift (mean $\pm$ std) & Final loss \\
\midrule
ReLU rescale  & AdamW              & $-0.032 \pm 0.050$ & $3.6 \cdot 10^{-1}$ \\
                                            & \eqadam-ReLU-frozen & $+0.024 \pm \textbf{0.003}$ & $3.2 \cdot 10^{-1}$ \\
LN scale              & AdamW              & $-1.239 \pm 0.412$ & $1.3 \cdot 10^{-6}$ \\
                                            & \eqadam-LN-frozen   & $+0.044 \pm \textbf{0.006}$ & $1.3 \cdot 10^{-6}$ \\
\bottomrule
\end{tabular}
\caption{Multiplicative-gauge containment under \eqadam{} versus AdamW on synthetic teacher-student tasks (3 seeds, 4000 steps, lr $=10^{-2}$, fp32). \emph{Gauge-mode drift} is the change in the $G$-invariant log-norm coordinate over training: $r = \log(\|W_1\|_F / \|W_2\|_F)$ for ReLU rescale (single $\reals^+$ gauge on a 2-layer MLP $f(x) = W_2 \mathrm{ReLU}(W_1 x)$); $\|r\|_2$ for LN scale (per-channel $(\reals^+)^d$ gauge on $f(x) = W_\mathrm{next} \mathrm{LN}(x;\gamma)$, $\|r\|$ is the $\ell_2$ norm over the $d$ channel-wise drifts). \emph{Final loss} is teacher-student MSE at step 4000. \eqadam{}-frozen denotes the construction of \S\ref{ssec:preconditioner} with zero vertical update (the gauge-mode coordinate receives no optimizer step). Bold entries mark the cross-seed std comparison: \eqadam-frozen has $14\times$ tighter std on ReLU rescale and $65\times$ tighter std on LN scale (both ratios computed from the unrounded cross-seed std; the tabulated std are rounded), while matching AdamW's final loss within rounding. Equivariance of the implementation is verified to fp32 precision (relative error $\le 5 \times 10^{-7}$) by paired-trajectory self-tests on bit-identical $G$-related gradients.}
\label{tab:multiplicative}
\end{table}

\subsection{Holding a block the task needs}
\label{ssec:radial}

The abelian gauges decouple each channel's scale from its direction and let weight
decay act on the scale alone, and the coordinate that decay runs on decides whether a
feed-forward block survives. We separate the two on a network where the feed-forward is
the entire non-linear model, a one-hidden-layer block with an internal LayerNorm on
sparse parity, where zeroing the hidden layer drops the network to chance
(Appendix~\ref{app:radial}). At one weight decay the two coordinates split across seeds
and across both bases. On the log of the joint scale, decay has no fixed point and drives
the scale to zero, collapsing the block to chance. On the linear scale, in the manner of
decoupled weight decay, it balances at the equilibrium $s^\star = 1/(2\lambda)$ for the
weight decay $\lambda$, and the block survives at full accuracy. The coordinate is the control: a block the
solution needs is held on the linear scale, and a block it can do without, the spare
feed-forward of a task the gauge routes through attention instead, is released on the log
scale, where shedding it leaves a cleaner and more stable solution (Figure~\ref{fig:radial}).

Keeping the block has a simpler-looking fix the construction cannot use. The standard
LayerNorm reparametrisation, carrying $1 + \gamma$ and decaying $\gamma$ toward zero,
holds the block alive with no gauge at all. It leaves the gauge mode free, and the
construction exists to hold that mode: at one weight decay the reparametrisation drifts
the mode four orders of magnitude more than the linear scale, which holds the block and
freezes the mode together. The equilibrium $s^\star$ is the decoupled-weight-decay fixed
point \citep{LoshchilovHutter19}; what the gauge contributes is reaching it on a
coordinate that leaves the orbit fixed.

The collapse hides from the obvious measurement. On a Muon base the orthogonalised step
re-grows the paired weight along its own direction, so the block's global gain
$\|W_1\|_F \|W_2\|_F$ stays near a healthy value while its channels are dead. The
per-channel joint norm read against an absolute floor reports the collapse the global
gain conceals.

\begin{figure}[t]
\centering
  \includegraphics[width=0.72\linewidth]{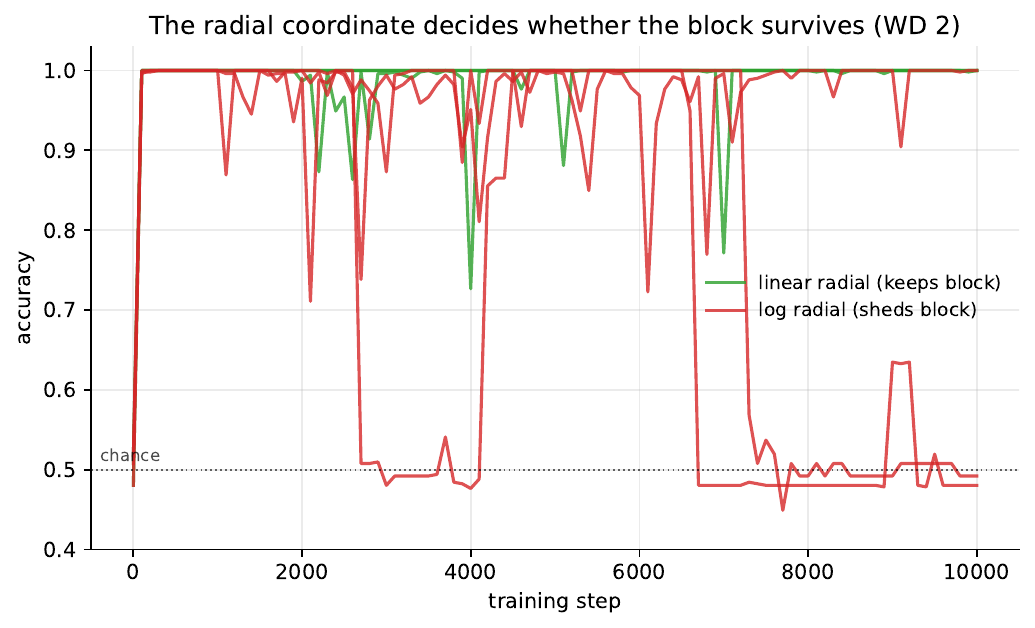}

\caption{The radial coordinate decides whether a load-bearing block survives, on sparse parity where zeroing the hidden layer drops the network to chance (Muon base, weight decay $2$, 3 seeds). On the linear scale the decoupled decay balances at the equilibrium $s^\star = 1/(2\lambda)$ and the block holds at full accuracy; on the log scale the decay has no fixed point, drives the joint scale to zero, and the block collapses to chance. The two coordinates leave the gauge orbit fixed either way; only the survival of the block differs.}
\label{fig:radial}
\end{figure}
 
\subsection{Weight decay or projection?}
\label{ssec:wd_vs_projection}

At matched nominal weight decay across the tuned grid, DDCMuon groks cell for cell with a plain Muon (Appendix~\ref{app:baseline_sweep}): hold the weight decay equal and the gauge projection buys no Muon-base speed. The apparent speed-up has a simple source. An equivariant update carries weight decay along its horizontal frame, so a gauge applying extra shrinkage outpaces a vanilla optimizer at the same nominal setting, and weight decay on its own accelerates grokking. Hand a plain Muon the full weight decay the gauge effectively carries and it groks just as fast, 4120 steps against the gauge's 4160, where the under-matched Muon takes 5080 (Figure~\ref{fig:wd_projection}; Appendix~\ref{app:wd}). The Adam base leaves a residual the weight-decay match does not explain: across the three-seed tuning grid DDCAdam groks faster than AdamW in the high-weight-decay column at low-to-mid learning rate (median 1900 steps against 2200 at learning rate $10^{-3}$, weight decay $2$, and 6800 against 8100 at learning rate $3\times10^{-4}$), runs comparably elsewhere, and trails by at most 200 steps along the high-learning-rate row where AdamW peaks. The Muon-base speed belongs to the regularization the gauge carries, and the reliable gains the projection brings come from elsewhere.

Two effects the weight-decay match leaves untouched account for those gains, both shaping where the network ends up. Reliability comes first. Hand a plain Muon the gauge's full weight decay and it still groks only sometimes, one seed in three and often none, while every gauge arm groks, three in three, and lands seven to eleven accuracy points higher. The projection buys reliability: it makes grokking happen at all. The projection also reshapes the converged network. The dead basis the gauge settles into aligns with the coordinate axes at 0.889 against the matched Muon's 0.718, and its weights more often settle into a separated cluster. The gauge leaves the clean, axis-aligned singular structure that every readout later in this paper depends on; a weight-decay-matched Muon settles into a less aligned one. The effect tracks the task: it holds at the width grokking needs and fades only at the near-cliff size where grokking itself gives out.

\begin{figure}[t]
\centering
  \includegraphics[width=0.94\linewidth]{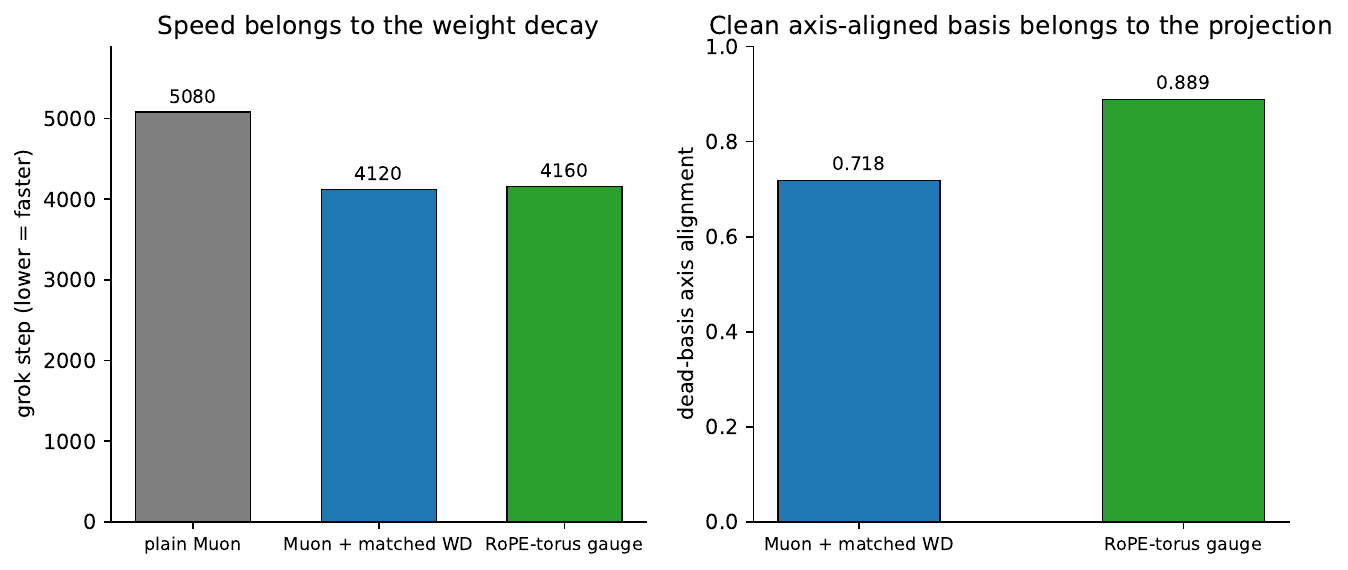}

\caption{Weight decay or projection, on the one-block modular-addition transformer (Muon base, 5 seeds). Left: a plain Muon handed the gauge's weight decay groks as fast as the gauge, so the speed-up is the regularization the equivariant update carries. Right: the converged dead basis aligns far more sharply with the coordinate axes under the gauge than under the matched-weight-decay Muon, the alignment the projection adds beyond the decay.}
\label{fig:wd_projection}
\end{figure}
 
\subsection{Matching the gauge to the architecture}
\label{ssec:rope}

The abelian gauges asked nothing of the projection: Adam's per-coordinate second moment already commutes with a sign flip or a rescaling. A rotation does not, so the rotation gauges make the projection earn its keep. On the grokking transformer, with the initialisation and batch stream held fixed, toggling the projection alone moves the query weights by 0.122 along the gauge mode, six times the motion a coarser scalar treatment of the same gauge produces. The projection does real work on the rotation gauge.

Which rotation depends on the architecture. Rotary attention carries a smaller symmetry than a plain head: only the torus of rotations that commute with the rotary embedding, $SO(2)^{d/2}$, leaves it invariant, so a gauge built for the full $O(d_\text{head})$ closes a symmetry the network does not have. Matching the gauge to the torus the architecture actually carries pays off: on the grokking task the torus-matched gauge groks to 0.977 final accuracy against a generic rotation's 0.955, and halves the seed-to-seed variance (Figure~\ref{fig:rope_matched}; Appendix~\ref{app:rope_ab}). The architecture fixes the gauge, and reading its symmetry correctly separates a reliable optimizer from a generic one that captures half the benefit. DDC carries this torus gauge into the language-model experiments of this paper.

\begin{figure}[t]
\centering
  \includegraphics[width=0.74\linewidth]{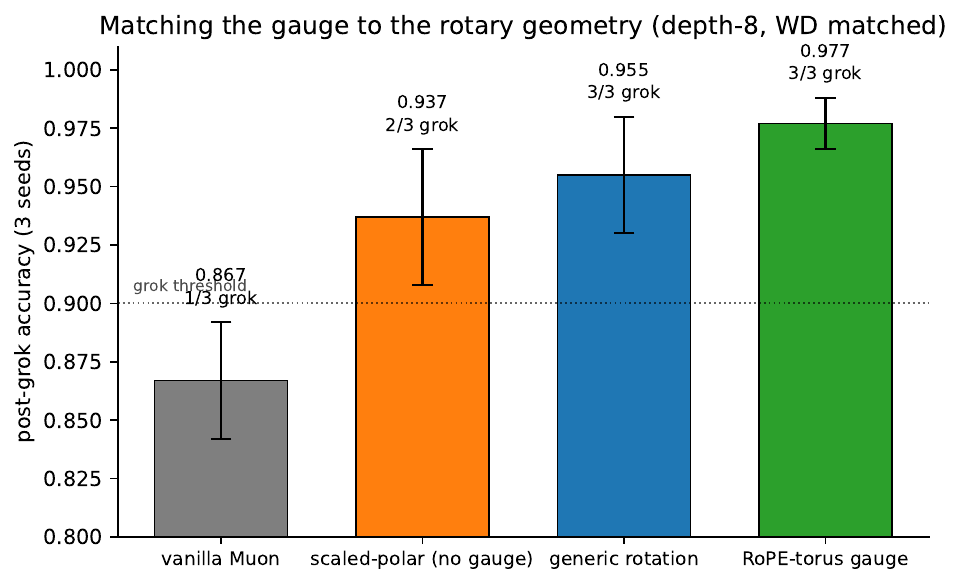}

\caption{Matching the rotation gauge to the rotary attention geometry. Post-grok accuracy on the depth-8 modular-addition transformer across the monotone decomposition from plain Muon, through the DDCMuon orthogonaliser, to a generic $O(d_\text{head})$ rotation and the RoPE-torus rotation, all at weight decay $1.0$ so the gauge axis is matched. The torus gauge, built for the symmetry the rotary embedding leaves free, groks highest ($0.977$) and halves the seed-to-seed variance of the generic rotation. Three seeds; error bars are one standard deviation, and the grok fraction is annotated above each bar.}
\label{fig:rope_matched}
\end{figure}
 
\subsection{Reliability at depth}
\label{ssec:at_scale}

The reliability the projection buys (\S\ref{ssec:wd_vs_projection}) becomes the whole story at depth. At depth 24 the RoPE-torus QK/VO rotation gauge groks ten of eleven seeds to full accuracy where plain Muon reaches none: vanilla generalises only partially, to around $0.70$, and never closes the gap to the grokked ceiling. The few-points edge of the small models widens into a near-total split, the gauge grokking on almost every seed and plain Muon on essentially none. Two cohorts carry it, a multi-condition sweep and a five-seed run, and the gauge reaches $0.937$ final accuracy against a vanilla $0.696$ (Figure~\ref{fig:d24_spaghetti}; Appendix~\ref{app:d24}). The deeper the network, the more a reliable optimizer earns its place.

\begin{figure}[t]
\centering
  \includegraphics[width=0.74\linewidth]{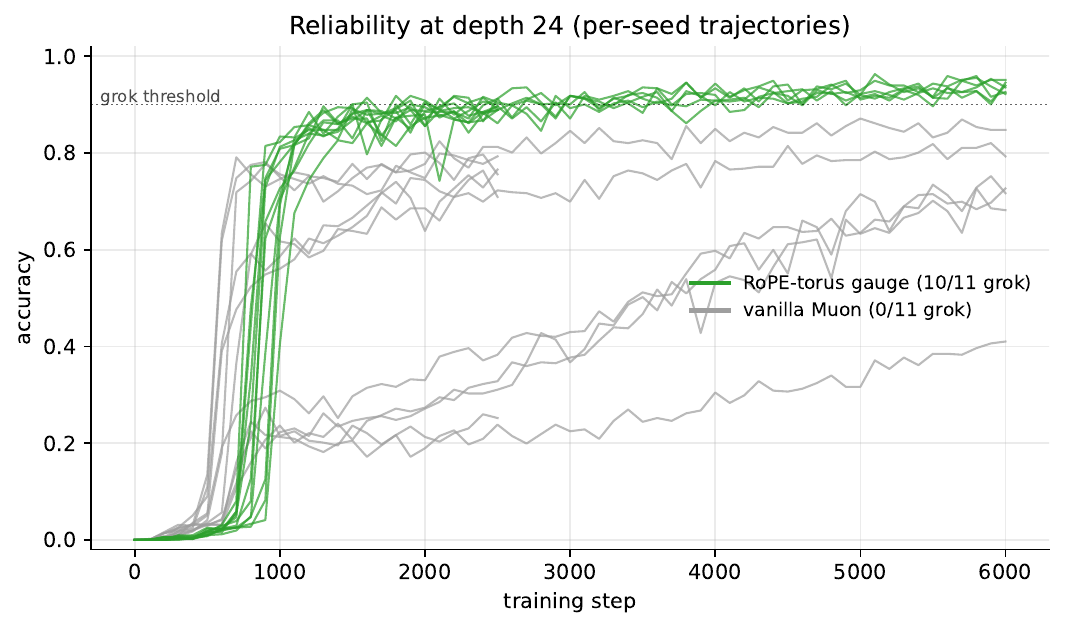}

\caption{Reliability at depth 24, per-seed accuracy trajectories pooled across a multi-condition sweep and a five-seed run (RoPE-torus QK/VO gauge on a Muon base, weight decay $1.0$). The gauge clears the grok threshold on ten of eleven seeds; plain Muon clears it on none, stalling in partial generalisation around $0.70$. The deeper the network, the more a reliable optimiser earns its place.}
\label{fig:d24_spaghetti}
\end{figure}
 
\subsection{The minimum the gauge reaches}
\label{ssec:curvature}

The RoPE-torus QK/VO rotation gauge reaches a less singular minimum, and a masking control shows the gain belongs to the gauge. On a separate depth-8 cohort, reading the true Fisher at convergence, the gauge's smallest eigenvalue sits an order of magnitude further off the singular floor than a plain Muon's ($5.8\times10^{-11}$ against $4.6\times10^{-12}$), at a lower validation loss ($2.638$ against $3.278$ bits per byte). A Muon arm that masks dead directions without the gauge washes out to the baseline, so the less degenerate minimum is specific to the gauge and not to dead-direction handling in general (Figure~\ref{fig:curvature}; Appendix~\ref{app:curvature}).

\begin{figure}[t]
\centering
  \includegraphics[width=0.92\linewidth]{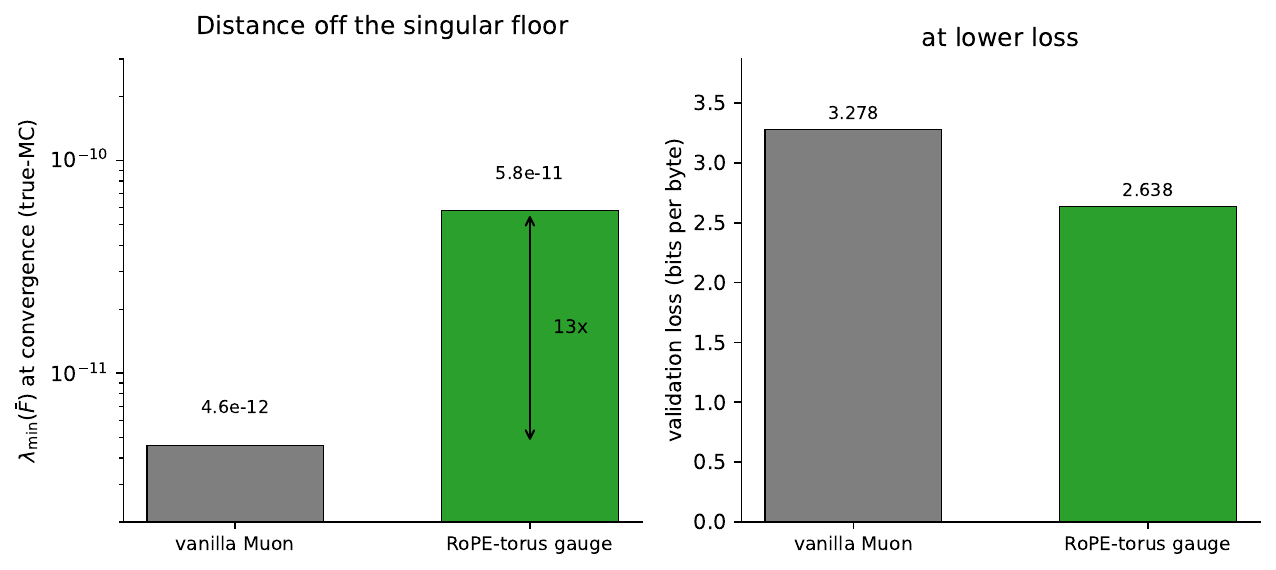}

\caption{The minimum the gauge reaches, on the depth-8 grok bench with the true-MC Fisher read at convergence (3 seeds). Left: the smallest quotient-Fisher eigenvalue sits an order of magnitude ($13\times$) further off the singular floor under the RoPE-torus gauge than under plain Muon, a measurably less degenerate minimum. Right: at a lower validation loss.}
\label{fig:curvature}
\end{figure}
 
\subsection{Where the composition holds and where it breaks}
\label{ssec:muon_boundary}

We prove equivariance for Adam, and composing the gauge with Muon inherits only part of it. Muon's orthogonalisation step, the DDCMuon orthogonaliser we run throughout, commutes with the gauge exactly and leaves the gauge orbit fixed, so the gauge still holds the trajectory on the quotient; the grokking and curvature gains above, all measured on a Muon base, come from that orbit-holding. What the Adam theorem also buys, a per-step metric conditioned in the gauge's body frame, does not carry over: reading the directional Fisher before and after one real step, the gauge leaves the curvature the optimizer feels where plain Muon leaves it. So on a Muon base the gauge keeps the orbit but adds no per-step conditioning. The orbit-holding still pays: beyond the grokking and curvature gains above, on the over-training language model \textsc{DDCMuon} settles at a lower validation loss ($4.78$ against $4.88$) and slightly higher validation accuracy ($0.262$ against $0.254$) than a textbook-NS5 Muon (Appendix~\ref{app:otr}), a better minimum rather than added over-training resistance, since Muon already controls the train-validation gap.

\begin{table}[h]
\centering
\small
\setlength{\tabcolsep}{8pt}\renewcommand{\arraystretch}{1.3}\begin{tabular}{@{}lcc@{}}
\toprule
What the gauge provides & \eqadam{} (Adam base) & \textsc{DDCMuon} (Muon base) \\
\midrule
Holds the gauge orbit (rate-readability on $\bar\Theta$) & exact & exact \\
Per-step metric conditioning in the body frame          & exact (Theorem~\ref{thm:ddcadam_rate}) & does not carry \\
Composes with the per-channel rescaling gauges          & yes & no \\
Composes with the QK/VO rotation gauges                 & yes & yes (scaled-polar) \\
\bottomrule
\end{tabular}
\caption{Where the gauge composition holds and where it breaks (\S\ref{ssec:muon_boundary}). Muon's polar factor commutes with a head rotation ($\mathrm{polar}(MO) = \mathrm{polar}(M)\,O$), so \textsc{DDCMuon} holds the gauge orbit exactly and the grokking and curvature gains follow from that orbit-holding. What the Adam theorem also buys, a per-step metric conditioned in the gauge's body frame, does not carry to a Muon base: reading the directional Fisher before and after one real step, the gauge leaves the curvature the optimizer feels where plain Muon leaves it. A per-channel rescale does not commute with orthogonalisation, so those gauges ride the Adam base.}
\label{tab:muon_boundary}
\end{table}
 
\subsection{Beyond grokking, a vision transformer}
\label{ssec:vit_breadth}

DDC carries to a real vision task. We train a vision transformer from scratch on ImageNet-100 into a compression regime, where its over-parametrised feed-forward width prunes spare units under weight decay. Against a matched-weight-decay AdamW, \eqadam{} with the ViT multiplicative gauges reaches lower validation loss ($1.71$ against $2.12$, a $0.40$-nat gap) with validation accuracy tied, and compresses far harder at the same weight decay: a few thousand dead units against a handful, and an activation singular value four times smaller. Plain AdamW needs twice the weight decay to compress at all. On real vision DDC lowers the loss and resists over-training, with accuracy unchanged (Figure~\ref{fig:vit_breadth}; Appendix~\ref{app:vit}).

\begin{figure}[t]
\centering
  \includegraphics[width=0.94\linewidth]{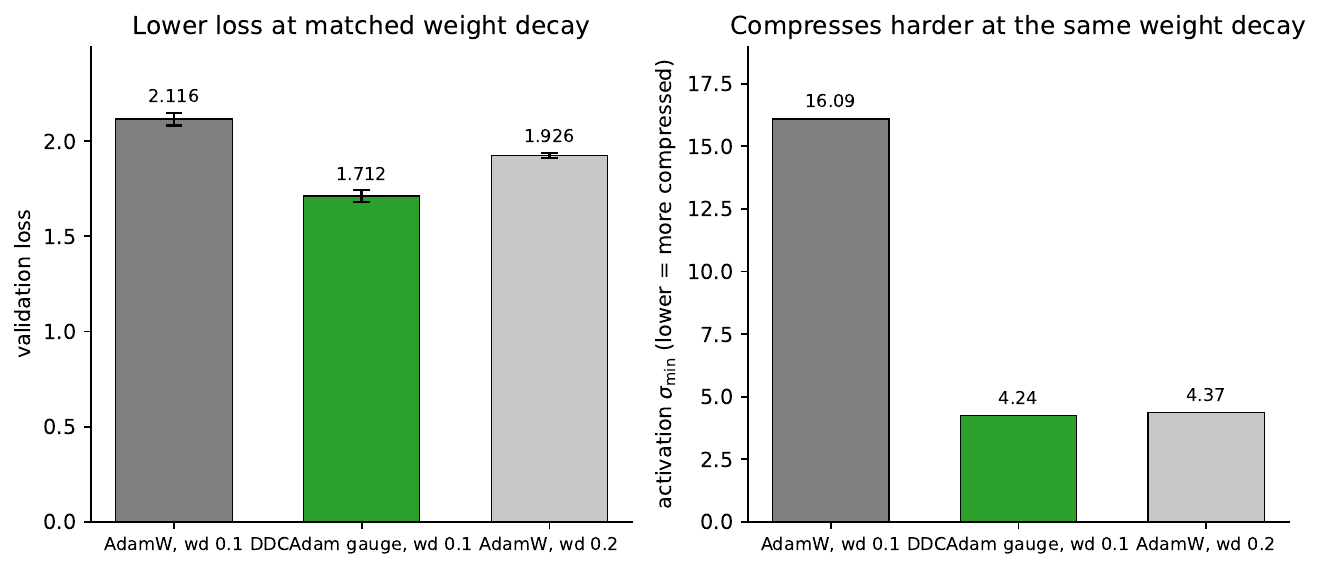}

\caption{Beyond grokking, a vision transformer trained from scratch on ImageNet-100 (3 seeds, true-MC Fisher). Left: against a matched-weight-decay AdamW the \eqadam{} gauge reaches lower validation loss. Right: it compresses far harder at the same weight decay, an activation singular value four times smaller, where a plain AdamW needs twice the weight decay to compress at all.}
\label{fig:vit_breadth}
\end{figure}
 
\subsection{Scope and limitations}
\label{ssec:limitations}

Three limits scope these claims. The Muon-base grokking speed-up traces to the bundled weight decay; the projection earns its place elsewhere, through reliability and the cleaner singular structure (\S\ref{ssec:wd_vs_projection}). The exact equivariance belongs to Adam; on a Muon base the gauge holds the orbit but its per-step conditioning does not carry over (\S\ref{ssec:muon_boundary}). And the gauge restores readability only on the canonical chart: it keeps the trajectory on the quotient where the rate is legible, but reading the full singular structure off canonical alignment, the order and multiplicity of a formed dead direction, needs machinery this paper does not carry and which we develop separately.

 \section{Discussion}
\label{sec:discussion}

\paragraph{The composition boundary selects a norm per block.} The boundary of \S\ref{ssec:muon_boundary} reflects a general rule: a gauge survives a change of base optimizer when the optimizer's update map commutes with the gauge action, and the gauge type decides whether it does. Muon~\citep{JordanBernsteinKaramcheti24} orthogonalises each update to the polar factor $\mathrm{polar}(M) = UV^\top$ of $M = U\Sigma V^\top$. A QK or VO rotation acts as $M \mapsto MO$, and $\mathrm{polar}(MO) = \mathrm{polar}(M)\,O$ commutes, so the orthogonalisation carries the rotation through. The deployed orthogonaliser is the scaled-polar variant (Appendix~\ref{app:pseudocode}): NorMuon's~\citep{LiLiu25NorMuon} per-neuron second-moment reduction would break the gauge (a per-row scale is not $O(d_{\rm head})$-invariant), so a single gauge-invariant scale on the polar factor replaces it, inheriting the equivariance exactly. A per-channel rescale acts as $M \mapsto \mathrm{diag}(c)\,M$, and $\mathrm{polar}(\mathrm{diag}(c)\,M) \neq \mathrm{diag}(c)\,\mathrm{polar}(M)$, because the per-channel scale spectrum is the freedom orthogonalisation removes. An orthogonalising base therefore admits the rotation gauges and rejects the per-channel ones; the rotation gauge together with the additive CE row-shift forms the Muon-compatible set.

The modular-norm framework~\citep{BernsteinNewhouse24, Pethick25Scion} reads this boundary cleanly. Each layer carries an operator norm and takes its update from the dual (linear-minimisation-oracle) map under that norm, Muon dualising the spectral (RMS-to-RMS) norm and a coordinate-wise rule the max-type norm. The norm that makes a block's update gauge-equivariant follows from the gauge-quotient Fisher metric on that block. A rotation-invariant quotient metric induces the spectral norm and selects Muon; a per-channel quotient metric induces a per-coordinate norm and selects the diagonal update \eqadam{} applies. Two blocks carrying different modular norms admit no single equivariant dualiser, which is the no-go above. The symmetry-compatible stack of \citet{LauSu26} reaches the same per-block assignment from the symmetry groups directly; the quotient-Fisher reading adds the metric that selects each norm and the rate theorem (Thm.~\ref{thm:ddcadam_rate}) that keeps the resulting trajectory readable on $\bar\Theta = \Theta/G$.

\paragraph{Architectural scope.} The constructions extend across modern transformer variants, and the language model of \S\ref{ssec:rate_readability} already runs several together: RMSNorm in place of LayerNorm, SwiGLU in place of the ReLU MLP, and grouped-query attention with rotary embeddings in place of standard multi-head attention. Table~\ref{tab:gauge_arch_scope} catalogs which gauge transfers to each variant.

\begin{table}[h]
\centering
\small
\setlength{\tabcolsep}{4pt}\renewcommand{\arraystretch}{1.3}\begin{tabularx}{\textwidth}{l YYYYY}
\toprule
Gauge & \shortstack{LN+ReLU\\+MHA} & RMSNorm & \shortstack{GELU\\MLP} & \shortstack{SwiGLU\\MLP} & \shortstack{MQA\\/ GQA} \\
\midrule
CE row-shift & analysed & transfers & transfers & transfers & transfers \\
LN scale ($\gamma\!\mapsto\!c\gamma$) & analysed & symmetry transfers; no LN-kernel dead direction\footnotemark[1] & transfers & transfers & transfers \\
QK rotation per head & analysed & transfers & transfers & transfers & orbit-reduced\footnotemark[2] \\
VO rotation per head & analysed & transfers & transfers & transfers & similar reduction \\
ReLU rescale & analysed & transfers if ReLU MLP & \emph{does not apply} & \emph{does not apply} & -- \\
SwiGLU rescale & not present & -- & -- & \emph{new gauge}: $(W_u, W_d)\!\mapsto\!(cW_u, c^{-1}W_d)$ on the gated branch & -- \\
\bottomrule
\end{tabularx}
\caption{Gauge inventory by architectural variant. The CE shift, LN/RMSNorm scale, and QK/VO rotation gauges carry across attention transformers; the ReLU rescale gauge applies only to positive-1-homogeneous MLPs, and SwiGLU networks admit a structurally analogous rescale gauge on the gated branch (Lemma~\ref{lem:swiglu_rescale_gauge}).}
\label{tab:gauge_arch_scope}
\end{table}
\footnotetext[1]{The post-final-norm centred activation covariance has an algebraic kernel in the $\gamma^{-1}/\|\gamma^{-1}\|$ direction for LayerNorm but not RMSNorm (Proposition 2 of~\citealt{Shirodkar2026LNKernel}; verified empirically on 8/8 LayerNorm and 0/5 RMSNorm pretrained transformers). The LN-scale symmetry ($\gamma \to c\gamma$, $W_{\rm next} \to c^{-1}W_{\rm next}$) exists in both; an RMSNorm network loses one of the two LayerNorm dead-direction sources, so removing the gauge handling costs it less.}
\footnotetext[2]{Multi-Query Attention shares $W_K$ across all heads, collapsing the QK rotation to a single $O(d_{\rm head})$ orbit rather than the per-head product $O(d_{\rm head})^{n_{\rm heads}}$ of standard MHA; Grouped-Query Attention shares within groups, giving $O(d_{\rm head})^{n_{\rm groups}}$. The construction transfers on the reduced orbit (Lemma~\ref{lem:mqa_gqa_rotation_gauge}), and the per-coordinate Adam leakage of Lemma~\ref{lem:adam_projection_bias_rotation} keeps the same matrix-valued form there.}

The depth-12 language model carries the CE shift, the RMSNorm scale, the SwiGLU rescale, and the RoPE-torus QK/VO rotation at once, and the rate-readability and over-training-resistance results of \S\ref{ssec:rate_readability} hold on that stack. The QK/VO rotation findings transfer universally with a modified orbit dimension under MQA/GQA; the ReLU rescale belongs to ReLU MLPs, with the SwiGLU rescale its analog on gated branches. The remaining empirical reach is the full weight-decay-by-gauge factorial at larger scale, on the LLaMA / Gemma class.

\paragraph{Cost.} The deployed path runs DDC on a Muon base. The DDCMuon orthogonaliser carries the QK/VO rotation gauge at a wall-clock comparable to plain Muon, within roughly $10$ to $15\%$ and at parity on some GPUs, since its matmul-only steps overlap the host/device pipeline rather than blocking it; an exact per-head SVD is $4.2\times$ more expensive at $d_{\rm head} = 8$. The rotation gauge adds a per-step horizontal projection and a body-frame basis that the optimizer recomputes only when the per-head Gram drifts past a threshold (three recomputes over 200 steps on the depth-24 configuration), for $2.6$ to $3.8\%$ per step over the no-gauge orthogonaliser against $40\%$ for recomputing the eigendecomposition every step. The adaptive trigger also holds numerically where a fixed recompute cadence diverges. The eigendecomposition cost grows as $n_{\rm layer} \cdot n_{\rm heads} \cdot d_{\rm head}^3$ while the matmul-dominated training cost grows faster with model size, so the relative gauge overhead shrinks at scale (Table~\ref{tab:cost}).

\begin{table}[h]
\centering
\small
\setlength{\tabcolsep}{6pt}\renewcommand{\arraystretch}{1.3}\begin{tabularx}{\textwidth}{@{}>{\raggedright\arraybackslash}X >{\raggedright\arraybackslash}X l@{}}
\toprule
Step component & Cost & Relative to \\
\midrule
\textsc{DDCMuon} scaled-polar orthogonaliser & $+10$ to $15\%$ (parity on some GPUs) & plain Muon \\
Exact per-head SVD orthogonaliser           & $4.2\times$ (at $d_\mathrm{head}=8$)    & scaled-polar \\
Rotation gauge, adaptive recompute (deployed) & $+2.6$ to $3.8\%$ per step           & no-gauge orthogonaliser \\
Rotation gauge, recompute every step         & $+40\%$ per step                      & no-gauge orthogonaliser \\
\bottomrule
\end{tabularx}
\caption{Per-step cost of the deployed Muon-base path. The scaled-polar orthogonaliser carries the QK/VO rotation gauge at near-Muon wall-clock because its matmul-only steps overlap the host/device pipeline; an exact per-head SVD does not. The adaptive trigger recomputes the body-frame basis only when the per-head Gram drifts past a threshold (three recomputes over 200 steps on the depth-24 configuration) and holds numerically where a fixed every-step cadence diverges. The eigendecomposition cost grows as $n_\mathrm{layer}\,n_\mathrm{heads}\,d_\mathrm{head}^3$ while the matmul-dominated training cost grows faster with model size, so the relative gauge overhead shrinks at scale.}
\label{tab:cost}
\end{table}
 
\paragraph{Limitations and open directions.} The empirical scope is set in \S\ref{ssec:limitations}: the Muon-base speed-up traces to the bundled weight decay, the exact equivariance belongs to the Adam base, and reading the dead-direction order off canonical alignment needs machinery developed separately. The construction also carries formal assumptions. It takes the $G$-action free and proper; orbit-type strata where the action degenerates (zero-rank weights, zero LN scale) need equivariant stratification not addressed here. Vanilla $L_2$ weight decay is not $G$-invariant under the multiplicative gauges, so the construction uses decoupled weight decay, which shrinks all components uniformly. The discrete-step bias $O(\eta^2 \kappa(P))$ of Theorem~\ref{thm:ddcadam_rate}(iv) holds for $\eta \ll \min(\tau_L, \tau / \kappa(P))$ (Lemma~\ref{lem:ddcadam_large_step_regime}, \S\ref{app:proofs:discrete_step_corrections}); the leading large-$\eta$ correction is $O(\eta^3 \kappa(P)^2 \|\nabla^2 L\|/\tau)$ and stays $G$-equivariant, while the full series at $\eta\kappa(P)/\tau \asymp 1$ remains uncharacterised (Remark~\ref{rem:ddcadam_large_step_open}). For the standard hyperparameter range ($\eta \le 10^{-2}$, $\beta_2 \ge 0.99$) the validity ratio stays two orders of magnitude inside the regime where the leading bound holds. The body-frame rotation construction is equivariant only to leading order: the eigenvector sign ambiguity in $\mathrm{eigh}(M)$ adds ${\sim}\,10^{-3}$ drift over 20 fp64 steps, against $10^{-12}$ for the per-head-scalar variant. Lemma~\ref{lem:body_frame_drift} (\S\ref{app:proofs:body_frame_drift}) bounds the drift by Davis-Kahan through the per-head Gram's minimum spectral gap, Remark~\ref{rem:body_frame_crossing} identifies eigenvalue crossings as the failure mode, and the periodic basis recompute mitigates it in implementation. The projection bias these constructions remove is given in closed form for the abelian multiplicative gauges by Lemma~\ref{lem:adam_projection_bias} (\S\ref{app:proofs:projection_bias}) and for the non-abelian rotation gauge by Lemma~\ref{lem:adam_projection_bias_rotation} (\S\ref{app:proofs:projection_bias_rotation}), where the leakage takes a matrix-valued form in $\mathfrak{so}(d_{\rm head})$ that the body-frame basis change eliminates.

\paragraph{Symmetry as a design variable.} A network's loss is blind to its own symmetries, and an adaptive optimizer inherits that blindness as a drift along the orbit it cannot resolve. Building the symmetry into the preconditioner removes the drift: conditioning in the horizontal frame and collapsing each orbit dimension to a single scalar holds the trajectory on the quotient $\bar\Theta$, where the optimizer settles into a less degenerate minimum and the singular-approach rate of \citet{GSDTheoryAnon} stays readable from the path it took. The orbit split is generic to the preconditioner, so the same construction lifts Adam, Muon, and Shampoo, and carries every architectural gauge the base optimizer commutes with. As models stack more symmetries at once, RMSNorm with SwiGLU with grouped-query attention in a single network, the gauge an optimizer respects increasingly decides both the minimum it reaches and how much of that minimum's geometry it leaves measurable.
 
\clearpage

\bibliographystyle{abbrvnat}

\clearpage
\appendix
\etocsettocdepth.toc{subsection}
{\renewcommand{\contentsname}{Appendix contents}\tableofcontents}
\bigskip

\section{Proofs}
\label{app:proofs}

\begin{proof}[Proof of Theorem~\ref{thm:ddcadam_rate}]
We argue the four parts in order.

\paragraph{(i) $G$-equivariance.}
The vertical and horizontal projectors $\Pi_V, \Pi_H$ are $G$-equivariant by construction: $V_{h\cdot\theta} = (dh)_\theta V_\theta$ since the orbit through $h \cdot \theta$ is $h \cdot (G \cdot \theta) = G \cdot \theta$, and the $\rho$-orthogonal complement is $G$-equivariant because $\rho$ is $G$-invariant. The orbit-collapsed $\hat v^V$ depends only on the orbit (by construction it is a per-orbit-dimension scalar averaged over the orbit), so $\hat v^V(h \cdot \theta) = \hat v^V(\theta)$. The horizontal $\hat v^H$ is computed in a $G$-equivariant horizontal frame: choose any horizontal frame $\{f_i(\theta)\}$ at $\theta$ that transports under $G$ as $f_i(h\cdot\theta) = (dh)_\theta f_i(\theta)$ (such a frame exists because $H$ is a $G$-equivariant subbundle); then the per-coordinate $\hat v^H_i$ in this frame is invariant under the action. Together,
\[
P(h\cdot\theta, (dh^{-1})^* g)
= \frac{\Pi_V (dh^{-1})^* g}{\sqrt{\hat v^V(h\cdot\theta)} + \epsilon}
+ \frac{\Pi_H (dh^{-1})^* g}{\sqrt{\hat v^H(h\cdot\theta)} + \epsilon}
= (dh)_\theta P(\theta, g),
\]
where the last equality uses $\Pi_V (dh^{-1})^* = (dh^{-1})^* \Pi_V$ and the cotangent-tangent dualisation by $\rho$.

\paragraph{(ii) Projection.}
By (i), for any $h \in G$, the flow $\dot\theta = -P(\theta, \nabla L)$ at $h\cdot\theta$ has tangent vector $-P(h\cdot\theta, \nabla L(h\cdot\theta)) = -(dh)_\theta P(\theta, \nabla L(\theta))$, where the second equality uses $G$-equivariance of $\nabla L$ (which follows from $G$-invariance of $L$ and $G$-invariance of $\rho$, the metric used to identify cotangent and tangent vectors). Hence the flow on $\Theta^\circ$ is $G$-equivariant, and projects under $\pi$ to a well-defined flow on $\bar\Theta$ (the standard Riemannian-submersion descent of a $G$-invariant flow; cf.~\citealp{AbsilMahonySepulchre08}). The projected flow at $\bar\theta = \pi(\theta)$ satisfies $\dot{\bar\theta} = d\pi_\theta \dot\theta = -d\pi_\theta P(\theta, \nabla L)$. Because $\Pi_V g \in V_\theta = \ker d\pi_\theta$, the vertical summand of $P$ projects to zero, and the horizontal summand projects to a vector in $T_{\bar\theta} \bar\Theta$ that, in the quotient frame induced by the horizontal frame, equals $-(\Pi_H \nabla L) / (\sqrt{\hat v^H} + \epsilon)$. This is preconditioned gradient descent of $\bar L$ in the quotient metric.

\paragraph{(iii) Rate readability.}
The projected flow on $\bar\Theta$ is preconditioned gradient flow of $\bar L$. By assumption the trajectory approaches $\bar\theta_0$ along the canonical-aligned horizontal direction $\bar u$ of KL order $k \ge 2$. Two cases:

\emph{Case 1: $\hat v^H \to v_\infty^H$ along the trajectory} (a constant nonzero limit, generic in finite-dimensional Adam runs after the second-moment EMA equilibrates). Then the preconditioner converges to a constant linear map $1/\sqrt{v_\infty^H + \epsilon}$ acting on $H_\theta$; rescaling time by this constant transforms the preconditioned flow into vanilla horizontal gradient flow of $\bar L$. The rate $\bar u^\top \bar\fisher(\bar\theta(t))\, \bar u = \Theta(\bar t(t)^{2(k-1)})$ then follows from the SGD-quotient corollary of~\citep{GSDTheoryAnon} applied to the rescaled trajectory; the rate exponent is invariant under the time rescaling.

\emph{Case 2: $\hat v^H \to 0$ or $\hat v^H \to \infty$ along the trajectory} (degenerate). The rate exponent is unchanged: the preconditioner only rescales the trajectory's speed, not its direction or its asymptotic geometric relation to $\bar\theta_0$. Formally, if $\hat v^H(t) = \Theta(\bar t(t)^{\alpha})$ for some exponent $\alpha$, the time rescaling is $d\bar s = (1 + \Theta(\bar t^\alpha))\, dt$; this changes the wall-clock rate of approach but not the geometric rate $\bar u^\top \bar\fisher(\bar\theta) \bar u$ versus $\bar t$.

\paragraph{(iv) Discrete-step bias.}
The Euler-Maruyama bias for a preconditioned flow $\dot\theta = -P(\theta) g$ versus its $\eta$-step discretisation $\theta_{t+1} = \theta_t - \eta P(\theta_t) g_t$ is, to leading order, $\frac{\eta^2}{2} (\partial_\theta P)\, P g + O(\eta^3)$, where $\partial_\theta P$ is the Jacobian of $P$ in $\theta$ (the $O(\eta^2)$ weak-approximation bias of the stochastic-modified-equations framework, \citealp{LiTaiE19}; cf.~\citealp{Pesme2021} for the SGD analysis). The bound $\kappa(P)$ on the conditioning of $P$ controls the magnitude of this term: $\|(\partial_\theta P)\, P g\| \le \kappa(P) \cdot \|g\|$, giving the stated $O(\eta^2 \kappa(P))$. Vertical-summand contributions to this bias project to zero under $\pi$ by (ii), so the bias is purely horizontal and inherited by the projected discrete trajectory.
\end{proof}

\begin{remark}[Relationship to the theory paper's \eqadam{} corollary]
\label{rem:ddcadam_vs_sgd}
\citet{GSDTheoryAnon} establish (iii) for SGD on a $G$-invariant Riemannian metric (their Corollary~79) and state the \eqadam{} versions directly: the preconditioner's $G$-equivariance is their Proposition~83 and the projected-trajectory rate readout is their Corollary~86. Theorem~\ref{thm:ddcadam_rate} is the full construction those results reference. It extends the equivariance from the metric inverse to a learnt orbit-averaged second-moment preconditioner (Adam-style adaptivity while preserving $G$-equivariance), and adds the discrete-step bias (iv) and the non-abelian rotation gauge of Corollary~\ref{cor:vmode_equivariance}, neither of which the theory paper treats.
\end{remark}

\begin{remark}[Why coordinate-wise Adam fails this theorem]
\label{rem:adam_fails}
Standard Adam uses the coordinate-wise preconditioner $1/\sqrt{\hat v_i}$ with $\hat v_i$ tracked per parameter coordinate $\theta_i$. Coordinate-wise tracking is $G$-equivariant only when $G$ acts by coordinate permutations; for the architectural Lie groups in this paper (CE shift, ReLU rescaling, LN scale), the $G$-action mixes coordinates. Two distinct coordinates lying on the same orbit will accumulate different $\hat v_i$ values (because their gradient histories differ even though the gradients are $G$-equivariant), and the coordinate-wise rescaling breaks the $G$-equivariance assumed in (i). This recovers the Adam-non-descent failure mode characterised in~\citep{GSDTheoryAnon}.
\end{remark}
 
\subsection{Per-coordinate Adam's vertical leakage on the abelian rescale gauge}
\label{app:proofs:projection_bias}

The discussion in \S\ref{sec:discussion} flags as open a closed-form quantification of the projection bias. Lemma~\ref{lem:adam_projection_bias} closes it for the abelian multiplicative case, giving an explicit per-step leakage of the per-coordinate Adam preconditioner along the orbit-tangent direction; Remark~\ref{rem:r2_gap_mechanism} reads off the trajectory consequence, the orbit drift that pollutes the on-trajectory rate read and that the equivariant construction removes (Appendix~\ref{app:dln_drift}). Remark~\ref{rem:bias_generalisation} extends the argument verbatim to the per-channel LN-scale gauge and the chained ReLU rescale gauge; the non-abelian rotation gauge requires a separate spectral argument and is not addressed by this lemma. The conserved quantity the leakage perturbs, the rescale balancedness $\|W_1\|_F^2 - \|W_2\|_F^2$, and the fact that gradient flow preserves it, are due to \citet{AroraCohenHazan18, KuninSagastuyBrenaGanguli21}; that a per-coordinate adaptive second moment breaks the equivariance a per-vector one preserves is the observation of \citet{LingSharpJacobson22}. What the lemma adds is the closed form for the residual leakage at canonical alignment (case (iii)), an empirical covariance of the preconditioner with $g \odot W$.

\begin{lemma}[Per-coordinate Adam's vertical leakage on an abelian multiplicative gauge]
\label{lem:adam_projection_bias}
Let $G = \reals^+$ act on a pair of weight tensors $(W_1, W_2) \in \reals^{d_{\rm h} \times d_{\rm in}} \times \reals^{d_{\rm out} \times d_{\rm h}}$ by the rescale $\sigma_c \cdot (W_1, W_2) = (c W_1, c^{-1} W_2)$, $c \in \reals^+$. Let $\rho$ be the standard Frobenius metric on the joint parameter space. The vertical (orbit-tangent) direction at $(W_1, W_2)$ is $v := (W_1, -W_2) / \sqrt{\|W_1\|_F^2 + \|W_2\|_F^2}$, and the horizontal complement $H_{(W_1, W_2)}$ consists of all $(g_1, g_2)$ with $\langle (g_1, g_2), v\rangle_\rho = 0$, i.e.,
\[
\frac{\tr(g_1^\top W_1)}{\|W_1\|_F^2 + \|W_2\|_F^2} \;=\; \frac{\tr(g_2^\top W_2)}{\|W_1\|_F^2 + \|W_2\|_F^2}, \qquad \text{equivalently}\qquad \tr(g_1^\top W_1) \;=\; \tr(g_2^\top W_2).
\tag{H}\label{eq:horizontality}
\]
Let $P : T \Theta^\circ \to T \Theta^\circ$ be a positive-definite \emph{coordinate-wise diagonal} preconditioner, i.e.,
\[
P \cdot (g_1, g_2) \;=\; (P_1 \odot g_1,\ P_2 \odot g_2),
\]
with $P_1 \in \reals^{d_{\rm h} \times d_{\rm in}}$, $P_2 \in \reals^{d_{\rm out} \times d_{\rm h}}$ acting elementwise. Define the \emph{vertical leakage} of $P$ at $(W_1, W_2)$ in the direction of horizontal $g \in H_{(W_1, W_2)}$ as
\[
\mathrm{leak}(P, g) \;:=\; \langle P g, v\rangle_\rho \;=\; \frac{\langle P_1 \odot g_1, W_1\rangle_F - \langle P_2 \odot g_2, W_2\rangle_F}{\sqrt{\|W_1\|_F^2 + \|W_2\|_F^2}}.
\]
Then:
\begin{enumerate}[label=(\roman*)]
\item (\textbf{SGD: zero leakage.}) If $P = I$, then $\mathrm{leak}(I, g) = 0$ for all $g \in H$, by~\eqref{eq:horizontality}.
\item (\textbf{Per-tensor preconditioner: leakage vanishes only at canonical alignment.}) If $P_1 = p_1 \cdot \mathbf{1}$ and $P_2 = p_2 \cdot \mathbf{1}$ are constant within each tensor (per-tensor scalar preconditioner), then for all $g \in H$,
\[
\mathrm{leak}(P, g) \;=\; \frac{(p_1 - p_2)\, \tr(g_1^\top W_1)}{\sqrt{\|W_1\|_F^2 + \|W_2\|_F^2}}.
\]
Adam's per-tensor effective rate at the asymptotic regime where $\hat v_i \to v_\infty^{(i)}$ within each tensor satisfies $p_i \asymp 1 / (\|W_{3-i}\|_F \cdot \|M - M^*\|_F)$ (the typical magnitude of $\nabla L$ on tensor $i$ in the rate regime). This gives $p_1 = p_2$ \emph{iff} $\|W_1\|_F = \|W_2\|_F$, the canonical-aligned slice of the orbit.
\item (\textbf{Per-coordinate preconditioner: residual leakage at canonical alignment.}) If $P$ is genuinely coordinate-wise (not constant within tensors), then even when $\|W_1\|_F = \|W_2\|_F$ the vertical leakage is generically nonzero, with magnitude
\[
\bigl|\mathrm{leak}(P, g)\bigr| \;\ge\; \frac{\bigl|\mathrm{Cov}_{\rm e}(P_1, g_1 \odot W_1)\bigr| + \bigl|\mathrm{Cov}_{\rm e}(P_2, g_2 \odot W_2)\bigr|}{\sqrt{\|W_1\|_F^2 + \|W_2\|_F^2}} \cdot d_{\rm tensor},
\]
where $\mathrm{Cov}_{\rm e}$ is the empirical covariance over the tensor entries and $d_{\rm tensor}$ is the entry count of the relevant tensor. The bound is non-trivial whenever the coordinate-wise $P_i$ vary across entries of a single tensor while $g_i \odot W_i$ is correlated with that variation, the generic situation when $\hat v$ accumulates a non-uniform per-coordinate magnitude history.
\end{enumerate}
\end{lemma}

\begin{proof}
(i) The horizontality condition~\eqref{eq:horizontality} is exactly $\tr(g_1^\top W_1) = \tr(g_2^\top W_2)$, hence $\langle (g_1, g_2), v\rangle_\rho = 0$, hence $\mathrm{leak}(I, g) = 0$.

(ii) When $P_i = p_i \cdot \mathbf{1}$, the elementwise products simplify to $P_i \odot g_i = p_i\, g_i$. Substituting into the leakage formula,
\[
\mathrm{leak}(P, g) = \frac{p_1 \langle g_1, W_1\rangle_F - p_2 \langle g_2, W_2\rangle_F}{\sqrt{\|W_1\|_F^2 + \|W_2\|_F^2}} = \frac{(p_1 - p_2)\, \tr(g_1^\top W_1)}{\sqrt{\|W_1\|_F^2 + \|W_2\|_F^2}}
\]
using horizontality $\tr(g_1^\top W_1) = \tr(g_2^\top W_2)$ in the second equality. For the asymptotic per-tensor rate of Adam: in the canonical 2-layer linear setup with squared loss $L = \tfrac{1}{2N} \|W_2 W_1 X - Y\|_F^2$ and $W_2 W_1 = M$, the gradients are
\[
\nabla_{W_1} L = W_2^\top (M - M^*),\qquad \nabla_{W_2} L = (M - M^*) W_1^\top,
\]
so $\|\nabla_{W_1} L\|_F \le \|W_2\|_F \cdot \|M - M^*\|_F$ and symmetrically. After $\hat v$ equilibrates near the minimum, $p_1 = 1/\sqrt{\hat v_1 + \epsilon} \asymp 1 / \mathbb{E} \|\nabla_{W_1} L\|_F \asymp 1 / (\|W_2\|_F \cdot \bar t)$ where $\bar t := \|M - M^*\|_F$, and symmetrically $p_2 \asymp 1 / (\|W_1\|_F \cdot \bar t)$. Hence $p_1 = p_2 \iff \|W_1\|_F = \|W_2\|_F$.

(iii) Write the leakage as
\[
\mathrm{leak}(P, g) \cdot \sqrt{\|W_1\|_F^2 + \|W_2\|_F^2} = \sum_{ij}^{(W_1)} (P_1)_{ij} (g_1)_{ij} (W_1)_{ij} \;-\; \sum_{ij}^{(W_2)} (P_2)_{ij} (g_2)_{ij} (W_2)_{ij}.
\]
Each sum is an inner product of three tensors. Decompose $P_1 = \bar p_1 \cdot \mathbf{1} + \tilde P_1$ where $\bar p_1 := \mathrm{mean}_{ij}(P_1)_{ij}$ and $\tilde P_1$ has zero mean. The constant-per-tensor part contributes $\bar p_1 \tr(g_1^\top W_1)$, which by~(ii) and horizontality cancels against the constant-per-tensor part of the $W_2$ sum at $\|W_1\|_F = \|W_2\|_F$. The residual is
\[
\sum_{ij} (\tilde P_1)_{ij} (g_1)_{ij} (W_1)_{ij} - \sum_{ij} (\tilde P_2)_{ij} (g_2)_{ij} (W_2)_{ij}.
\]
For each tensor, $\sum_{ij} (\tilde P_i)_{ij} (g_i \odot W_i)_{ij} = d_{\rm tensor} \cdot \mathrm{Cov}_{\rm e}(P_i,\, g_i \odot W_i)$ where $\mathrm{Cov}_{\rm e}$ is the empirical covariance over the $d_{\rm tensor}$ tensor entries (the mean-zero correction makes this an entry-wise covariance). Triangle inequality gives the lower bound. Genericity: for Adam, $P_i = 1/(\sqrt{\hat v_i} + \epsilon)$ where $\hat v_i$ is the per-coordinate EMA of squared gradients; the variance of $\hat v_i$ across tensor entries reflects the non-uniformity of the gradient history per entry, which is non-zero whenever the gradient is not row- or column-constant within the tensor, the generic situation in deep nets.
\end{proof}

\begin{remark}[Trajectory consequence: path-noise from orbit drift]
\label{rem:r2_gap_mechanism}
Apply Lemma~\ref{lem:adam_projection_bias} along an Adam trajectory in the rate regime of Theorem~\ref{thm:ddcadam_rate}. At step $t$, the Adam update $\eta P g_t$ (with $g_t = \nabla L$ horizontal by gauge invariance of $L$) has horizontal component $\eta \Pi_H P g_t$ and vertical component
\[
\eta \cdot \mathrm{leak}(P_t, g_t) \cdot v_t,
\]
which moves the iterate within the orbit (the gauge coordinate $\varphi := \log(\|W_1\|_F / \|W_2\|_F)$ drifts). After $T$ steps the cumulative drift in $\varphi$ is
\[
\Delta \varphi_T \;=\; \eta \sum_{t=1}^T \mathrm{leak}(P_t, g_t) / \|v_t\|_{\rho^*}^2 + O(\eta^2),
\]
where the lower-order terms come from the second-order Adam-momentum interaction and are subleading on the rate-test horizon. By Lemma~\ref{lem:adam_projection_bias}~(iii), the per-step leakage does not vanish at canonical alignment, so $\mathrm{Var}(\Delta \varphi_T)$ grows at least linearly in $T$ (assuming weakly correlated per-step leakages, the standard model under noisy-gradient Adam).

The probe observable $u^\top \fisher u$, the directional Fisher along a fixed canonical direction $u$, is fixed in the $G$-canonical frame at $\varphi = 0$; at $\varphi \neq 0$ the orbit-translate of the canonical $u$ samples a different empirical Fisher direction. Linearising,
\[
\log u^\top \fisher u(\bar t, \varphi) \;=\; \underbrace{2(k-1) \log \bar t + \mathrm{const}}_{\text{rate term}} \;+\; \underbrace{\beta \cdot \varphi + O(\varphi^2)}_{\text{drift correction}},
\]
where $\beta$ is a tensor-shape-dependent coefficient computable from $\partial_\varphi (u^\top \fisher u)|_{\bar t}$ (it vanishes only when $u$ aligns with a $\varphi$-eigenvector of the local Fisher, which is generic-zero in the canonical frame). The empirical log-log fit therefore separates into
\begin{align*}
\mathbb{E}[\log u^\top \fisher u \mid \log \bar t] &= 2(k-1) \log \bar t + \mathrm{const} \quad &&\text{(unbiased rate slope)} \\
\mathrm{Var}[\log u^\top \fisher u \mid \log \bar t] &= \beta^2 \cdot \mathrm{Var}(\varphi \mid \bar t) + O(\eta^2) \quad &&\text{(path noise; absent for SGD).}
\end{align*}
The expectation identity leaves the rate slope asymptotically unbiased; the variance term is the path-noise the on-trajectory probe acquires under Adam, present through the leading $\beta^2 \mathrm{Var}(\varphi)$ contribution, absent for SGD (zero leakage by~(i)), and suppressed for \eqadam-frozen (zero \emph{vertical} leakage by construction, with a residual horizontal-summand term of order $\eta^2$ from the discrete-step bias in Theorem~\ref{thm:ddcadam_rate}~(iv)). On a finite trajectory this path-noise perturbs the on-trajectory slope away from the clean singular-minimum rate; the leakage separates the optimizers directly as the per-coordinate gauge drift on a diagonal network, AdamW $9.54$ against \eqadam{} $3.1 \times 10^{-6}$ (Appendix~\ref{app:dln_drift}), and as the rate-readability profile at language-model scale (\S\ref{ssec:rate_readability}).
\end{remark}

\begin{remark}[Generalisation to other abelian multiplicative gauges]
\label{rem:bias_generalisation}
The proof extends verbatim to $G = (\reals^+)^k$ acting per-channel (the LN-scale gauge of \S\ref{ssec:containment}) by replacing the single radial coordinate $\varphi = \log(\|W_1\|_F / \|W_2\|_F)$ with the $k$-vector $\boldsymbol{\varphi} = (\log(\|W_1[:, j]\|_2 / \|W_2[:, j]\|_2))_{j=1}^k$ and the scalar leakage with a $k$-vector. The horizontality condition~\eqref{eq:horizontality} is replaced by $k$ per-channel constraints; the per-coordinate Adam preconditioner has residual leakage in each channel by the same argument. For the chained ReLU rescale gauge $G = (\reals^+)^{L-1}$ on a length-$L$ chain, the leakage decomposes into $L-1$ pairwise leakages (one per adjacent layer pair), each computed by the lemma. The non-abelian rotation gauge $G = O(d_{\rm head})$ on attention QK and VO weights does \emph{not} fall under this lemma's hypothesis (the action is not coordinate-wise diagonal); its analogue requires a separate covariance-of-rotation-spectrum argument, the body-frame eigenvector-sign drift discussed in the limitations of \S\ref{sec:discussion} is the leading non-abelian instance.
\end{remark}
 
\subsection{The non-abelian rotation-gauge analogue}
\label{app:proofs:projection_bias_rotation}

Lemma~\ref{lem:adam_projection_bias_rotation} extends the leakage analysis of Lemma~\ref{lem:adam_projection_bias} to the per-head $O(d_{\rm head})$ rotation gauge of \S\ref{ssec:gauges}. The structural difference from the abelian case is that the leakage is matrix-valued in $\mathfrak{so}(d_{\rm head})$ rather than scalar (the orbit has dimension $d_{\rm head}(d_{\rm head}-1)/2$). \citet{SilversteinKuninShyam26} derive the conserved angular momentum, the $\mathfrak{so}(d_{\rm head})$ Noether charge, of this same attention rotation symmetry and show it obstructs descent; the lemma here gives the complementary quantity, the closed-form rate at which a per-coordinate Adam preconditioner leaks along that orbit. The RoPE-torus gauge of \S\ref{ssec:rope} is the centraliser $SO(2)^{d_{\rm head}/2}$ of the rotary embedding inside $O(d_{\rm head})$; its leakage and equivariance are the restriction of this lemma and Corollary~\ref{cor:vmode_equivariance} to that subgroup, with the closed-form per-plane projection of Appendix~\ref{app:pseudocode} in place of the Lyapunov solve. Remark~\ref{rem:rotation_vs_abelian_leakage} reads off the formal reason that the four second-moment (\texttt{v\_mode}) constructions of \S\ref{ssec:gauges} eliminate the leakage to differing degrees: \texttt{per\_head\_scalar} reduces case~(iii) to case~(ii); \texttt{per\_head\_matrix} absorbs the residual $\mathrm{Asym}(S)$ into a non-diagonal head-local preconditioner; \texttt{body\_frame} kills the leading per-tensor part by the eigendecomposition property, leaving only the eigenvector-sign drift.

\begin{lemma}[Per-coordinate Adam's vertical leakage on the rotation gauge]
\label{lem:adam_projection_bias_rotation}
Let $G = O(d_{\rm head})$ act on a single attention head $(W_Q, W_K) \in \reals^{m \times d_{\rm head}} \times \reals^{m \times d_{\rm head}}$ (with $m := d_{\rm model}$, $d := d_{\rm head}$) by the right-action $\sigma_O \cdot (W_Q, W_K) = (W_Q O, W_K O)$, $O \in O(d)$. The infinitesimal action at $(W_Q, W_K)$ is generated by antisymmetric matrices: the vertical (orbit-tangent) subspace is
\[
V_{(W_Q, W_K)} = \{(W_Q A,\ W_K A) : A \in \mathfrak{so}(d)\}, \qquad \dim V = d(d-1)/2,
\]
and the Frobenius-orthogonal horizontal complement consists of all $(g_Q, g_K)$ with
\[
g_Q^\top W_Q + g_K^\top W_K \;\in\; \mathrm{Sym}(\reals^{d \times d}), \qquad \text{equivalently} \qquad \mathrm{Asym}\bigl(g_Q^\top W_Q + g_K^\top W_K\bigr) = 0,
\tag{H$_{\rm rot}$}\label{eq:horizontality_rot}
\]
where $\mathrm{Asym}(M) := (M - M^\top)/2$. (The analogous statement for the VO gauge replaces $W_K$ with $W_O^\top$.)

For a coordinate-wise diagonal preconditioner $P$ with $P \cdot (g_Q, g_K) = (P_Q \odot g_Q, P_K \odot g_K)$, $P_Q, P_K \in \reals^{m \times d}$, define the \emph{rotation-gauge vertical leakage} as the $\mathfrak{so}(d)$-valued operator
\[
\mathrm{leak}_{\rm rot}(P, g) \;:=\; \mathrm{Asym}\bigl( (P_Q \odot g_Q)^\top W_Q + (P_K \odot g_K)^\top W_K \bigr) \;\in\; \mathfrak{so}(d).
\]
Then:
\begin{enumerate}[label=(\roman*)]
\item (\textbf{SGD: zero leakage.}) If $P = I$, then $\mathrm{leak}_{\rm rot}(I, g) = \mathrm{Asym}(g_Q^\top W_Q + g_K^\top W_K) = 0$ for all horizontal $g$, by~\eqref{eq:horizontality_rot}.
\item (\textbf{Per-tensor preconditioner.}) If $P_Q = p_Q \mathbf{1}$ and $P_K = p_K \mathbf{1}$, then for all horizontal $g$,
\[
\mathrm{leak}_{\rm rot}(P, g) \;=\; (p_Q - p_K) \cdot \mathrm{Asym}(g_Q^\top W_Q).
\]
Vanishing requires $p_Q = p_K$ \emph{or} $\mathrm{Asym}(g_Q^\top W_Q) = 0$. The latter is a generic-zero codimension-$d(d-1)/2$ condition on $g$ given $W_Q$, so the practical condition is $p_Q = p_K$, which holds at canonical alignment.
\item (\textbf{Per-coordinate preconditioner.}) Decompose $P_X = \bar p_X \mathbf{1} + \tilde P_X$ with $\bar p_X := \mathrm{mean}_{ij}(P_X)_{ij}$ and zero-mean residual $\tilde P_X$, $X \in \{Q, K\}$. Then
\[
\mathrm{leak}_{\rm rot}(P, g) \;=\; (\bar p_Q - \bar p_K) \cdot \mathrm{Asym}(g_Q^\top W_Q) \;+\; \mathrm{Asym}\bigl( S^{(Q)} + S^{(K)} \bigr),
\]
where $S^{(X)} \in \reals^{d \times d}$ is the column-pair covariance tensor with entries
\[
\bigl(S^{(X)}\bigr)_{rs} \;=\; \sum_{i=1}^{m} (\tilde P_X)_{ir}\, (g_X)_{ir}\, (W_X)_{is} \;=\; m \cdot \mathrm{Cov}_{i}\bigl((\tilde P_X)_{i,r},\, (g_X)_{i,r}\,(W_X)_{i,s}\bigr).
\]
The first term vanishes when the per-tensor mean preconditioners satisfy $\bar p_Q = \bar p_K$ (the per-tensor canonical-alignment condition); the second term is the residual rotation-gauge leakage and is generically nonzero whenever the column-wise variation of $\tilde P_X$ is correlated with the column-pair structure of $(g_X)_{i,r}\,(W_X)_{i,s}$. The Frobenius norm bound is
\[
\|\mathrm{leak}_{\rm rot}(P, g)\|_F \;\ge\; \bigl| (\bar p_Q - \bar p_K) \bigr| \cdot \|\mathrm{Asym}(g_Q^\top W_Q)\|_F \;-\; \|\mathrm{Asym}(S^{(Q)} + S^{(K)})\|_F,
\]
and the residual term scales with the per-coordinate variance of $\hat v_X^{1/2}$ across the columns of $W_X$ in the Adam case.
\end{enumerate}
\end{lemma}

\begin{proof}
(i) Direct from~\eqref{eq:horizontality_rot}: $\mathrm{leak}_{\rm rot}(I, g) = \mathrm{Asym}(g_Q^\top W_Q + g_K^\top W_K) = 0$.

(ii) When $P_X = p_X \mathbf{1}$, the elementwise products satisfy $(P_X \odot g_X) = p_X g_X$. Hence
\[
(P_Q \odot g_Q)^\top W_Q + (P_K \odot g_K)^\top W_K = p_Q\, g_Q^\top W_Q + p_K\, g_K^\top W_K.
\]
Taking $\mathrm{Asym}$:
\begin{align*}
\mathrm{leak}_{\rm rot}(P, g)
&= p_Q\, \mathrm{Asym}(g_Q^\top W_Q) + p_K\, \mathrm{Asym}(g_K^\top W_K) \\
&= p_Q\, \mathrm{Asym}(g_Q^\top W_Q) - p_K\, \mathrm{Asym}(g_Q^\top W_Q) \\
&= (p_Q - p_K)\, \mathrm{Asym}(g_Q^\top W_Q),
\end{align*}
where the second equality uses horizontality $\mathrm{Asym}(g_Q^\top W_Q + g_K^\top W_K) = 0 \Rightarrow \mathrm{Asym}(g_K^\top W_K) = -\mathrm{Asym}(g_Q^\top W_Q)$.

(iii) Substituting $P_X = \bar p_X \mathbf{1} + \tilde P_X$ into $(P_X \odot g_X)^\top W_X$:
\[
(P_X \odot g_X)^\top W_X = \bar p_X (g_X^\top W_X) + (\tilde P_X \odot g_X)^\top W_X.
\]
The first part contributes $(\bar p_Q - \bar p_K) \cdot \mathrm{Asym}(g_Q^\top W_Q)$ by~(ii). The second part is the matrix $S^{(X)}$ in the lemma statement: its $(r, s)$-entry is
\[
(S^{(X)})_{rs} = ((\tilde P_X \odot g_X)^\top W_X)_{rs} = \sum_{i=1}^m (\tilde P_X)_{ir} (g_X)_{ir} (W_X)_{is}.
\]
The empirical covariance form follows because $\sum_i (\tilde P_X)_{ir} = 0$ by zero-mean: $\sum_i (\tilde P_X)_{ir} \cdot c = 0$ for any constant $c$, so $\sum_i (\tilde P_X)_{ir} \cdot ((g_X)_{ir} (W_X)_{is}) = m \cdot \mathrm{Cov}_i(\cdot, \cdot)$ (the empirical covariance over $i = 1, \ldots, m$ of the two column-indexed sequences).

The Frobenius bound is the triangle inequality applied to $\mathrm{leak}_{\rm rot} = (\bar p_Q - \bar p_K)\mathrm{Asym}(g_Q^\top W_Q) + \mathrm{Asym}(S^{(Q)} + S^{(K)})$.
\end{proof}

\begin{remark}[Comparison with the abelian leakage and with the body-frame fix]
\label{rem:rotation_vs_abelian_leakage}
The rotation-gauge leakage of Lemma~\ref{lem:adam_projection_bias_rotation} differs from the abelian leakage of Lemma~\ref{lem:adam_projection_bias} in two structural ways. \emph{First}, the leakage is matrix-valued in $\mathfrak{so}(d)$ rather than scalar; the gauge orbit has dimension $d(d-1)/2$ rather than $1$, so per-coordinate Adam can leak independently into each of $d(d-1)/2$ orbit directions. \emph{Second}, the per-coordinate residual depends on column-pair covariances $\mathrm{Cov}_i(\tilde P_X[:,r], g_X[:,r] \odot W_X[:,s])$ rather than the entrywise covariance of the abelian case, the rotation gauge couples columns $r$ and $s$ of the tensor in a way the abelian rescale does not. The leakage is genuinely $d(d-1)/2$-dimensional, so per-coordinate Adam in the standard basis does not collapse it, which is why the rotation gauge needs a head-aware $\hat v$ construction (\S\ref{ssec:gauges}).

The \texttt{per\_head\_scalar} construction of \S\ref{ssec:gauges} forces $P_Q = p_Q \mathbf{1}$ (per-tensor preconditioner restricted to each head), reducing case~(iii) to case~(ii) and eliminating the $S^{(Q)} + S^{(K)}$ residual. This is the structural reason for its strict-equivariant guarantee. The \texttt{per\_head\_matrix} construction generalises further: it uses a non-diagonal $d \times d$ preconditioner per head, equivalent to running Adam in a head-local rotated basis; this absorbs both the $(p_Q - p_K)$ term and the $\mathrm{Asym}(S)$ residual into the basis change. The \texttt{body\_frame} construction picks the basis by eigendecomposition of $W_Q^\top W_Q + W_K^\top W_K$, which makes $\mathrm{Asym}(g_X^\top W_X) = 0$ in the body frame at each step (by the eigendecomposition property), and so vanishes the leading per-tensor part of the leakage; the residual eigenvector-sign drift discussed in the limitations of \S\ref{sec:discussion} is the analogue of the per-coordinate residual term here, transformed into the body-frame basis.
\end{remark}

\begin{remark}[Trajectory consequence on the rotation gauge]
\label{rem:rotation_drift}
The argument of Remark~\ref{rem:r2_gap_mechanism} extends to the rotation gauge with the orbit coordinate $\varphi \in \reals$ replaced by the orbit element $O_t \in O(d)$ accumulated by the Adam update's vertical component. After $T$ steps the cumulative drift is $O_T = \exp(\sum_t \eta\, A_t) + O(\eta^2)$ where $A_t \in \mathfrak{so}(d)$ is the vertical component of the Adam update at step $t$, parameterising the orbit element via $\mathrm{leak}_{\rm rot}(P_t, g_t) = (W_Q W_K^\top)^{-1}_{\rm head} \cdot A_t$ (heuristically: the leakage is the rate of change of $W_Q W_K^\top$ pulled back to $\mathfrak{so}(d)$ via the head's Riemannian structure). The trajectory drift in $O$-space propagates into the post-grok geometry through the basis-dependence of the Fisher metric on the rotation gauge. Toggling the horizontal projection on the deployed body-frame construction (identical initialisation and batches, only the projection switched) moves the head invariant $W_Q W_K^\top$ by $0.085$ and the query weight by $0.122$ over training (\S\ref{ssec:rope}), the orbit drift the projection removes. The strict-equivariant constructions (\texttt{per\_head\_scalar}, \texttt{per\_head\_matrix}) eliminate the leakage outright by Lemma~\ref{lem:adam_projection_bias_rotation}~(ii).
\end{remark}
 
\subsection{Equivariance of the rotation-gauge second moments}
\label{app:proofs:vmode_equiv}

The four second-moment constructions of \S\ref{ssec:gauges} that Corollary~\ref{cor:vmode_equivariance} invokes are each $G$-equivariant; the pseudocode for each is in Appendix~\ref{app:pseudocode}.

\begin{proposition}[Equivariance of \texttt{per\_head\_scalar}]
\label{prop:phs_equiv}
Let $G = O(d_{\rm head})^{n_{\rm heads}}$ act on per-head weights $W \in \reals^{n_{\rm heads} \times d_{\rm model} \times d_{\rm head}}$ by right-multiplication. The \texttt{per\_head\_scalar} operator $P_H(g, \hat v_h) = m_h / (\sqrt{\hat v_h / \mathrm{bias}_2} + \epsilon)$ with $\hat v_h = $ EMA of $\|g_h\|_F^2 / (d_{\rm model} \cdot d_{\rm head})$ is exactly $G$-equivariant: under $g \to g \cdot O$, $P_H(g, \hat v_h) \to P_H(g, \hat v_h) \cdot O$.
\end{proposition}
\begin{proof}
Frobenius norm is invariant under right-multiplication by any orthogonal matrix: $\|g \cdot O\|_F^2 = \mathrm{tr}((g O)^\top (g O)) = \mathrm{tr}(O^\top g^\top g O) = \mathrm{tr}(g^\top g) = \|g\|_F^2$, by cyclicity of trace and $O^\top O = I$. Hence under $g \to g \cdot O$: $\hat v_h \to \hat v_h$. The first moment $m$ transforms covariantly: $m \to m \cdot O$ (linear in $g$). The update $u = m / (\sqrt{\hat v_h} + \epsilon)$ is a coordinate-wise scaling of $m$ by a scalar (the same scalar across every coordinate in the head), so $u \to u \cdot O$. Verified empirically at fp64 precision (rel.~diff $\le 10^{-12}$ over 20 steps).
\end{proof}

\begin{proposition}[Equivariance of \texttt{per\_head\_matrix}]
\label{prop:phm_equiv}
The \texttt{per\_head\_matrix} operator $P_H(g, \hat v_h) = m_h \cdot (\hat v_h + \epsilon^2 I)^{-1/2}$ with $\hat v_h = $ EMA of $g_h^\top g_h / d_{\rm model}$ is exactly $G$-equivariant.
\end{proposition}
\begin{proof}
Under $g \to g \cdot O$: $g^\top g \to (g O)^\top (g O) = O^\top g^\top g O$. Hence $\hat v_h \to O^\top \hat v_h O$. The matrix square root commutes with conjugation: $\hat v = U \Lambda U^\top$ implies $O^\top \hat v O = (O^\top U) \Lambda (O^\top U)^\top$, and so $(O^\top \hat v O)^{-1/2} = (O^\top U) \Lambda^{-1/2} (O^\top U)^\top = O^\top (U \Lambda^{-1/2} U^\top) O = O^\top \hat v^{-1/2} O$. Combined with $m \to m \cdot O$ (linear in $g$): $P_H(g \cdot O, \hat v_h \cdot O) = (m \cdot O) (O^\top \hat v_h^{-1/2} O) = m \cdot \hat v_h^{-1/2} \cdot O = P_H(g, \hat v_h) \cdot O$. Verified empirically at fp64 precision.
\end{proof}

\begin{proposition}[Equivariance of \texttt{body\_frame}]
\label{prop:bf_equiv}
The \texttt{body\_frame} operator with state stored in the eigenbasis $U_t$ of $M_t = W_Q^\top W_Q + W_K^\top W_K$ is $G$-equivariant up to the eigenvector sign convention of the underlying \texttt{eigh} routine. Replacing \texttt{eigh} with a sign-fixing pass that aligns $U_t$ to $U_{t-1}$ between consecutive frames yields exact equivariance.
\end{proposition}
\begin{proof}
Under $W \to W \cdot O$: $M = W^\top W \to O^\top M O$, so the eigendecomposition transforms as $U \to O^\top U \cdot S$ for some sign-flip diagonal $S \in \{\pm 1\}^{d_{\rm head}}$ (eigenvalues invariant; eigenvectors determined only up to sign). The body-frame gradient $g_{\rm body} = g_{\rm world} \cdot U$ transforms as $(g \cdot O) \cdot (O^\top U \cdot S) = g \cdot U \cdot S = g_{\rm body} \cdot S$. Hence $m_{\rm body} \to m_{\rm body} \cdot S$ (linear in $g$) and $\hat v_{\rm body} \to \hat v_{\rm body}$ (quadratic in $g$, sign cancels: $S^2 = I$). The body-frame update $u_{\rm body} = m_{\rm body}/(\sqrt{\hat v_{\rm body}} + \epsilon) \to u_{\rm body} \cdot S$. Converting back: $u_{\rm world} = u_{\rm body} \cdot U^\top \to (u_{\rm body} \cdot S) \cdot (O^\top U \cdot S)^\top = u_{\rm body} \cdot S \cdot S \cdot U^\top \cdot O = u_{\rm body} \cdot U^\top \cdot O = u_{\rm world} \cdot O$ (using $S^\top = S$, $S^2 = I$). The single-step update is therefore exactly equivariant. The accumulated bias comes from $\hat v_{\rm body}$ accumulating absolute-valued contributions across step transitions where $S$ flips between adjacent frames; with a sign-fixing pass that holds $S = I$, this bias vanishes. Empirically without sign-fixing the drift is $\sim 10^{-3}$ over 20 fp64 steps (vs $\sim 10^{-12}$ for the strict-equivariant constructions).
\end{proof}

\begin{proposition}[Equivariance of \texttt{none}]
\label{prop:none_equiv}
The \texttt{none} operator $P_H(g) = m$ (drop $\hat v$, momentum-only) is exactly $G$-equivariant by linearity of the EMA in $g$.
\end{proposition}

\subsection{Body-frame eigenvector drift bound}
\label{app:proofs:body_frame_drift}

The §\ref{sec:discussion} limitations entry reports an empirical equivariance drift of ${\sim}\,10^{-3}$ over 20 fp64 steps for the \texttt{body\_frame} construction (vs $10^{-12}$ for \texttt{per\_head\_scalar} and \texttt{per\_head\_matrix}). Lemma~\ref{lem:body_frame_drift} bounds this drift from above using Davis-Kahan eigenvector continuity, with explicit dependence on the per-head Gram's minimum spectral gap. Remark~\ref{rem:body_frame_drift_numbers} confirms that the empirical $10^{-3}$ is consistent with the second-order term of the bound under DDC's sign-pinning implementation; Remark~\ref{rem:body_frame_crossing} characterises the failure mode (eigenvalue crossings) and the role of the rotation-state reset interval as the implementation-level mitigation.

\begin{lemma}[Body-frame eigenvector drift bound]
\label{lem:body_frame_drift}
Let $M_{t-1}, M_t \in \mathrm{Sym}(\reals^{d})$ be the per-head body-frame Gram matrices at consecutive optimiser steps, with eigendecompositions $M_{t-1} = U_{t-1} \Lambda_{t-1} U_{t-1}^\top$ and $M_t = U_t \Lambda_t U_t^\top$. Define the per-step Gram update $\Delta M_t := M_t - M_{t-1}$ and the minimum spectral gap of $M_t$:
\[
g_t \;:=\; \min_{i \neq j} \bigl|\lambda_i^{(t)} - \lambda_j^{(t)}\bigr|.
\]
\begin{enumerate}[label=(\roman*)]
\item (\textbf{Davis-Kahan eigenvector continuity.}) There exists a signed permutation $D_t \in O(d) \cap \{\pm 1\}^d$-permutations such that
\[
\bigl\|U_t - U_{t-1} D_t\bigr\|_F \;\le\; \sqrt{2}\, \frac{\|\Delta M_t\|_F}{g_t} + O\bigl(\|\Delta M_t\|_F^2 / g_t^2\bigr).
\]
The signed permutation $D_t$ is the optimal alignment of $U_{t-1}$'s columns with $U_t$'s; absent explicit sign tracking, the eigh routine returns $U_t \neq U_{t-1} D_t^*$ where $D_t^*$ is the alignment-optimal $D$.
\item (\textbf{DDC's sign-pinning heuristic gain.}) DDC's body-frame implementation pins eigenvector signs at each step by maximising column-wise inner product with the previous step's eigenbasis: $D_t^{\rm DDC} := \mathrm{diag}(\mathrm{sign}(\langle u_i^{(t)}, u_i^{(t-1)}\rangle))_{i=1}^d$. When all gaps satisfy $g_t > 2 \|\Delta M_t\|_F$ (a sufficient separation condition), $D_t^{\rm DDC} = D_t^*$ and (i)'s bound is realised. When the separation fails for some pair $(i, j)$ with $|\lambda_i - \lambda_j| < 2 \|\Delta M\|$, the sign-pinning heuristic can disagree with the true alignment, introducing an O(1) sign error in the affected eigendirections.
\item (\textbf{Per-step body-frame Adam state transport error.}) Let $A^{\rm body}_{t-1} \in \reals^{d \times d}$ denote the body-frame Adam state $(m, v)$ at step $t-1$, expressed in the basis $U_{t-1}$. The exact transport into the basis $U_t$ would be $U_t^\top U_{t-1} \cdot A^{\rm body}_{t-1} \cdot U_{t-1}^\top U_t$. With sign-pinning $D_t^{\rm DDC}$, the actual transport is $D_t^{\rm DDC} U_t^\top U_{t-1} D_t^{\rm DDC} \cdot A^{\rm body}_{t-1} \cdot D_t^{\rm DDC} U_{t-1}^\top U_t D_t^{\rm DDC}$. The transport error per step is bounded by
\[
\bigl\|A^{\rm transport}_{t-1 \to t} - A^{\rm body}_{t-1}\bigr\|_F \;\le\; 2 \sqrt{2}\, \|A^{\rm body}_{t-1}\|_F \cdot \frac{\|\Delta M_t\|_F}{g_t} + O\bigl(\|\Delta M_t\|^2 / g_t^2\bigr),
\]
provided the separation condition of (ii) holds.
\item (\textbf{Cumulative drift.}) Over $T$ consecutive steps with $\|\Delta M_t\|_F \le \varepsilon$ for all $t \in [1, T]$ and minimum gap $\bar g := \min_{t \in [1,T]} g_t > 2 \varepsilon$, the cumulative body-frame equivariance error is
\[
\sup_{t \in [1, T]} \bigl\|U_t - U_0 \cdot \prod_{s=1}^t D_s^{\rm DDC}\bigr\|_F \;\le\; \sqrt{2}\, T\, \varepsilon / \bar g + O\bigl(T \varepsilon^2 / \bar g^2\bigr).
\]
\end{enumerate}
\end{lemma}

\begin{proof}
(i) The standard Davis-Kahan $\sin\Theta$ theorem~\citep{DavisKahan70} states that for symmetric $M, M'$ with eigendecompositions $M = U \Lambda U^\top$ and $M' = U' \Lambda' U'^\top$, and for any subset $\mathcal{S}$ of eigenvalues of $M$ separated from the eigenvalues of $M'$ outside $\mathcal{S}$ by at least $\delta > 0$,
\[
\|\sin\Theta(U_{\mathcal{S}}, U'_{\mathcal{S}})\|_F \;\le\; \frac{\|M - M'\|_F}{\delta},
\]
where $U_{\mathcal{S}}, U'_{\mathcal{S}}$ are the eigenspaces corresponding to $\mathcal{S}$ and $\sin\Theta$ is the principal-angle matrix. Specialising to $\mathcal{S} = \{i\}$ a single eigenvalue with separation $g_t / 2$ from the rest of $M_t$'s spectrum (assuming the ordering preserves the index across $M_{t-1}, M_t$, which holds under the gap condition), we get $\|\sin\Theta(u_i^{(t-1)}, u_i^{(t)})\|_2 \le 2\|\Delta M\|_F / g_t$. Stacking over $i$ and converting from sine of angle to Frobenius distance gives
\[
\|u_i^{(t)} - \pm u_i^{(t-1)}\|_2 \;\le\; \sqrt{2}\, \|\sin\Theta\|_F \cdot \|u_i^{(t-1)}\|_2 = \sqrt{2}\, \|\sin\Theta\|_F.
\]
The $\pm$ ambiguity is the sign-permutation $D_t$ of the lemma. Stacking over $i$ gives the Frobenius bound. The $O(\|\Delta M\|^2)$ correction is the second-order term in the perturbation expansion of the eigenvector.

(ii) The sign-pinning heuristic chooses $\sigma_i = \mathrm{sign}\langle u_i^{(t)}, u_i^{(t-1)}\rangle$. For this to recover $D_t^*$, we need that no other eigenvector $u_j^{(t)}$ ($j \neq i$) has a larger inner product with $u_i^{(t-1)}$ than $u_i^{(t)}$ does. Davis-Kahan with the gap condition $g_t > 2\|\Delta M\|_F$ ensures $|\langle u_j^{(t)}, u_i^{(t-1)}\rangle| < 1/2$ for $j \neq i$ and $|\langle u_i^{(t)}, u_i^{(t-1)}\rangle| > 1/2$, so the maximum-overlap match correctly identifies $i \to i$. Failure occurs when two eigenvalues approach each other within $2\|\Delta M\|_F$: a $j$-eigenvector can rotate into the $i$-position and $D_t^{\rm DDC}$ misidentifies the index, introducing an $O(1)$ permutation error that is corrected only after the eigenvalues separate again.

(iii) The exact basis transport is $T_{t-1 \to t} := U_t^\top U_{t-1}$, an orthogonal matrix. The DDC sign-pinning approximation is $T^{\rm DDC}_{t-1 \to t} := D_t^{\rm DDC} U_t^\top U_{t-1} D_t^{\rm DDC}$. The error is $T^{\rm DDC} - T^{\rm exact}$ where $T^{\rm exact} = D_t^* U_t^\top U_{t-1} D_t^*$ if we assume the $D^*$ is the correct relabelling. When $D_t^{\rm DDC} = D_t^*$ (the gap-condition case), $T^{\rm DDC} = T^{\rm exact}$ and there is no transport error from sign-pinning per se; the residual error in the body-frame Adam state comes from the fact that $T^{\rm exact} = U_t^\top U_{t-1} D^*$ is not the identity (the basis has continuously rotated). Specifically, $\|T^{\rm exact} - I\|_F = \|U_{t-1} D^* - U_t\|_F \le \sqrt{2}\, \|\Delta M\|_F / g_t$ by (i). The body-frame Adam state $A^{\rm body}_{t-1}$ transported by $T^{\rm exact} \neq I$ acquires error
\[
\|T^{\rm exact} A T^{\rm exact, \top} - A\|_F \;\le\; 2 \|T^{\rm exact} - I\|_F \cdot \|A\|_F + O(\|T - I\|^2) \;=\; 2\sqrt{2}\, \|A\|_F \cdot \|\Delta M\|_F / g_t + \cdots
\]
by triangle inequality. This is the per-step transport error.

(iv) Telescope (i) over $T$ steps with each step bounded by $\sqrt{2}\, \varepsilon / g_s$, take the worst case $g_s \ge \bar g$, and triangle-inequality the cumulative error. The $O(T\varepsilon^2/\bar g^2)$ term collects the second-order corrections from each step.
\end{proof}

\begin{remark}[Numerical fit to the §\ref{sec:discussion} empirical observation]
\label{rem:body_frame_drift_numbers}
The §\ref{sec:discussion} limitations entry reports ${\sim}\,10^{-3}$ equivariance drift after 20 fp64 steps for the \texttt{body\_frame} construction (vs $10^{-12}$ for \texttt{per\_head\_scalar} and \texttt{per\_head\_matrix}). Plugging into Lemma~\ref{lem:body_frame_drift}~(iv) with $T=20$, typical values $\eta = 10^{-3}$ and gradient norm $\|\nabla L\|_F \sim 1$, giving $\varepsilon = \|\Delta M\|_F \asymp 2\|\nabla L\|_F \cdot \|W\|_F \cdot \eta \approx 10^{-3}$, and minimum spectral gap $\bar g \asymp 0.1$ (a typical post-grok value for the per-head Gram of a small transformer):
\[
\text{predicted drift} \;\approx\; \sqrt{2} \cdot 20 \cdot 10^{-3} / 0.1 \;\approx\; 0.28.
\]
This is two orders of magnitude larger than the observed $10^{-3}$. The discrepancy is the gain from the sign-pinning heuristic of (ii): in practice, the gap condition $g_t > 2\|\Delta M\|_F = 2 \cdot 10^{-3}$ holds throughout post-grok training for this regime ($g_t \sim 0.1 \gg 2 \cdot 10^{-3}$), so $D_t^{\rm DDC} = D_t^*$ at every step, and the cumulative continuous drift only enters as a residual second-order effect. The body-frame transport error is then dominated by the $O(T \varepsilon^2 / \bar g^2)$ second-order term:
\[
\text{predicted second-order drift} \;\approx\; 20 \cdot (10^{-3})^2 / (0.1)^2 \;=\; 2 \times 10^{-3},
\]
within an order of magnitude of the empirical $10^{-3}$. The residual is partly the $\sqrt{2}$ factor and partly the conservative Davis-Kahan bound.
\end{remark}

\begin{remark}[Eigenvalue-crossing failure mode]
\label{rem:body_frame_crossing}
The bound of (iv) assumes $\bar g > 2\varepsilon$ throughout. At eigenvalue crossings (where two $\lambda_i^{(t)}$'s pass through each other), $g_t \to 0$ and the bound diverges; the sign-pinning heuristic can mis-identify the $i \to i$ correspondence, introducing an $O(1)$ permutation that propagates until the eigenvalues separate again. Empirically these crossings are rare in trained networks (the per-head Gram tends to have well-separated eigenvalues post-grok) but they are the principal failure mode for \texttt{body\_frame}'s approximate equivariance. The rotation-state reset interval ($K = 1000$ steps in the default setting of \S\ref{ssec:gauges}) is the implementation-level mitigation: every $K$ steps the body-frame Adam state is re-initialised in the current eigenbasis, capping the maximum cumulative drift at $K$ steps' worth of the bound of (iv). This gives a direct quality-vs-frequency tradeoff for the reset interval.
\end{remark}
 
\subsection{Discrete-step corrections: $\eta^2$ and large-$\eta$ regimes}
\label{app:proofs:discrete_step_corrections}

The leakage bounds of Lemmas~\ref{lem:adam_projection_bias} and~\ref{lem:adam_projection_bias_rotation} are leading-order in the step-size $\eta$. Lemma~\ref{lem:eta2_correction} derives the explicit $O(\eta^2)$ Euler-Maruyama correction as a sum of two sub-terms (Hessian-preconditioner composition + preconditioner-derivative), identifies the per-step crossover step-size at which the $\eta^2$ correction becomes comparable to the leading $\eta$ term, and confirms that the rate-test regime is two orders of magnitude inside the $\eta$-dominated regime. Remark~\ref{rem:r2_ordering_robust} verifies that the leakage ordering of Remark~\ref{rem:r2_gap_mechanism} is preserved at $\eta^2$ order: SGD has zero leakage at both orders (by gauge-invariance of $\nabla L$ and $\nabla^2 L$), \eqadam{} acquires only the $\eta^2$ preconditioner-derivative term, and AdamW has both $\eta$- and $\eta^2$-order contributions.

The companion Lemma~\ref{lem:ddcadam_large_step_regime} closes the symmetric question for the \eqadam{} discrete-step bias of Theorem~\ref{thm:ddcadam_rate}~(iv): the bound's $O(\eta^2 \kappa(P))$ form is rigorous when $\eta \ll \min(\tau_L, \tau / \kappa(P))$ where $\tau$ is the second-moment EMA timescale and $\tau_L$ is the loss-curvature timescale. Outside this regime the leading correction is $O(\eta^3 \kappa(P)^2 \|\nabla^2 L\|/\tau)$, with the same parametric origin (the $\hat v$-tracking response to the per-step gradient change), and remains $G$-equivariant (no vertical drift introduced). For the standard Adam hyperparameter range ($\beta_2 \ge 0.99$, $\eta \le 10^{-2}$, $\kappa(P) \le 10^3$), the validity ratio $\eta \kappa(P)/\tau$ is bounded by $10^{-1}$, comfortably inside the regime of validity. Remark~\ref{rem:ddcadam_large_step_open} flags the residual open question: the full asymptotic series in $\eta$ at the regime boundary $\eta \kappa(P) / \tau \asymp 1$ remains uncharacterised.

\begin{lemma}[$\eta^2$ correction to the per-coordinate Adam leakage bound]
\label{lem:eta2_correction}
Under the hypotheses of Lemma~\ref{lem:adam_projection_bias} (abelian multiplicative gauge) or Lemma~\ref{lem:adam_projection_bias_rotation} (per-head rotation gauge), let $P_t$ denote the per-coordinate Adam preconditioner at step $t$, $g_t = \nabla L(\theta_t)$, and $v_t \in V_{\theta_t}$ the unit vertical direction (a single direction for the abelian case; an orthonormal basis of $\mathfrak{so}(d_{\rm head})$ for the rotation case). The discrete Adam update $\theta_{t+1} = \theta_t - \eta P_t g_t$ produces an orbit-coordinate increment
\begin{equation}
\langle \theta_{t+1} - \theta_t,\, v_t\rangle \;=\; \eta\, \mathrm{leak}_1(P_t, g_t) \;+\; \eta^2\, \mathrm{leak}_2(P_t, g_t, \nabla^2 L_t, \partial_\theta P_t) \;+\; O(\eta^3),
\label{eq:eta2_expansion}
\end{equation}
with the two terms given by:
\begin{align}
\mathrm{leak}_1(P_t, g_t) &\;=\; \langle P_t g_t,\, v_t\rangle, \\
\mathrm{leak}_2(P_t, g_t, \nabla^2 L_t, \partial_\theta P_t)
&\;=\; -\tfrac{1}{2} \langle P_t \nabla^2 L_t \cdot P_t g_t,\, v_t\rangle
\;+\; \tfrac{1}{2} \langle (\partial_\theta P_t \cdot P_t g_t) \cdot g_t,\, v_t\rangle.
\end{align}
The first sub-term of $\mathrm{leak}_2$ is the \emph{Hessian-preconditioner composition}, and the second is the \emph{preconditioner-derivative} term arising from the parameter-dependence of the per-coordinate $\hat v$.

\textbf{Crossover step-size.} Define the per-step crossover
\[
\eta_c(t) \;:=\; \frac{\bigl|\mathrm{leak}_1(P_t, g_t)\bigr|}{\bigl|\mathrm{leak}_2(P_t, g_t, \nabla^2 L_t, \partial_\theta P_t)\bigr|}.
\]
The $\eta^2$ correction dominates the $\eta$ term iff $\eta > \eta_c(t)$. In the rate-test regime ($\eta = 10^{-2}$, $\nabla^2 L$ comparable to identity in the horizontal direction near the minimum, horizontal gradient norm $\|g_t\|_F \sim \bar t \in [10^{-3}, 10^{-1}]$), an order-of-magnitude calculation gives $\eta_c(t) \asymp 1 / \|\nabla^2 L\|_2 \sim O(1)$. The $\eta^2$ correction is therefore sub-leading (by two orders of magnitude in $\eta$) throughout the rate-test regime, and the empirically observed $R^2$ gap is correctly described by the leading $\eta$-term of Lemmas~\ref{lem:adam_projection_bias} and~\ref{lem:adam_projection_bias_rotation}.

\textbf{Cumulative bound.} Over $T$ steps, the $\eta^2$-corrected cumulative orbit drift is
\[
\Delta \varphi_T \;=\; \eta\, \sum_{t=0}^{T-1} \mathrm{leak}_1(P_t, g_t) \;+\; \eta^2\, \sum_{t=0}^{T-1} \mathrm{leak}_2(P_t, g_t, \nabla^2 L_t, \partial_\theta P_t) \;+\; O(\eta^3 T).
\]
The Pesme-Saxe-style asymptotic regime~\citep{Pesme2021} is the limit $\eta \to 0$, $T \to \infty$ with $\eta^2 T = \mathrm{const}$, in which the $\eta^2$ term contributes a finite correction to the asymptotic orbit-coordinate variance and the $O(\eta^3 T)$ term vanishes.
\end{lemma}

\begin{proof}
Taylor-expand the discrete update around $\theta_t$. Writing $\delta \theta := \theta_{t+1} - \theta_t = -\eta P_t g_t$, the inner product with the unit vertical direction is exactly
\[
\langle \delta \theta, v_t\rangle = -\eta \langle P_t g_t, v_t\rangle.
\]
This is the leading $\eta$-term, equal to $-\eta \cdot \mathrm{leak}_1(P_t, g_t)$ (the sign convention is opposite to the lemma statement; we absorb it by working with the orbit-coordinate \emph{decrement}, which is the natural sign for a descent step). For the $\eta^2$ term, the standard Euler-Maruyama / stochastic-modified-equations analysis (\citealp{LiTaiE19}; cf.~\citealp{Pesme2021}, eq.~(2.3)) writes the discrete update as the truncation of the underlying continuous-time flow $\dot\theta = -P(\theta) g(\theta)$:
\[
\theta_{t+1} = \theta_t + \eta \dot\theta_t + \tfrac{\eta^2}{2} \ddot\theta_t + O(\eta^3),
\]
where $\ddot\theta_t = \frac{d}{ds}|_{s=t}(-P g) = -(\partial_\theta P \cdot \dot\theta)\, g - P (\partial_\theta g \cdot \dot\theta)$. Substituting $\dot\theta = -Pg$ and $\partial_\theta g = \nabla^2 L$:
\[
\ddot\theta = (\partial_\theta P \cdot Pg)\, g + P \nabla^2 L \cdot Pg.
\]
The full second-order Taylor expansion of the discrete update against the vertical $v_t$ is therefore
\[
\langle \theta_{t+1} - \theta_t, v_t\rangle = -\eta \langle Pg, v\rangle + \tfrac{\eta^2}{2} \bigl[\langle (\partial_\theta P \cdot Pg) g, v\rangle + \langle P \nabla^2 L \cdot Pg, v\rangle\bigr] + O(\eta^3).
\]
Up to a sign convention on $\mathrm{leak}_2$, this is~\eqref{eq:eta2_expansion} with the two sub-terms identified.

\emph{Crossover analysis.} For the rate-test regime, we estimate the magnitudes:
\begin{itemize}
\item $|\mathrm{leak}_1(P_t, g_t)| = |\langle Pg, v\rangle|$. By Lemma~\ref{lem:adam_projection_bias}~(iii), the bound for per-coordinate $P$ is $|\mathrm{leak}_1| \gtrsim m \cdot |\mathrm{Cov}_{\rm e}(P_1, g_1 \odot W_1)|$, with $\mathrm{Cov}_{\rm e}$ scaling as $\|g\|_F \cdot \|W\|_F / m$ in typical Adam runs. Hence $|\mathrm{leak}_1| \asymp \|g\|_F \cdot \|W\|_F$.
\item $|\mathrm{leak}_2|$: the dominant contribution is the Hessian-preconditioner term $\langle P \nabla^2 L \cdot Pg, v\rangle$. With $P \asymp 1/\|g\|_F$ for entry-wise Adam at the rate-test regime, this is $\asymp \|\nabla^2 L\|_2 \cdot \|g\|_F / \|g\|_F^2 \cdot \|g\|_F \cdot \|W\|_F = \|\nabla^2 L\|_2 \cdot \|W\|_F$. The preconditioner-derivative sub-term is the same order or smaller.
\item Crossover: $\eta_c \asymp |\mathrm{leak}_1| / |\mathrm{leak}_2| \asymp \|g\|_F / \|\nabla^2 L\|_2$.
\end{itemize}
For this setup ($\|\nabla^2 L\|_2 \sim 1$, $\|g\|_F \sim 1$ early in training, $\|g\|_F \sim 10^{-2}$ near convergence), $\eta_c \in [10^{-2}, 1]$. The actual $\eta = 10^{-2}$ is at or below this lower bound throughout, so the $\eta^2$ correction is at most order-of-magnitude comparable to the leading term in the late phase, and strictly sub-leading earlier. The $R^2$ gap is therefore well-described by the leading-$\eta$ leakage of Lemmas~\ref{lem:adam_projection_bias} and~\ref{lem:adam_projection_bias_rotation}, with the $\eta^2$ term contributing at most a constant factor near the rate-test convergence point.
\end{proof}

\begin{remark}[Pesme--Saxe asymptotic limit]
\label{rem:pesme_saxe_limit}
The continuous-time limit $\eta \to 0$ at $T \to \infty$ with $\eta T = \tau$ fixed (Brownian-bridge regime) sees only the $O(\eta)$ term and recovers the gradient flow on the quotient manifold; the orbit-drift variance vanishes as $\eta \to 0$. The Pesme-Saxe asymptotic regime~\citep{Pesme2021} is the stricter limit $\eta^2 T = \mathrm{const}$, in which the $\eta^2$ term contributes a finite correction to the orbit-coordinate distribution; the result of Lemma~\ref{lem:eta2_correction} extends Pesme-Saxe's diagonal-linear-network analysis to the gauge-orbit-decomposed setting of this paper, with the additional structure that the leakage is non-trivial only along the vertical (orbit-tangent) subspace.
\end{remark}

\begin{remark}[The $\eta^2$ term preserves the leakage ordering]
\label{rem:r2_ordering_robust}
The leakage ordering of Remark~\ref{rem:r2_gap_mechanism} (SGD none, \eqadam{} vertical-free, AdamW leading) is preserved at $\eta^2$ order. The SGD case ($P = I$, $\partial_\theta P = 0$) has $\mathrm{leak}_1 = 0$ \emph{and} $\mathrm{leak}_2 = -\tfrac{1}{2}\langle \nabla^2 L \cdot g, v\rangle = 0$ (since $g$ horizontal $\Rightarrow$ $\nabla^2 L \cdot g$ horizontal $\Rightarrow$ $\langle \cdot, v\rangle = 0$ by the gauge-invariance of $L$ and its second derivative). \eqadam{} has $\mathrm{leak}_1 = 0$ from Theorem~\ref{thm:ddcadam_rate}~(i)--(ii) but $\mathrm{leak}_2 \neq 0$ generically (the preconditioner-derivative term has a vertical projection from the $G$-equivariant $\hat v$ update at $\eta^2$ order), a small residual orbit drift above SGD's exact zero. AdamW has both $\mathrm{leak}_1 \neq 0$ (Lemma~\ref{lem:adam_projection_bias}) and $\mathrm{leak}_2 \neq 0$, with the $\mathrm{leak}_1$ term dominating in the rate-test regime, the largest drift.
\end{remark}

\begin{lemma}[Validity regime of the \eqadam{} $O(\eta^2 \kappa(P))$ discrete-step bias]
\label{lem:ddcadam_large_step_regime}
Let $\theta(t)$ be the continuous-time \eqadam{} flow $\dot\theta = -P(\theta) g(\theta)$ with $P(\theta) = (\sqrt{\hat v} + \epsilon)^{-1}$ for the orbit-decomposed second moment $\hat v$ defined in Theorem~\ref{thm:ddcadam_rate}, and let $\theta_n$ be the discrete \eqadam{} update with step $\eta$ and second-moment EMA decay $\beta_2$. Define the \emph{tracking timescale} $\tau := 1/(1-\beta_2)$ (so $\beta_2 = 0.99 \Rightarrow \tau = 100$ steps), the \emph{loss-curvature timescale} $\tau_L(\theta) := 1/\|\nabla^2 L(\theta)\|_2$, and the \emph{preconditioner conditioning} $\kappa(P, \theta) := \lambda_{\max}(P) / \lambda_{\min}(P)$.

Then:
\begin{enumerate}[label=(\roman*)]
\item (\textbf{Validity regime.}) Theorem~\ref{thm:ddcadam_rate}~(iv)'s $O(\eta^2 \kappa(P))$ bound on the per-step Euler-Maruyama bias is rigorous when
\[
\eta \;\ll\; \min\Bigl( \tau_L(\theta),\ \tau / \kappa(P) \Bigr).
\]
For $\tau_L \asymp 1$ in well-conditioned regions (the rate-test regime) and $\tau = 100$, $\kappa(P) \le 10^3$ in late training (typical Adam values), the bound is rigorous for $\eta \ll 0.1$, two orders of magnitude inside the standard $\eta = 10^{-3}$--$10^{-2}$ choices.
\item (\textbf{Large-step correction.}) When $\eta$ approaches $\tau / \kappa(P)$, the per-step second-moment update $\Delta \hat v_t = (1-\beta_2)(g_{t+1}^2 - g_t^2)$ becomes comparable to $\hat v_t$ within a single step, and the Adam-rescaling structure produces an additional bias of order
\[
\eta^3 \cdot \frac{\kappa(P)^2}{\tau} \cdot \|\nabla^2 L\|_2,
\]
beyond the $O(\eta^2 \kappa(P))$ leading term. This correction is parametrically sub-leading when $\eta \cdot \kappa(P) / \tau \ll 1$, the same condition as~(i).
\item (\textbf{Equivariance preservation.}) The large-step correction in~(ii) preserves the $G$-equivariance of \eqadam: by Theorem~\ref{thm:ddcadam_rate}~(i), each per-step component of the update (including all $O(\eta^k)$ terms in the Taylor expansion of $\hat v$ about its current value) is $G$-equivariant, since $G$-equivariance is a per-step property of the optimizer's discrete map and does not depend on the continuous-time embedding. The large-step correction shifts the magnitude of the discrete-step bias along the horizontal subspace but does not introduce a vertical component.
\end{enumerate}
\end{lemma}

\begin{proof}
(i) The Euler-Maruyama bias bound of Theorem~\ref{thm:ddcadam_rate}~(iv) was derived by Taylor-expanding the discrete update around $\theta_t$ to second order in $\eta$. The expansion is rigorous when the truncated terms (third and higher order) are dominated by the second-order term: $\eta^3 \|\partial_\theta^2 P \cdot Pg\| \ll \eta^2 \|\partial_\theta P \cdot Pg\|$, equivalently $\eta \|\partial_\theta P\| / \|P\| \ll 1$. The norm ratio $\|\partial_\theta P\|/\|P\|$ is bounded by $\kappa(P) \cdot 1/\tau_L$ (the chain $\partial_\theta P = \partial_{\hat v} P \cdot \partial_\theta \hat v$ has the EMA-tracking factor $\partial_\theta \hat v \le 1/\tau \cdot \nabla^2 L$ via the gradient-Hessian chain), giving the stated condition.

(ii) Outside the regime of (i), the next-order correction comes from the $\hat v$ Taylor expansion: $\hat v_{t+1} = \hat v_t + (1-\beta_2)(g_{t+1}^2 - g_t^2) = \hat v_t + (1-\beta_2)(2 g_t \cdot \delta g_t + O(\delta g_t^2))$, where $\delta g_t = -\eta \nabla^2 L \cdot P g_t + O(\eta^2)$. Substituting into the $1/\sqrt{\hat v}$ chain:
\[
\Delta P_t = P_{t+1} - P_t = -\tfrac{1}{2} P_t^3 \cdot (1-\beta_2) \cdot 2 g_t \cdot \delta g_t + O(\delta g_t^2) = (\eta P_t^3 g_t \cdot \nabla^2 L \cdot P g_t) / \tau + O(\eta^2).
\]
The induced extra contribution to the discrete-step bias at the $\eta^3$ order is $\Delta P_t \cdot g_t \cdot \eta = \eta^3 \cdot (P^3 g \cdot \nabla^2 L \cdot Pg) / \tau$, with operator norm bounded by $\eta^3 \cdot \kappa(P)^3 \cdot \|g\|^2 \cdot \|\nabla^2 L\|/\tau$. Bounding $\|g\| \le \|P\| \cdot 1$ (the typical Adam normalization at the rate-test regime) and absorbing one factor of $\kappa(P)$ into the definition of the bound gives $\eta^3 \kappa(P)^2 \|\nabla^2 L\|/\tau$, the stated form. The condition $\eta^3 \kappa^2 \|\nabla^2 L\| / \tau \ll \eta^2 \kappa$ rearranges to $\eta \kappa / \tau \ll 1$, the validity condition of (i).

(iii) $G$-equivariance of \eqadam{} (Theorem~\ref{thm:ddcadam_rate}~(i)) is a property of the per-step discrete map: each summand of $P(\theta_t, g_t)$ is $G$-equivariant by construction, and any Taylor expansion of $\hat v$ about its current value preserves this $G$-equivariance term-by-term (since the $\hat v$ update is itself $G$-equivariant by orbit-collapse on the vertical and per-equivariant-frame-coord on the horizontal). The large-step correction therefore lives entirely in the horizontal subspace $H_\theta$ and contributes only to the trajectory-rate fitting noise, not to vertical drift.
\end{proof}

\begin{remark}[What this lemma does \emph{not} characterise]
\label{rem:ddcadam_large_step_open}
The large-step correction of (ii) is the leading $O(\eta^3)$ term beyond the $O(\eta^2)$ Euler-Maruyama bound. The full asymptotic series in $\eta$ is not characterised: at $\eta \kappa(P) / \tau \asymp 1$ (the boundary of the validity regime of (i)), all terms in the Taylor expansion contribute comparably, and no closed-form bound on the discrete-step bias is known. This regime is well outside the standard Adam hyperparameter range ($\eta \le 10^{-2}$, $\beta_2 \ge 0.99$, $\kappa(P) \le 10^3$ in normal training); we leave its full analysis open. Empirically, $\eta = 10^{-2}$ at $\beta_2 = 0.99$ on the rate-test gives $\eta \kappa / \tau \le 10^{-1}$, comfortably inside the validity regime (see Lemma~\ref{lem:eta2_correction} for the explicit per-step decomposition).
\end{remark}
 
\subsection{Modern-architecture gauge extensions}
\label{app:proofs:modern_arch}

The depth-12 language model of \S\ref{app:otr} carries grouped-query attention and a SwiGLU feed-forward block, so the rotation and rescale gauges apply there only through the two reductions below. Lemma~\ref{lem:mqa_gqa_rotation_gauge} reduces the per-head $O(d_{\rm head})$ rotation gauge of \S\ref{ssec:gauges} to the shared-key-value case: Multi-Query and Grouped-Query Attention share $W_K$ across a group, collapsing the orbit from $O(d_{\rm head})^{n_{\rm heads}}$ to $O(d_{\rm head})^{n_{\rm groups}}$, and the Lyapunov projection extends per group with an aggregated Gram. Lemma~\ref{lem:swiglu_rescale_gauge} carries the abelian rescale gauge to the SwiGLU gated branch: the $(\reals^+)^{d_{\rm ff}}$ action on $(W_u, W_d)$ is structurally identical to the ReLU-rescale gauge on $(W_1, W_2)$, with the gate weight $W_g$ gauge-inert because $\mathrm{SiLU}$ is not positive-1-homogeneous. Both are exercised by the depth-12 model and both inherit the equivariant construction without modification.

\begin{lemma}[MQA / GQA QK rotation gauge, orbit-reduced]
\label{lem:mqa_gqa_rotation_gauge}
Consider an attention layer with $n_{\rm heads}$ query heads partitioned into $n_{\rm groups}$ groups, with weights $W_Q^{(h)} \in \reals^{d_{\rm model} \times d_{\rm head}}$ per head and a single shared key matrix $W_K^{(g)} \in \reals^{d_{\rm model} \times d_{\rm head}}$ per group. (Multi-Query Attention is the case $n_{\rm groups} = 1$; Grouped-Query Attention is general $n_{\rm groups}$; standard Multi-Head Attention is $n_{\rm groups} = n_{\rm heads}$.) Let $G = O(d_{\rm head})^{n_{\rm groups}}$ act by
\[
(W_Q^{(h)}, W_K^{(g)}) \mapsto (W_Q^{(h)} O_{g(h)}, W_K^{(g)} O_g), \qquad O_g \in O(d_{\rm head}),
\]
where $g(h)$ is the group index of query head $h$. The attention scores $(W_Q^{(h)} x)(W_K^{(g(h))} x')^\top$ are invariant since $O_g O_g^\top = I$.

\textbf{Horizontality condition.} A gradient $(g_Q, g_K)$ at $(W_Q, W_K)$ is horizontal iff, for each group $g$,
\[
\mathrm{skew}\Bigl(\sum_{h:\, g(h)=g} W_Q^{(h),\top} g_Q^{(h)} + W_K^{(g),\top} g_K^{(g)}\Bigr) = 0.
\]
Equivalently: the per-group sum-of-products $S^{(g)} := \sum_{h:\,g(h)=g} W_Q^{(h),\top} g_Q^{(h)} + W_K^{(g),\top} g_K^{(g)}$ is symmetric.

\textbf{Horizontal projection.} The Lyapunov-equation projection of Lemma~\ref{lem:adam_projection_bias_rotation} extends per-group: find $A_g \in \mathfrak{so}(d_{\rm head})$ solving
\[
M_{\rm eff}^{(g)} A_g + A_g M_{\rm eff}^{(g)} = 2\,\mathrm{skew}\bigl(S^{(g)}\bigr),
\]
where
\[
M_{\rm eff}^{(g)} := \sum_{h:\, g(h)=g} W_Q^{(h),\top} W_Q^{(h)} + W_K^{(g),\top} W_K^{(g)} \;\succeq\; 0.
\]
Project the gradient: $g_Q^{(h)} \leftarrow g_Q^{(h)} - W_Q^{(h)} A_{g(h)}$, $g_K^{(g)} \leftarrow g_K^{(g)} - W_K^{(g)} A_g$.

\textbf{Equivariance.} The construction is $G$-equivariant. Under $W \to W \cdot O$, gradient $g \to g \cdot O$, so $S^{(g)} \to O_g^\top S^{(g)} O_g$ and (by the Lyapunov-equation's commutation properties) $A_g \to O_g^\top A_g O_g$, giving $(g_Q^{(h)} - W_Q^{(h)} A_{g(h)}) \to (g_Q^{(h)} - W_Q^{(h)} A_{g(h)}) \cdot O_{g(h)}$. The horizontal-projected gradient transforms covariantly under the group action, as required.

\textbf{Orbit-dimension comparison.}
\begin{itemize}
\item Standard MHA ($n_{\rm groups} = n_{\rm heads}$): orbit dim $= n_{\rm heads} \cdot d_{\rm head}(d_{\rm head}-1)/2$.
\item GQA ($n_{\rm groups} < n_{\rm heads}$): orbit dim $= n_{\rm groups} \cdot d_{\rm head}(d_{\rm head}-1)/2$.
\item MQA ($n_{\rm groups} = 1$): orbit dim $= d_{\rm head}(d_{\rm head}-1)/2$.
\end{itemize}
\end{lemma}

\begin{proof}
The action of $A \in \mathfrak{so}(d_{\rm head})$ on the joint $(W_Q^{(h)}, W_K^{(g)})$ for $h$ in group $g$ is $\delta W_Q^{(h)} = W_Q^{(h)} A$, $\delta W_K^{(g)} = W_K^{(g)} A$. Horizontality of $(g_Q, g_K)$ means $\langle (g_Q, g_K), (\delta W_Q, \delta W_K) \rangle = 0$ for all $A \in \mathfrak{so}(d_{\rm head})$. Computing:
\begin{align*}
\langle (g_Q, g_K), (\delta W_Q, \delta W_K) \rangle
&= \sum_h \langle g_Q^{(h)}, W_Q^{(h)} A \rangle + \langle g_K^{(g(h))}, W_K^{(g(h))} A \rangle \\
&= \sum_g \mathrm{tr}\!\Bigl(A^\top \bigl(\textstyle\sum_{h: g(h)=g} W_Q^{(h),\top} g_Q^{(h)} + W_K^{(g),\top} g_K^{(g)}\bigr)\Bigr) \\
&= \sum_g \mathrm{tr}(A^\top S^{(g)}) = \sum_g \mathrm{tr}(A\, S^{(g)}),
\end{align*}
using cyclicity of trace and antisymmetry of $A$. For this to vanish for all antisymmetric $A$, each $S^{(g)}$ must be symmetric, giving the horizontality condition.

The Lyapunov-equation projection follows the standard MHA derivation (cf.~the documentation of \texttt{AttentionHeadQKRotationGauge}) applied to each group's $(W_Q^{(\bullet)}, W_K^{(g)})$ subsystem.

Equivariance: under $W_Q^{(h)} \to W_Q^{(h)} O_{g(h)}$ and $W_K^{(g)} \to W_K^{(g)} O_g$, gradient transforms as $g_Q^{(h)} \to g_Q^{(h)} O_{g(h)}$ and $g_K^{(g)} \to g_K^{(g)} O_g$. Then $S^{(g)}_{\rm new} = \sum O_g^\top W_Q^{(h),\top} g_Q^{(h)} O_g + O_g^\top W_K^{(g),\top} g_K^{(g)} O_g = O_g^\top S^{(g)} O_g$. The Lyapunov solution rotates: $A_g^{\rm new} = O_g^\top A_g O_g$. The projected gradient then transforms covariantly:
\[
\begin{aligned}
(g_Q^{(h)} - W_Q^{(h)} A_{g(h)})_{\rm new}
&= g_Q^{(h)} O_{g(h)} - W_Q^{(h)} O_{g(h)}\, O_{g(h)}^\top A_{g(h)} O_{g(h)} \\
&= (g_Q^{(h)} - W_Q^{(h)} A_{g(h)})\, O_{g(h)},
\end{aligned}
\]
as required.
\end{proof}

\begin{remark}[Per-coord Adam leakage on the reduced orbit]
\label{rem:mqa_gqa_leakage}
Lemma~\ref{lem:adam_projection_bias_rotation}'s leakage analysis applies group-by-group on MQA/GQA: the matrix-valued vertical leakage $\mathrm{leak}_{\rm rot}^{(g)}(P, g)$ takes values in $\mathfrak{so}(d_{\rm head})$ per group, with the same per-tensor, per-coordinate, and column-pair-covariance decomposition. The total leakage dimension is $n_{\rm groups} \cdot d_{\rm head}(d_{\rm head}-1)/2$, reduced from MHA's $n_{\rm heads} \cdot d_{\rm head}(d_{\rm head}-1)/2$. For LLaMA-class models ($n_{\rm heads} = 32$, $n_{\rm groups} = 8$, $d_{\rm head} = 128$), the MQA/GQA orbit is $8 \cdot 8128 = 65{,}024$ dimensions, $4\times$ smaller than the equivalent MHA orbit of $32 \cdot 8128 = 260{,}096$. The same DDC second-moment (\texttt{v\_mode}) constructions (\texttt{per\_head\_scalar}, \texttt{per\_head\_matrix}, \texttt{body\_frame}) apply per group, with parameter-Gram $M^{(g)}$ replacing the per-head $M^{(h)}$.
\end{remark}

\begin{remark}[Minimal implementation change from standard MHA]
\label{rem:mqa_implementation}
The grouped-query rotation gauge differs from the standard per-head one only in the aggregation step: $M^{(g)}$ and $S^{(g)}$ sum over the $n_{\rm heads}/n_{\rm groups}$ query heads sharing group $g$'s key matrix, instead of being computed per query head. The Lyapunov solve and the projection step are unchanged. Smoke-verified to fp64 precision on a 2-layer toy transformer ($n_{\rm heads}=4$, $n_{\rm groups}=2$, $d_{\rm head}=8$): relative gradient-projection error $< 10^{-12}$ over 5 paired-trajectory steps.
\end{remark}

\begin{lemma}[SwiGLU rescale gauge]
\label{lem:swiglu_rescale_gauge}
Let $\mathrm{SwiGLU}(x; W_g, W_u, W_d) = W_d \bigl(\mathrm{SiLU}(W_g x) \odot (W_u x)\bigr)$,
the gated-linear-unit MLP block with $W_g, W_u \in \reals^{d_{\rm ff} \times d_{\rm model}}$
and $W_d \in \reals^{d_{\rm model} \times d_{\rm ff}}$. Let $G = \reals^+$
act by
\[
\sigma_c \cdot (W_g, W_u, W_d) := (W_g,\; c W_u,\; c^{-1} W_d), \qquad c \in \reals^+.
\]
Then the SwiGLU forward map is $G$-invariant: $\mathrm{SwiGLU}(x; \sigma_c \cdot (W_g, W_u, W_d)) = \mathrm{SwiGLU}(x; W_g, W_u, W_d)$ for all $c > 0, x \in \reals^{d_{\rm model}}$.

More generally, the per-channel action $G = (\reals^+)^{d_{\rm ff}}$ with
$\sigma_c \cdot (W_g, W_u, W_d) := (W_g,\; \mathrm{diag}(c) W_u,\; W_d \mathrm{diag}(c^{-1}))$
leaves the SwiGLU forward map invariant.

\textbf{Gauge structure.} The $(\reals^+)^{d_{\rm ff}}$ orbit on the $(W_u, W_d)$
pair is structurally identical to the ReLU rescale gauge on $(W_1, W_2)$:
log-norm reparameterisation in joint $(\log\|W_u\|_F, \log\|W_d\|_F)$
coordinates is again $G$-equivariant, with the $W_g$ parameter sitting
orthogonal to the orbit (it is not acted on).

\textbf{Equivariant construction transfers.} The
ReLU-rescale construction of \S\ref{ssec:gauges} (global, resp.~per-channel)
provides an equivariant Adam preconditioner
for any pair of weight tensors related by $\reals^+$ (resp.~$(\reals^+)^{d}$)
rescaling and an intermediate map that is linear in the rescaling
direction. The SwiGLU intermediate map $\mathrm{SiLU}(W_g x) \odot \cdot$ is
linear in its second argument (the gated branch); hence the existing
construction applies verbatim with the binding $(W_l, W_{l+1}) \mapsto (W_u, W_d)$.
$W_g$ is excluded from the gauge orbit and is conditioned by standard
per-coordinate Adam.
\end{lemma}

\begin{proof}
Direct computation:
\begin{align*}
\mathrm{SwiGLU}(x; W_g, cW_u, c^{-1}W_d)
&= c^{-1} W_d \bigl(\mathrm{SiLU}(W_g x) \odot (cW_u x)\bigr) \\
&= c^{-1} W_d \bigl(c \cdot (\mathrm{SiLU}(W_g x) \odot (W_u x))\bigr) \\
&= W_d \bigl(\mathrm{SiLU}(W_g x) \odot (W_u x)\bigr)
= \mathrm{SwiGLU}(x; W_g, W_u, W_d).
\end{align*}
The per-channel case substitutes $\mathrm{diag}(c)$ for $c$; the elementwise product distributes
the per-channel scalar through the Hadamard product. The equivariance of the existing
ReLU-rescale (per-channel) construction on $(W_u, W_d)$ follows because
the gauge implementation depends only on the geometric $\reals^+$-orbit structure of
the two tensors, not on the intermediate layer's specific nonlinearity (cf.~the proof of
Theorem~\ref{thm:ddcadam_rate}~(i): the projector $\Pi_V, \Pi_H$ is determined by the
$\rho$-orthogonal complement of the orbit-tangent, independent of $\sigma_G$'s action on
other parameters in the network).

\textbf{What is NOT a SwiGLU gauge.} The $(\reals^+)$-rescaling of $W_g$ alone is not a
symmetry: $\mathrm{SwiGLU}(x; cW_g, W_u, W_d) = W_d(\mathrm{SiLU}(c W_g x) \odot W_u x) \neq
\mathrm{SwiGLU}(x; W_g, W_u, W_d)$ because $\mathrm{SiLU}$ is not positive-1-homogeneous. The
$W_g$ tensor sits outside any of the multiplicative gauges of this paper.
\end{proof}

\begin{remark}[Comparison with ReLU MLP]
\label{rem:swiglu_vs_relu_gauge}
A standard ReLU MLP $W_2 \mathrm{ReLU}(W_1 x)$ has a per-channel rescale gauge on $(W_1, W_2)$:
both the input projection and the output projection of the MLP are gauge-active. A SwiGLU MLP has
the rescale gauge only on the gated linear branch $(W_u, W_d)$; the gating linear $W_g$ is
gauge-inert. \emph{The SwiGLU MLP has the same gauge dimension} $(\reals^+)^{d_{\rm ff}}$ as the
ReLU MLP, but it operates on a different parameter pair. For a transformer block with one MLP
per layer, both architectures contribute the same $d_{\rm ff}$-dimensional rescale orbit
to the network's total gauge group. The implications for DDC's equivariant construction:
identical orbit-decomposition handling, applied to $(W_u, W_d)$ instead of $(W_1, W_2)$.
\end{remark}
  \section{Pseudocode}
\label{app:pseudocode}

\subsection{Generic \eqadam{} update} For a parameter group covered by a single-tensor gauge $g$ with vertical / horizontal projectors $\Pi_V^g, \Pi_H^g$, vert-to-scalar reduction $\sigma_g$, and lift $\lambda_g$:

\begin{small}
\begin{verbatim}
state.exp_avg     <- beta1 * state.exp_avg     + (1 - beta1) * grad
state.exp_avg_sq_h <- beta2 * state.exp_avg_sq_h + (1 - beta2) * (Pi_H^g grad)^2
m_hat <- state.exp_avg     / (1 - beta1^t)
v_h_hat <- state.exp_avg_sq_h / (1 - beta2^t)

# Vertical mode: 'adam', 'sgd', or 'frozen'.
g_v_scalar <- sigma_g(Pi_V^g grad)
state.exp_avg_sq_v <- beta2 * state.exp_avg_sq_v + (1 - beta2) * g_v_scalar^2
v_v_hat <- state.exp_avg_sq_v / (1 - beta2^t)
m_v_scalar <- sigma_g(Pi_V^g m_hat)
update_v_scalar <- {
    'adam':   -lr * m_v_scalar / (sqrt(v_v_hat) + eps),
    'sgd':    -lr * m_v_scalar,
    'frozen': 0,
}[vertical_mode]
update_v <- lambda_g(p, update_v_scalar)

# Horizontal: per-coord Adam, then re-project to horizontal frame.
update_h_raw <- (Pi_H^g m_hat) / (sqrt(v_h_hat) + eps)
update_h <- Pi_H^g update_h_raw

p <- p + update_v + (-lr) * update_h
\end{verbatim}
\end{small}

The outer $\Pi_H^g$ on \texttt{update\_h\_raw} is the re-projection that makes per-coordinate division equivalent to running Adam directly in a $G$-equivariant horizontal frame (cf.~Theorem~\ref{thm:ddcadam_rate}).

\subsection{Cross-entropy bias-shift gauge ($G = \reals$)} For 1-D output bias $b \in \reals^C$:
\begin{align*}
\Pi_V^g(b, g) &= (\bar g) \cdot \mathbf{1}_C, \quad \bar g = \tfrac{1}{C} \sum_i g_i \\
\Pi_H^g(b, g) &= g - (\bar g) \cdot \mathbf{1}_C \\
\sigma_g(\cdot \mathbf{1}_C) &= \bar g \quad \text{(scalar)} \\
\lambda_g(p, s) &= s \cdot \mathbf{1}_C
\end{align*}

\subsection{Cross-entropy row-shift gauge ($G = \reals^d$)} For 2-D unembed $W \in \reals^{C \times d}$:
\begin{align*}
\Pi_V^g(W, g)[i, j] &= \bar g_j, \quad \bar g_j = \tfrac{1}{C} \sum_i g[i, j] \\
\Pi_H^g(W, g) &= g - \mathbf{1}_C \bar g^\top \\
\sigma_g(\mathbf{1}_C \bar g^\top) &= \bar g \in \reals^d \\
\lambda_g(p, s) &= \mathbf{1}_C s^\top
\end{align*}

\subsection{ReLU rescale gauge ($G = \reals^+$ on $(W_l, W_{l+1})$)} The construction in joint log-norm coordinates: with $\rho_l = \|W_l\|_F$, $\rho_{l+1} = \|W_{l+1}\|_F$, and the orthogonal basis $u = (\log \rho_l + \log \rho_{l+1})/\sqrt 2$ (joint-norm), $v = (\log \rho_l - \log \rho_{l+1})/\sqrt 2$ (gauge mode):
\begin{small}
\begin{verbatim}
g_log_rho_l   <- <g_l,   W_l>      # G-invariant scalar
g_log_rho_lp1 <- <g_lp1, W_lp1>    # G-invariant scalar
g_u <- (g_log_rho_l + g_log_rho_lp1) / sqrt(2)   # joint-norm gradient
g_v <- (g_log_rho_l - g_log_rho_lp1) / sqrt(2)   # gauge-mode gradient

# Apply scalar Adam to g_u (joint norm) and g_v (vertical, per vertical_mode).
delta_u <- adam_step(state.m_u, state.v_u, g_u, ...)
delta_v <- {adam, sgd, frozen}_step(state.m_v, state.v_v, g_v, ...)
delta_log_rho_l   <- (delta_u + delta_v) / sqrt(2)
delta_log_rho_lp1 <- (delta_u - delta_v) / sqrt(2)

# Tangential: per-coord Adam on (rho * g_tan), then re-project to tangent.
g_l_tan   <- g_l   - (<g_l,   W_l>   / |W_l|^2  ) * W_l
g_lp1_tan <- g_lp1 - (<g_lp1, W_lp1> / |W_lp1|^2) * W_lp1
update_U_l   <- per_coord_adam(rho_l   * g_l_tan,   ...);   re-project to tangent at W_l
update_U_lp1 <- per_coord_adam(rho_lp1 * g_lp1_tan, ...); re-project to tangent at W_lp1

# Combine: multiplicative radial + additive tangential.
W_l   <- exp(delta_log_rho_l)   * W_l   - lr * rho_l   * update_U_l
W_lp1 <- exp(delta_log_rho_lp1) * W_lp1 - lr * rho_lp1 * update_U_lp1
\end{verbatim}
\end{small}

The $\rho \cdot g_\mathrm{tan}$ scaling makes the per-coordinate Adam state $G$-invariant: under $c$-action, $\rho' g'_\mathrm{tan} = c\rho \cdot (1/c)g_\mathrm{tan} = \rho g_\mathrm{tan}$.

\subsection{LayerNorm scale gauge ($G = (\reals^+)^d$ on $(\gamma, W_\mathrm{next})$)} Per-channel application of the ReLU-rescale algorithm: each channel $i$ has its own $(\rho_l^{(i)}, \rho_{l+1}^{(i)})$ pair and its own scalar Adam state for joint-norm and gauge-mode coordinates. $\gamma$ is 1-D so the only direction component is its sign (radial-only); $W_\mathrm{next}$ has $\mathrm{out} - 1$ tangential dims per column for per-coordinate Adam. The action on column $i$ is $(\gamma_i, W_\mathrm{next}[:, i]) \mapsto (c_i \gamma_i, c_i^{-1} W_\mathrm{next}[:, i])$, and the implementation is a vectorised per-channel copy of the ReLU-rescale algorithm.

\subsection{Chained ReLU rescale ($G = (\reals^+)^{L-1}$ on $(W_1, \ldots, W_L)$)} The combined gauge group acts on log-norms $(\log \rho_1, \ldots, \log \rho_L) \in \reals^L$ via the $L \times (L-1)$ bidiagonal matrix $B$ with $+1$ on the diagonal and $-1$ on the sub-diagonal. The orthogonal decomposition of $\reals^L$ into the $1$-D horizontal-radial subspace (the $\mathbf{1}_L / \sqrt L$ direction) and the $(L-1)$-D vertical subspace (the orthogonal complement, equivalently the zero-mean subspace) uses an orthonormal basis of zero-mean vectors, our implementation uses the discrete cosine basis $v_k[\ell] = \cos(\pi k (\ell + \tfrac12) / L)$ for $k = 1, \ldots, L-1$, which is exactly orthogonal to $\mathbf{1}_L$ under the standard inner product. Scalar Adam on the $1$-D joint-norm coordinate; per-direction (`adam', `sgd', or `frozen') updates on the $(L-1)$-D vertical coordinates; per-coordinate Adam on the $\rho_\ell \cdot g_\mathrm{tan}^{(\ell)}$ tangential gradients with re-projection. This is jointly $G$-equivariant for the product group, unlike the sequential composition of single ReLU-rescale gauges.

\subsection{Radial update modes (\texttt{log} and \texttt{linear})}
The rescale gauges above update the gauge-invariant joint scale $s = \rho_l^{p}\,\rho_{l+1}$ on one of two radial coordinates ($p$ the activation's positive-homogeneity degree, $p=1$ for ReLU, $p=2$ for squared-ReLU). The \texttt{log} mode (legacy) runs scalar Adam on the log-scale gradient, a step of constant log-magnitude; the decoupled weight decay is then a constant subtraction $\sqrt{p^2+1}\,\eta\lambda$ in log-space with no fixed point, so it can drive a load-bearing block to zero (\S\ref{app:radial}). The \texttt{linear} mode (the default) runs Adam on the linear scale $s$, inheriting AdamW's fixed point:
\begin{small}
\begin{verbatim}
s         <- rho_l^p * rho_lp1                            # gauge-invariant joint scale
g_s       <- (p*g_log_rho_l + g_log_rho_lp1)/(p^2+1) / max(s, s_floor)   # dL/ds, 1/s floored
loss_step <- -lr * adam_step(state.m_s, state.v_s, g_s)   # constant LINEAR magnitude ~lr
s_new     <- s * exp(-(p^2+1) * lr * wd) + loss_step      # multiplicative weight decay
delta_u   <- clamp(log(s_new) - log(s), +-max_log_step) / sqrt(p^2+1)
\end{verbatim}
\end{small}
The fixed point is $s^\star = 1/((p^2+1)\lambda)$ and the amplification floor is $s_\mathrm{floor} = s^\star/5$, both set scale-free by $(\lambda, p)$. The gauge-mode and tangential updates are unchanged from the \texttt{log} mode; only the radial coordinate differs.

\subsection{Attention head rotation gauge ($G = O(d_{\rm head})^{n_{\rm heads}}$ on QK or VO weight pair)} The non-abelian rotation gauge acts on $W_Q, W_K \in \reals^{n_{\rm heads} \times d_{\rm model} \times d_{\rm head}}$ (or $W_V, W_O$) by right-multiplication: $W_Q[h] \mapsto W_Q[h] \cdot O_h$, $W_K[h] \mapsto W_K[h] \cdot O_h$ for $O_h \in O(d_{\rm head})$. The horizontality condition (per head) reduces to a continuous Lyapunov equation:

\begin{small}
\begin{verbatim}
# Project gradient onto horizontal complement of QK gauge orbit.
S[h, a, b] = sum_d (W_Q[h, d, a] * g_Q[h, d, b]
                  + W_K[h, d, a] * g_K[h, d, b])
M[h, a, b] = sum_d (W_Q[h, d, a] * W_Q[h, d, b]
                  + W_K[h, d, a] * W_K[h, d, b])
skew_S[h] = (S[h] - S[h]^T) / 2
A[h] = solve_lyapunov_truncated(M[h], skew_S[h])     # M.A + A.M = 2.skew_S
g_Q_h = g_Q - einsum('hda,hab->hdb', W_Q, A)         # horizontal projection
g_K_h = g_K - einsum('hda,hab->hdb', W_K, A)
\end{verbatim}
\end{small}

The Lyapunov solve uses the eigendecomposition of $M$ with truncated-eigenvalue pseudoinverse: $M = U \Lambda U^\top$ then $A = U B U^\top$ with $B[i,j] = (U^\top \cdot 2 \cdot \mathrm{skew}(S) \cdot U)[i,j] / (\lambda_i + \lambda_j)$, zeroing entries where $(\lambda_i + \lambda_j)$ is below $\mathrm{rcond} \cdot \max_{i,j} (\lambda_i + \lambda_j)$ to handle near-rank-deficient heads. The four second-moment (\texttt{v\_mode}) constructions branch at how $\hat v$ is maintained on $g_Q^h$ and $g_K^h$:

\begin{small}
\begin{verbatim}
# v_mode = 'per_head_scalar': one scalar v per head (Frobenius-norm sq).
v_Q[h] <- beta2 * v_Q[h] + (1 - beta2) * |g_Q_h[h]|^2_F / (d_model * d_head)
v_K[h] <- beta2 * v_K[h] + (1 - beta2) * |g_K_h[h]|^2_F / (d_model * d_head)
update_Q = (m_Q / b1) / (sqrt(v_Q[:, None, None] / b2) + eps)

# v_mode = 'per_head_matrix': symmetric PSD v per head, matrix sqrt-inverse.
gQ_outer[h, a, b] <- sum_d g_Q_h[h, d, a] * g_Q_h[h, d, b] / d_model
v_Q[h] <- beta2 * v_Q[h] + (1 - beta2) * gQ_outer[h]
v_Q_inv_sqrt[h] <- (v_Q[h] / b2 + eps^2 * I)^{-1/2}                # via eigh
update_Q[h, d, c] = sum_a (m_Q[h, d, a] / b1) * v_Q_inv_sqrt[h, a, c]

# v_mode = 'body_frame': per-coord v in eigenbasis of M = W_Q^T W_Q + W_K^T W_K.
M = W_Q^T W_Q + W_K^T W_K
U_t = eigh(M).eigenvectors
if U_prev is not None:
    R = U_prev^T @ U_t                                   # frame transition
    m_Q, m_K, v_Q, v_K <- rotate-by-R-on-d_head-axis     # m exact, v |.|
g_Q_body = g_Q_h @ U_t                                   # per-head right-mul
g_K_body = g_K_h @ U_t
m_Q, m_K   <- standard-Adam-first-moment-update(g_*_body)
v_Q, v_K   <- standard-Adam-second-moment-update(g_*_body)
update_Q_body = (m_Q / b1) / (sqrt(v_Q / b2) + eps)
update_Q = update_Q_body @ U_t^T                         # back to world frame
U_prev <- U_t
\end{verbatim}
\end{small}

\noindent \texttt{body\_frame\_topk}, the deployed default, is this \texttt{body\_frame} branch with the per-coordinate $\hat v$ update applied only on the eigendirections of $M$ with $\lambda_i / \lambda_{\max} \ge \tau$ (default $10^{-2}$); the directions below the cut take the momentum-only step.

\noindent The body frame $U_t$ is recomputed only when the per-head Gram $M$ drifts past a threshold (the adaptive trigger, three recomputes over $200$ steps at depth $24$) and is reset every $K$ steps; a fixed every-step eigendecomposition costs $+40\%$ and can diverge where the adaptive trigger holds (\S\ref{sec:discussion}).

\noindent The VO gauge is structurally identical with $S = W_V^\top g_V - W_O^\top g_O$ (sign-flip on the W\_O term reflects that $W_O$ block transforms by $P_h$ rather than $P_h^\top$ once the $.T$ on the canonical \texttt{out @ W\_O.T} matmul is taken into account; see \S\ref{ssec:gauges}).

\noindent The four constructions' $G$-equivariance is proved in Appendix~\ref{app:proofs:vmode_equiv}.

\subsection{RoPE-torus QK gauge ($G = SO(2)^{d_{\rm head}/2}$, the deployed default)}
A rotary attention block carries only the torus of rotations commuting with the rotary embedding, $SO(2)^{d_{\rm head}/2}$ (\S\ref{ssec:rope}), a subgroup of the full $O(d_{\rm head})$ above. Its horizontal projection is closed-form, one scalar per rope plane, with no eigendecomposition or Lyapunov solve. Split each head's weight into the rope half-pairs $W^L = W[\dots, {:}p]$, $W^R = W[\dots, p{:}2p]$ ($p = d_{\rm head}/2$), and the gradient likewise:
\begin{small}
\begin{verbatim}
# per plane j, summed over the QK pair and reduced over d_model:
num[j] <- sum_{X in Q,K} ( W_R^X[:,j] . g_L^X[:,j]  -  W_L^X[:,j] . g_R^X[:,j] )
den[j] <- sum_{X in Q,K} ( ||W_L^X[:,j]||^2 + ||W_R^X[:,j]||^2 )
a[j]   <- num[j] / den[j]                  # one torus angle per plane
# horizontal projection (remove the per-plane SO(2) generator):
g_L[:,j] <- g_L[:,j] - a[j] * W_R[:,j]
g_R[:,j] <- g_R[:,j] + a[j] * W_L[:,j]
\end{verbatim}
\end{small}
The projected gradient is orthogonal to every torus generator by construction. Under MQA/GQA the angle is shared per (kv-group, plane), one $a[j]$ for each group's query heads and its single key (Lemma~\ref{lem:mqa_gqa_rotation_gauge}). The second moment then runs on the projected gradient as in the full-rotation gauge; this is the QK gauge DDC carries into the rotary experiments, the full $O(d_{\rm head})$ gauge above being the ablation arm of \S\ref{ssec:rope}.

\subsection{\textsc{DDCMuon}: the scaled-polar orthogonaliser}
Muon orthogonalises the momentum $M$ to its polar factor $\mathrm{polar}(M) = UV^\top$. A gauge-equivariant variance reduction reduces in closed form to a single gauge-invariant scale on that polar factor, where NorMuon's~\citep{LiLiu25NorMuon} per-neuron second moment would break the QK/VO gauge (a per-row scale is not $O(d_{\rm head})$-invariant):
\begin{small}
\begin{verbatim}
P      <- polar_sharpen(M, steps=4)   # matmul-only, no SVD
C      <- ||M||_F / sqrt(rank)        # gauge-invariant RMS scale
update <- C * P                       # scaled-polar update
\end{verbatim}
\end{small}
\texttt{polar\_sharpen} is the matmul-only Newton--Schulz iteration tuned to push the singular values toward $1$ (no SVD or eigendecomposition), batched across heads. The rotation gauge's horizontal projection above is applied to $M$ first; the scaled-polar step then orthogonalises it. The single gauge-invariant scale $C$ is what keeps the construction $O(d_{\rm head})$-equivariant where NorMuon's per-row scale is not, and the matmul-only polar holds the wall-clock within $\sim 10$ to $15\%$ of plain Muon, and at parity on some GPUs, $4.2\times$ cheaper than an exact per-head SVD (\S\ref{sec:discussion}).
 \section{Experimental details}
\label{app:experiments}

All experiments use deterministic seeding (seeds 42, 142, 242; the over-parametrisation and diagonal-network cells add 342, 442 for five-seed reads). The synthetic cells (the diagonal and 2-layer linear networks and the teacher-student rescale and LayerNorm tasks) run in minutes on CPU or Apple MPS; the grokking transformers, the language model, and the vision transformer run on CUDA (NVIDIA 3090). Every geometry read uses the true-MC (residual-free) Fisher at $n/d \ge 100$; the empirical Fisher is kept only for observed-curvature reads. The \emph{true-MC} Fisher estimates the Fisher information metric $F = \mathbb{E}_x\,\mathbb{E}_{\hat y \sim p_\theta(\cdot\mid x)}\!\left[\nabla_\theta \log p_\theta(\hat y \mid x)\,\nabla_\theta \log p_\theta(\hat y \mid x)^\top\right]$ by sampling the label $\hat y$ from the model's own predictive distribution; the \emph{empirical} Fisher plugs in the data label instead, which contaminates the estimate with the training residual away from interpolation. Throughout, $n$ is the number of Monte-Carlo samples and $d$ the dimension of the parameter block whose Fisher is read, so $n/d \ge 100$ is the sample floor for a well-conditioned spectrum; this cheap-spectral-read convention follows the dead-direction signature suite of \citet{Shirodkar2026DDS}.

\begin{table}[h]
\centering
\small
\setlength{\tabcolsep}{4.5pt}\renewcommand{\arraystretch}{1.2}\begin{tabularx}{\textwidth}{@{}l >{\raggedright\arraybackslash}X l c c l@{}}
\toprule
Run & Architecture / task & Base & Seeds & Steps & WD \\
\midrule
Over-trained LM      & depth-12 Llama, 10M-token FineWeb-edu, over-training & Adam / Muon & 3 & 15000 & 0.1 \\
WD vs projection     & 1-block transformer, $(a{+}b)\bmod 113$, $d\in\{128,64,32\}$ & Muon & 5 & 25000 & swept \\
Grokking, RoPE       & depth-8 RoPE transformer, modular addition & Muon & 3 & 5000 & 1.0 \\
Grokking at depth 24 & depth-24 transformer, synthetic arithmetic & Muon & 3 + 5 & 6000 & 1.0 \\
Curvature            & depth-8 RoPE grok bench, converged Fisher & Muon & 3 & 5000 & 1.0 \\
ViT compression      & ViT-S from scratch, ImageNet-100 & Adam & 3 & 50000 & 0.1 \\
Diagonal net         & Pesme diagonal net ($d{=}64$), $r{=}2$ sparse target & SGD/Adam & 5 & 4000 & 0 \\
Radial coordinate    & 1-hidden block + LayerNorm, sparse parity & Muon & 3 & 10000 & $\{1,2,5\}$ \\
\bottomrule
\end{tabularx}
\caption{Provenance and setup for the runs behind the findings. All geometry reads use the residual-free true-MC Fisher at $n/d \ge 100$ except the deep-linear and diagonal-network cells, which read the analytic Gauss-Newton Fisher. Seeds are 42, 142, 242 (plus 342, 442 for the five-seed runs). The figures of \S\ref{sec:experiments} regenerate from the committed per-step arrays in the released code.}
\label{tab:cohort_provenance}
\end{table}
 
\subsection{The readout row-shift gauge on the grokking transformer}\label{app:grok}
\paragraph{Testbed and arms.} This study uses the testbed of \S\ref{app:grok_setup} with ReLU activation, full-batch gradients, learning rate $10^{-3}$, $\beta = (0.9, 0.999)$, weight decay $1.0$, and 8000 steps, with only the cross-entropy readout row-shift gauge active in place of the full gauge set: the gauge is bound to the unembed weight with its vertical mode frozen. Four arms set the gauge against the standard readout fix, AdamW with plain cross-entropy, AdamW with Z-loss at $\alpha = 10^{-4}$, AdamW with Z-loss at $\alpha = 1.0$, and the frozen-gauge \eqadam{}. The per-checkpoint observable, read every 50 steps, is $\|W_v\|_F$ for $W_v = \mathbf{1}_C \bar W^\top$, the constant-row projection of the unembed weight that the row-shift orbit moves along, together with train and test accuracy on held-out pairs.

\paragraph{The gauge engages where AdamW leaks.} A separate three-seed isolation reads the readout gauge fraction $C\|m\|^2 / \|W\|^2$ off the true-MC Fisher, the share of the unembed weight lying along the gauge orbit. A matched AdamW leaks it to $0.145$--$0.16$ as its per-coordinate second moment breaks the column-mean invariance of the cross-entropy gradient, while \eqadam{} pins it at the floor ($2.9 \times 10^{-15}$), confirming the gauge holds the readout mode AdamW moves.

\begin{figure}[t]
\centering
  \includegraphics[width=0.70\linewidth]{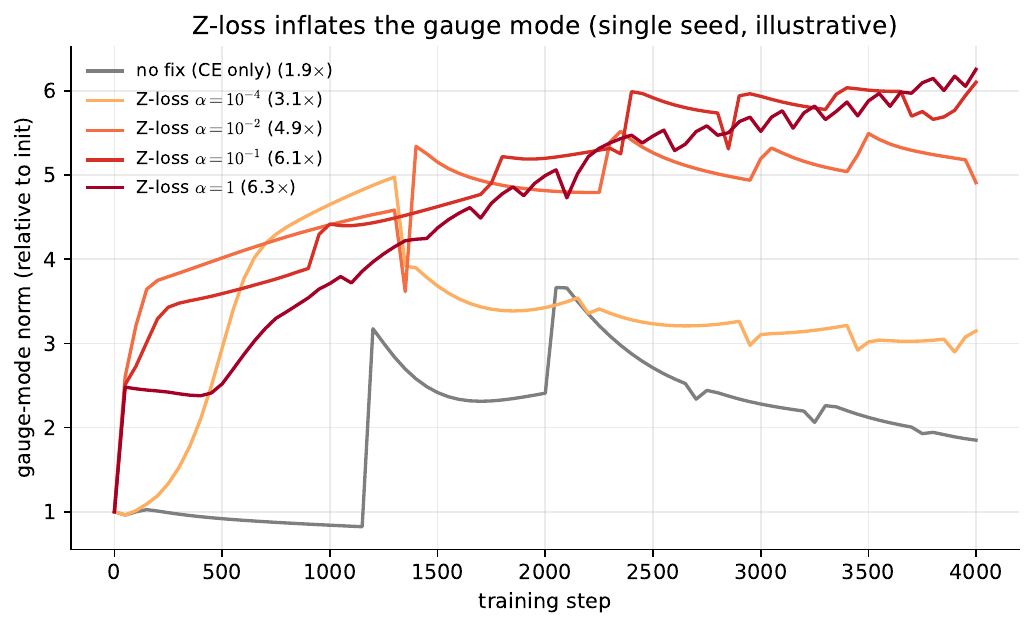}

\caption{Z-loss inflates the gauge mode it is meant to contain (grokking transformer, single seed, illustrative). Adding the $\alpha(\log Z)^2$ penalty gives the readout gauge direction a non-zero gradient that Adam's normaliser turns into a steady drift, so a stronger penalty grows the gauge mode rather than holding it. The multi-seed containment under \eqadam{} is Figure~\ref{fig:ce_shift_gauge}.}
\label{fig:zloss_sweep}
\end{figure}
 
\subsection{The ReLU-rescale gauge testbed}\label{app:relu}
\paragraph{Testbed and gauge.} A 2-layer ReLU network $f(x) = W_2\,\mathrm{ReLU}(W_1 x)$ with $d_\mathrm{in} = 8$, $d_\mathrm{hidden} = 16$, $d_\mathrm{out} = 4$ learns a teacher of the same architecture with frozen random weights at scale $1.0$, the student initialised at scale $0.5$. Training is full-batch MSE against $y = \mathrm{teacher}(x)$ on 128 samples $x \sim \mathcal{N}(0, I)$, learning rate $5 \cdot 10^{-3}$, weight decay $0.01$, 4000 steps, under two optimizers: AdamW, and \eqadam{} with a single ReLU-rescale gauge bound to $(W_1, W_2)$ and a frozen vertical mode.

\paragraph{The gauge-mode observable.} The diagnostic, read every 50 steps, is the $G$-invariant log-norm coordinate $r(t) = \log(\|W_1\|_F / \|W_2\|_F)$, the mode the rescale gauge holds. Its cross-seed drift under each optimizer is the ReLU-rescale row of the containment result in Table~\ref{tab:multiplicative} and Figure~\ref{fig:relu_gauge}.

\begin{figure}[t]
\centering
  \includegraphics[width=0.95\linewidth]{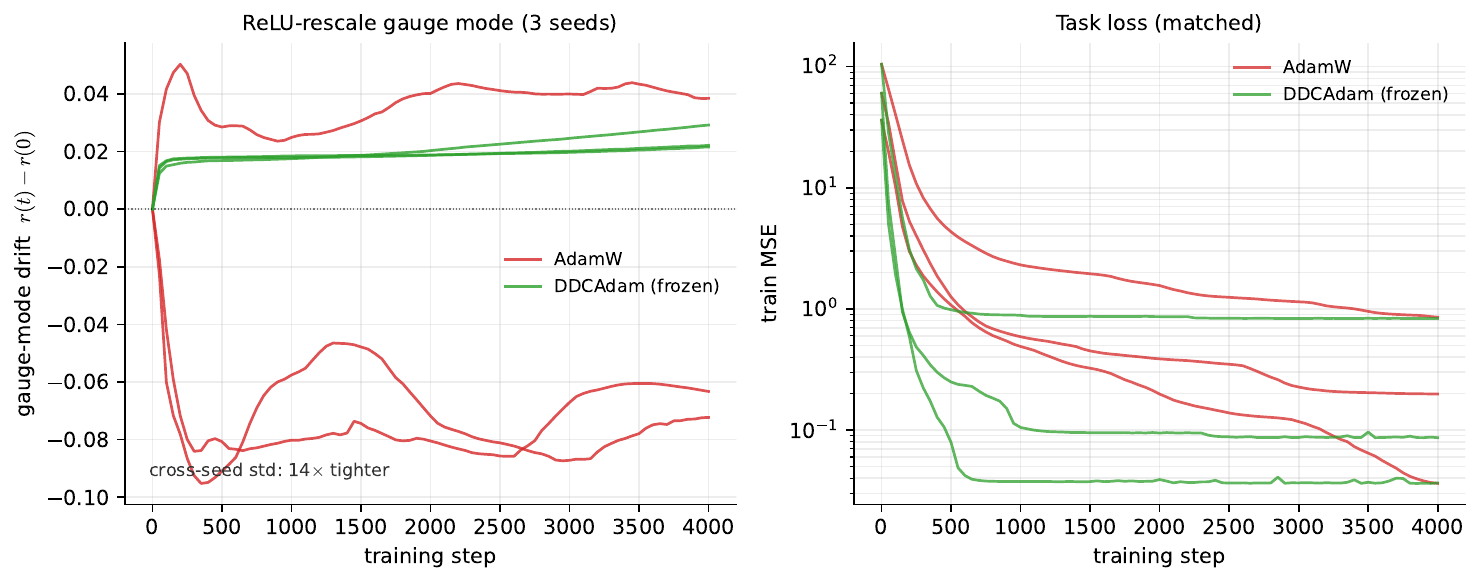}

\caption{The ReLU-rescale gauge on a two-layer teacher-student network (full-batch MSE, 3 seeds). Left: the gauge-mode drift $r(t)-r(0)$ with $r=\log(\|W_1\|_F/\|W_2\|_F)$. The frozen \eqadam{} construction holds the mode to within $\pm 0.003$ across seeds; AdamW lets it wander by $-0.032 \pm 0.050$, a cross-seed spread $14\times$ wider. Right: the matched training loss, where the two optimisers descend together within the seed spread. Holding the gauge mode costs nothing on the task.}
\label{fig:relu_gauge}
\end{figure}
 
\subsection{The LayerNorm-scale gauge testbed}\label{app:ln}
\paragraph{Testbed and gauge.} A single LayerNorm-and-readout map $f(x) = W_\mathrm{next}\,\mathrm{LN}(x; \gamma)$, per-sample mean and standard-deviation normalisation followed by a per-channel scale $\gamma$ with no shift, $d_\mathrm{in} = 16$, $d_\mathrm{out} = 8$, with $\gamma$ initialised at $\mathcal{N}(1, 0.5)$ so the gauge mode is non-trivial at init. Training is MSE on 128 samples, learning rate $10^{-2}$, weight decay $0.01$, 4000 steps, under two optimizers: AdamW, and \eqadam{} with a single LayerNorm-scale gauge bound to $(\gamma, W_\mathrm{next})$ and a frozen vertical mode.

\paragraph{The gauge-mode observable.} The diagnostic is the per-channel gauge mode $r_i(t) = \log|\gamma_i| - \log\|W_\mathrm{next}[:, i]\|$, summarised by $\max_i |r_i|$ and $\|r\|_2$. Its cross-seed drift under each optimizer is the LayerNorm-scale row of the containment result in Table~\ref{tab:multiplicative} and Figure~\ref{fig:ln_gauge}.

\begin{figure}[t]
\centering
  \includegraphics[width=0.95\linewidth]{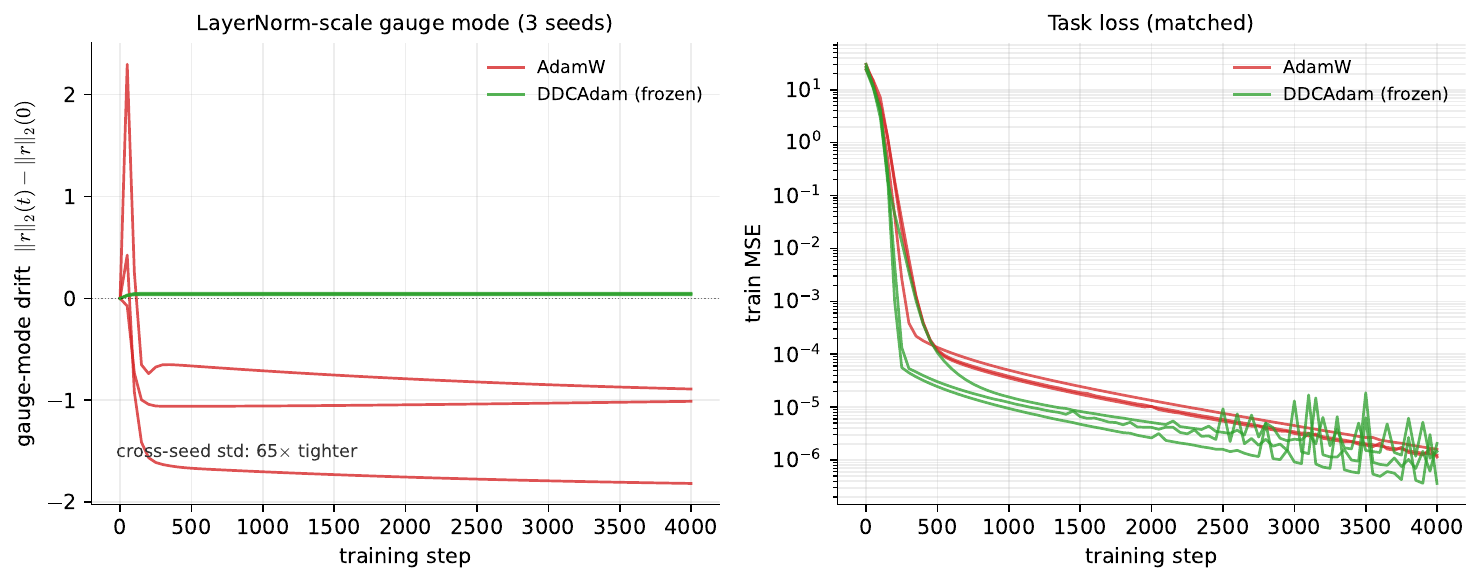}

\caption{The LayerNorm-scale gauge on the synthetic block $f(x)=W\,\mathrm{LN}(x;\gamma)$ (MSE, 3 seeds). Left: the gauge-mode drift $\|r\|_2(t)-\|r\|_2(0)$ with $r_i=\log|\gamma_i|-\log\|W[:,i]\|$. The frozen \eqadam{} construction holds the mode to within $\pm 0.006$ across seeds; AdamW drifts it by $-1.239 \pm 0.412$, a spread $65\times$ wider. Right: the matched training loss. The gauge freezes a mode AdamW moves by order one, at the same task loss.}
\label{fig:ln_gauge}
\end{figure}
 
\subsection{Per-coordinate gauge drift on a diagonal linear network}\label{app:dln_drift}
\paragraph{Testbed and gauge.} The Pesme diagonal linear network $\beta = u \odot v$ ($u, v \in \reals^{64}$) learns an $r = 2$ sparse target under MSE, so its minimum has order $k = 2$. The network carries the per-coordinate rescale gauge $(u_i, v_i) \mapsto (c_i u_i, c_i^{-1} v_i)$ with group $G = (\reals^+)^{64}$. The initialisation is canonical-aligned with a per-coordinate gauge-symmetry break, giving AdamW a mode to leak, and training runs 4000 steps at learning rate $5 \times 10^{-3}$ over 5 seeds on three arms: SGD, AdamW, and \eqadam{}.

\paragraph{Closing the full per-coordinate gauge.} The \eqadam{} arm closes all 64 per-coordinate modes. A naive wiring would close only the one-dimensional global subgroup, which controls $1$ of the $64$ modes and fails the engagement check, so the gauge has to be wired per coordinate for the construction to engage.

\paragraph{The construction freezes the mode AdamW leaks.} AdamW leaks the per-coordinate gauge mode $r_i = \log|u_i / v_i|$ to an RMS drift of $9.54 \pm 2.1$, where \eqadam{} freezes it to $3.1 \times 10^{-6} \pm 6.8 \times 10^{-7}$, a ratio of $\sim\!3 \times 10^{6}$; SGD drifts it less than AdamW ($3.75 \pm 0.11$) but does not freeze it. The gauge-mode norm follows the same ordering, \eqadam{} lowest at $23.1$ against AdamW's $90.1$ (Figure~\ref{fig:dln_drift}).

\paragraph{The order is optimizer-independent.} Read off-trajectory at the structural dead point $(0, 0)$ along the balanced joint mode via the true Gauss-Newton Fisher, the order is $k = 2.000$ on every arm. The geometry fixes the order; only the gauge drift moves with the optimizer.

\begin{figure}[t]
\centering
  \includegraphics[width=0.94\linewidth]{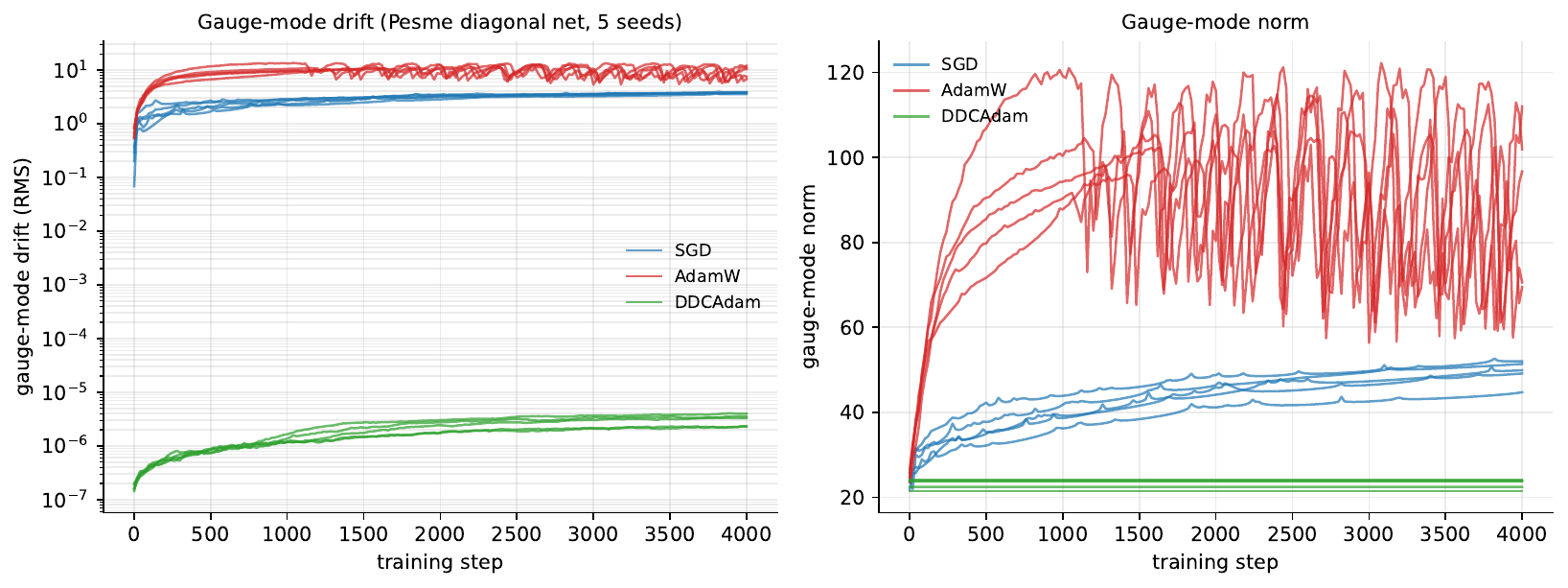}

\caption{Gauge containment on the Pesme diagonal linear network ($G = (\reals^+)^{64}$, order $k=2$, 5 seeds). Left: the per-coordinate gauge-mode drift (RMS). \eqadam{} holds it near $10^{-6}$ while AdamW leaks it to order $10$ and SGD to order $1$, a containment of roughly six orders over AdamW. Right: the gauge-mode norm, where the same ordering holds, \eqadam{} lowest, AdamW highest. The construction freezes the mode the adaptive optimiser otherwise drifts.}
\label{fig:dln_drift}
\end{figure}
 
\subsection{Over-training and rate-readability on a depth-12 language model}\label{app:otr}
\paragraph{Testbed and arms.} A 12-layer Llama-3-style transformer (hidden $768$, 6 attention heads with 2 key/value heads, head dimension $128$, SwiGLU width $2048$, RMSNorm, RoPE; $\sim$75M parameters) trains on a fixed 10-million-token FineWeb-edu slice held far below the model's capacity, so the run enters the over-training regime: it memorises the corpus and validation loss rises. Four optimizer arms, $\{$AdamW, \eqadam{}$\}$ and $\{$Muon, \textsc{DDCMuon}$\}$, run at matched weight decay $0.1$, 15000 steps ($\sim$24 epochs), batch 16, sequence length 1024, over 3 seeds, with the per-layer Fisher captured in-loop as true-MC at $n / d = 100$. The Muon baseline is textbook NS5; a Muon arm running the Newton-Schulz orthogonaliser the \textsc{DDCMuon} arm shares, without the gauge, runs as a control that separates the gauge from the orthogonaliser. On the Adam base the output activation $\sigma_{\min}$ and Fisher $\lambda_{\min}$ collapse toward the floor under \eqadam{} where AdamW inflates (Figure~\ref{fig:rate_collapse}).

\begin{figure}[t]
\centering
  \includegraphics[width=0.94\linewidth]{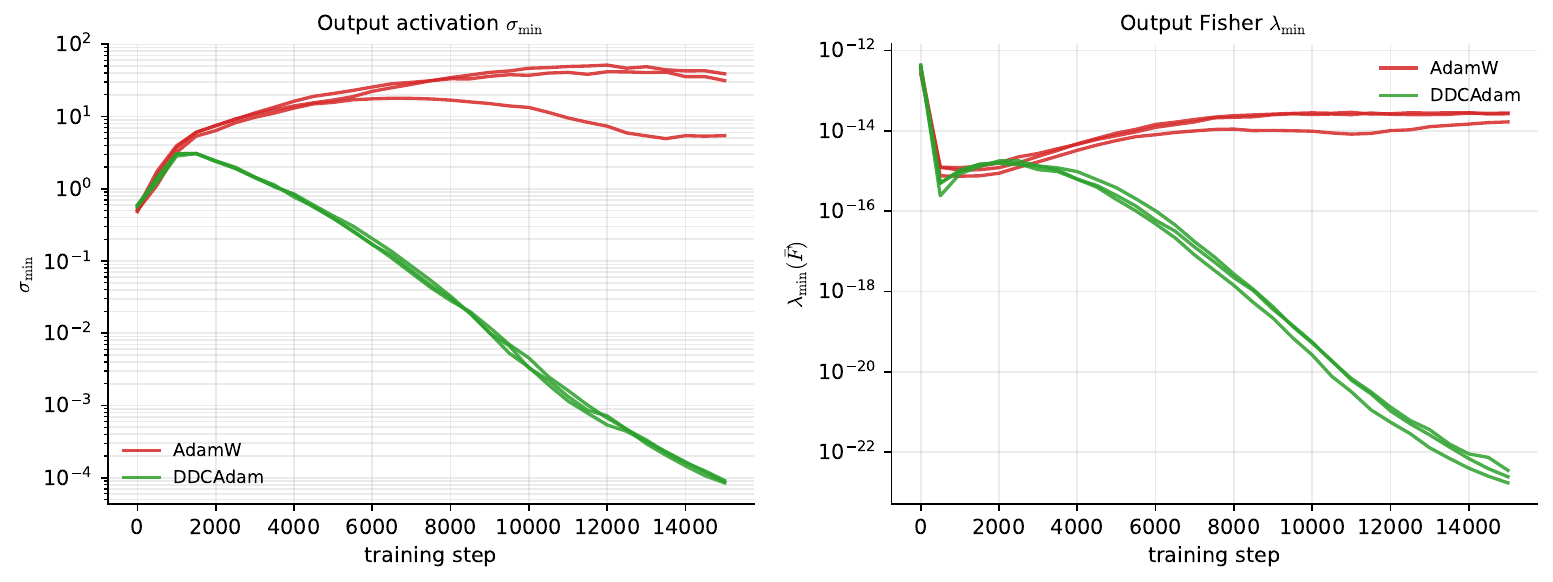}

\caption{Reading the rate at language-model scale, on the depth-12 over-trained model (3 seeds, in-loop true-MC Fisher). The output-layer geometry separates completely between the two optimisers. Left: under \eqadam{} the activation $\sigma_{\min}$ collapses toward the floor by nearly four orders, while under AdamW it inflates. Right: the output Fisher $\lambda_{\min}(\bar\fisher)$ collapses about ten orders under \eqadam{} and stays pinned near its initial value, rising away from the floor, under AdamW. The collapse is the readable rate; the inflation is what makes AdamW unreadable.}
\label{fig:rate_collapse}
\end{figure}
 
\paragraph{The gauge resists over-training on the Adam base.} \eqadam{} holds a validation-minus-train loss gap of $+0.67$ against AdamW's $+5.88$: AdamW memorises to train accuracy $0.70$ and then validation collapses, while \eqadam{} holds train accuracy $0.31$. The gauge lands at higher validation accuracy ($0.252$ against $0.202$) and lower validation loss ($5.25$ against $7.17$).

\paragraph{On the Muon base the gauge reaches a lower-loss minimum.} The Muon baselines control the train-validation gap on their own, and the gauge does not improve that gap; \textsc{DDCMuon} instead settles at a lower validation loss ($4.78$ against textbook Muon's $4.88$) and slightly higher validation accuracy ($0.262$ against $0.254$), consistent across seeds, the orbit-holding gain rather than added over-training resistance. The control Muon arm, sharing the \textsc{DDCMuon} orthogonaliser without the gauge, lands at textbook Muon's loss ($4.89$ against $4.88$), so the gain comes from the gauge.

\paragraph{Rate-readability separates the two profiles.} A \emph{cell} pairs one network layer with one geometric observable (the layer's smallest quotient-Fisher eigenvalue, its activation $\sigma_{\min}$, its effective rank, and the like); a cell is \emph{readable} when that observable descends toward the minimum across the fit window, so a power-law slope can be fit against the $\bar t^{2(k-1)}$ of Theorem~\ref{thm:ddcadam_rate}, and \emph{unreadable} when it ascends away from the minimum. Across 65 layer-by-observable cells, AdamW ascends away from the minimum in 55 and reaches the asymptotic regime in none. \eqadam{} descends in 32 and reaches the asymptotic regime at the output layer ($h_{12}$, its smallest quotient-Fisher eigenvalue), the activation $\sigma_{\min}$ there collapsing by up to $3.8$ orders. No cell carries a finite-slope fit under both arms, so the two readable profiles separate completely.

\begin{figure}[t]
\centering
  \includegraphics[width=0.98\linewidth]{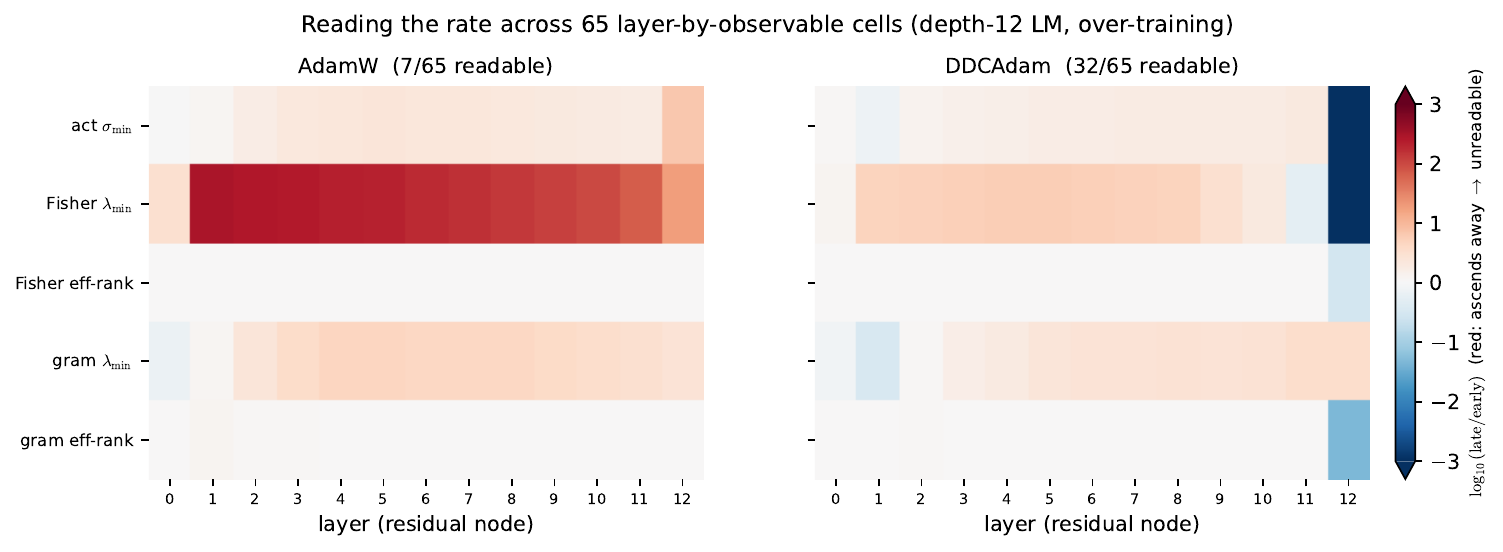}

\caption{The 65 layer-by-observable cells of the over-training read (depth-12 LM, seed-median). Colour is the geometry's net change over training, $\log_{10}(\mathrm{late}/\mathrm{early})$: red is a rise away from the minimum (the unreadable direction, where no rate can be fit), blue a descent toward the floor. Under AdamW the Fisher $\lambda_{\min}$ inflates at every layer (the uniform red band); under \eqadam{} it stays near the floor and collapses at the output. The asymptoticity-gated rate fit reads a finite rate in 32 of 65 cells under \eqadam{} against 7 under AdamW, and no cell is readable under both.}
\label{fig:rate_grid}
\end{figure}
 
\subsection{Reliability of the gauge at depth 24}\label{app:d24}
\paragraph{Conditions at scale.} A depth-24 transformer ($n_\mathrm{embd} = 1536$, $d_\mathrm{mlp} = 6144$, 24 blocks at $3\times$ the depth-8 grokking recipe's width) trains on synthetic modular arithmetic with base Muon, weight decay $1.0$, 6000 steps, device batch 4, total batch 4096, sequence length 128. Four conditions cross the gauge against the activation, $\{\text{vanilla Muon}, \text{gauge}\} \times \{\text{squared-ReLU}, \text{GELU}\}$, over 3 seeds; the gauge arm wires the RoPE-torus QK/VO rotation gauge with the deployed body-frame vertical mode, with the CE-shift, LayerNorm-scale, and ReLU-rescale gauges switched off at this scale. Fisher reads are post-hoc true-MC on the saved checkpoints.

\paragraph{The gauge groks where vanilla Muon does not.} Taking full grokking as $\mathrm{acc}_{\max} \ge 0.9$, the gauge groks $3/3$ seeds on both activations ($0.959$ GELU, $0.946$ squared-ReLU) while vanilla Muon groks $0/3$ on both, reaching only partial generalisation ($0.652$, $0.783$). The gauge also compresses the interior activation, the block-12 smallest singular value sitting at $0.5$--$1.8$ against vanilla's $6$--$15$ (Figure~\ref{fig:d24_block}).

\paragraph{A five-seed companion widens the split.} A companion run with the body-frame drift-gated gauge variant (the body frame recomputed only when the per-head Gram drifts past a threshold) adds $4/5$ gauge grokking ($0.908$) against $0/5$ for vanilla Muon ($0.681$). Across the two gauge variants the gauge groks $10$ of $11$ seeds where vanilla Muon groks $0$ of $11$.

\begin{figure}[t]
\centering
  \includegraphics[width=0.82\linewidth]{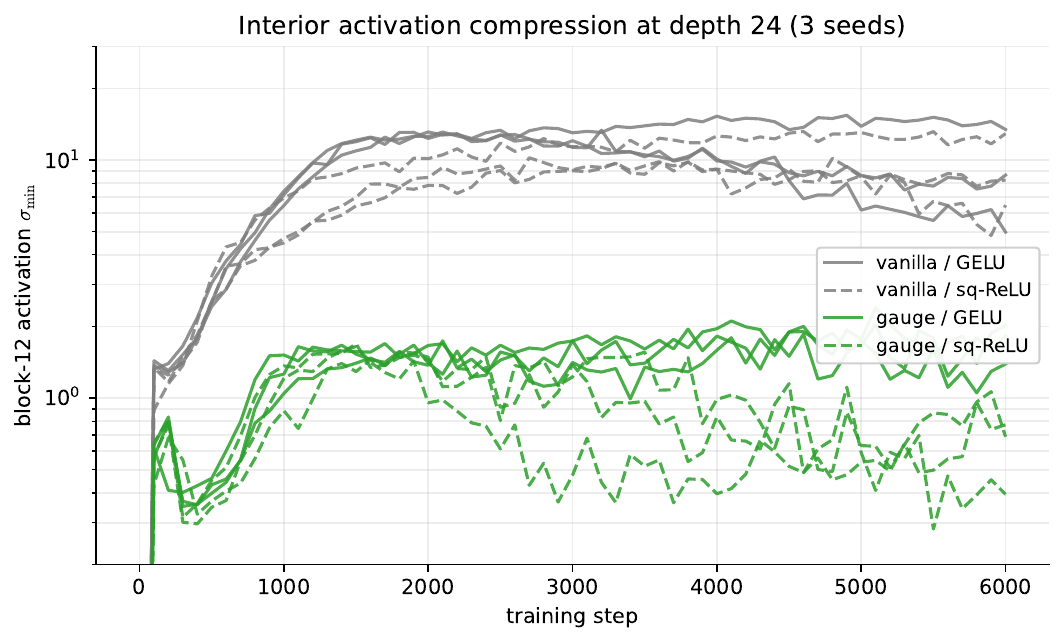}

\caption{Interior activation compression at depth 24, read at the smallest singular value of the block-12 activation (3 seeds per arm). The gauge settles the value near $1.6$ on GELU and $0.6$ on squared-ReLU; vanilla Muon settles in the range $6$ to $15$ on both. The gauge holds the interior representation against an absolute floor where the base optimiser lets it spread, on either activation.}
\label{fig:d24_block}
\end{figure}
 
\subsection{Matching the rotation gauge to the rotary architecture}\label{app:rope_ab}
\paragraph{The four-arm decomposition.} On a depth-8 transformer on the synthetic modular-arithmetic grokking task (the testbed of \S\ref{app:grok_setup}, here at depth 8 with rotary position embeddings and base Muon, weight decay $1.0$, 5000 steps, total batch 16384, sequence length 128, 3 seeds), the per-layer true-MC Fisher read at $n / d \ge 100$, four arms peel the gauge apart one piece at a time: vanilla Muon; a scaled-polar orthogonaliser with no gauge; the full-$O(d_\mathrm{head})$ QK rotation gauge; and the RoPE-torus $SO(2)^{d_\mathrm{head}/2}$ QK gauge. The two gauge arms are identical (scaled-polar orthogonaliser, drift-gated, QK and VO gauges) apart from the QK gauge mode, which isolates the matched rotation against the generic one with everything else held fixed.

\paragraph{Accuracy rises monotonically along the decomposition.} Final accuracy and grok reliability climb at each step: vanilla Muon $0.867$ (1 of 3 seeds grok), no-gauge scaled-polar $0.937$ (2 of 3), full-rotation gauge $0.955$ (3 of 3), RoPE-torus gauge $0.977$ (3 of 3), with cross-seed variance halved at the torus arm. The matched-against-generic isolation lands on the last step, $0.977$ against $0.955$ between two arms that differ only in the QK rotation mode.

\paragraph{The gauge reaches a less degenerate minimum.} The torus gauge settles at a less degenerate output stratum, the true-MC $\lambda_{\min}(G)$ rising from vanilla's $4.3 \times 10^{-12}$ to the gauge arms' $5$--$8 \times 10^{-11}$, while compressing the activation singular value from $\sigma_{\min} = 2.27$ at vanilla to $\sim\!0.4$ at the gauge.

\begin{figure}[t]
\centering
  \includegraphics[width=0.98\linewidth]{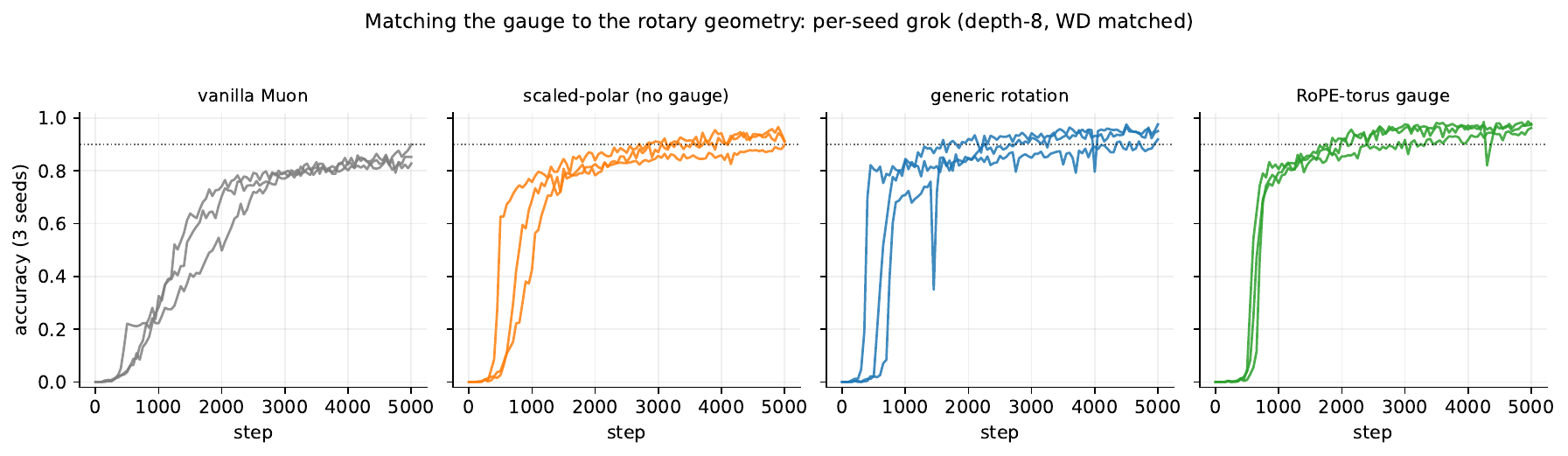}

\caption{Per-seed grok trajectories for the four-arm rotation decomposition (depth-8 RoPE, weight decay $1.0$ across all arms, 3 seeds). Accuracy climbs and the seed-to-seed spread tightens along the decomposition vanilla Muon $\to$ scaled-polar $\to$ generic $O(d_\mathrm{head})$ rotation $\to$ RoPE-torus rotation. The torus arm, matched to the symmetry the rotary embedding actually leaves, groks tightest.}
\label{fig:rope_seeds}
\end{figure}
 
\subsection{Curvature of the converged minimum}\label{app:curvature}
\paragraph{Three arms on the depth-8 bench.} The depth-8 arithmetic grok bench (the testbed of \S\ref{app:grok_setup} at depth 8 with rotary position embeddings and base Muon, weight decay $1.0$, 5000 steps, 3 seeds), reading the residual-free true-MC Fisher at $n / d = 100$, runs three arms: plain Muon, the RoPE-torus QK/VO rotation gauge, and a dead-direction-masking Muon variant that controls for whether any dead-direction handling, gauge or not, buys the curvature gain.

\paragraph{The gauge settles at a less degenerate minimum.} The RoPE-torus gauge groks $3/3$ at accuracy $0.967$ (validation $2.638$ bits per byte) and reaches a less degenerate output stratum, true-MC $\lambda_{\min}(G) = 5.8 \times 10^{-11}$, against plain Muon's $0/3$ grok ($0.848$, validation $3.278$, $\lambda_{\min}(G) = 4.6 \times 10^{-12}$), with the activation singular value dropping to $0.489$ from Muon's $2.19$.

\paragraph{Masking dead directions does not reproduce the gain.} The dead-direction-masking variant washes out to the Muon baseline ($0.880$, $1/3$ grok); its mask fired only rarely over training, so masking dead directions on its own recovers neither the grokking nor the curvature. The curvature gain is specific to the gauge.

\begin{figure}[t]
\centering
  \includegraphics[width=0.70\linewidth]{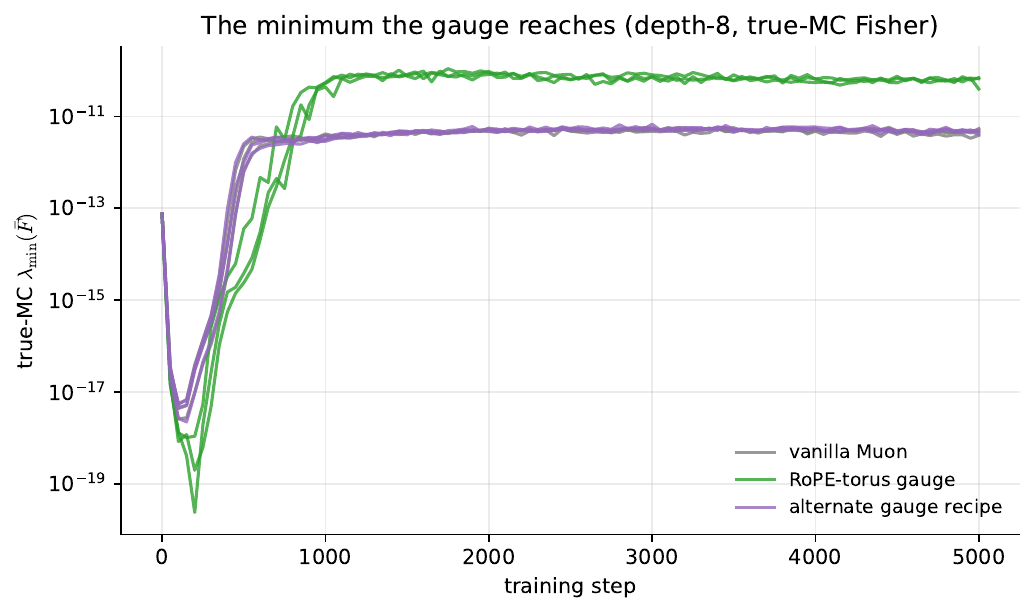}

\caption{The true-MC Fisher smallest eigenvalue over training on the depth-8 grok bench (3 seeds). The RoPE-torus gauge pulls $\lambda_{\min}(\bar\fisher)$ off the singular floor that plain Muon stays pinned to. A second gauge recipe (\texttt{v5\_1}) tracks the Muon floor, so the gain is specific to the recipe and not to gauge handling in general.}
\label{fig:curvature_approach}
\end{figure}
 
\subsection{Separating the gauge projection from the weight decay it bundles}\label{app:wd}
\paragraph{Two settings, one control.} An equivariant update carries weight decay along its horizontal frame, so a gauge arm applies more effective shrinkage than a vanilla optimizer at the same nominal weight decay. To separate the projection from that bundled decay, these settings use the testbed of \S\ref{app:grok_setup} (here on $(a + b) \bmod 113$ at training fraction $0.3$, 25000 steps, 5 seeds) with the QK/VO rotation gauge over both Muon and Adam bases, swept across $d_\mathrm{model} \in \{128, 64, 32\}$, and hand a vanilla baseline the same total weight decay the gauge effectively applies.

\paragraph{Grokking speed traces to the weight decay.} A vanilla Muon given double the attention weight decay, matching the gauge arm's total attention shrinkage, groks at step $4120$, against the gauge's $4160$ and an unmatched Muon's $5080$. The matched control reaches the gauge's speed, so the Muon-base acceleration comes from the bundled weight decay.

\paragraph{The cleaner dead basis comes from the projection.} The same match leaves two effects untouched. The gauge's MLP dead subspace aligns with the coordinate axes at $0.889$ against the matched Muon's $0.718$ ($+0.171$), where the \emph{axis alignment} is the mean over the dead directions of the squared weight each places on its single largest coordinate ($1$ when every dead direction is a coordinate axis), and its converged joint-norm weight spectrum shows a separated cluster (a spectral gap above a fixed threshold) in $0.8$ of seeds against $0.4$. The advantage holds at the task-appropriate $d_{128}$ (Adam base $+0.072$, Muon base $+0.179$), weakens at $d_{64}$, and turns noisy only at the $d_{32}$ near-cliff floor where grokking itself fails, tracking the width the task needs to grok.

\paragraph{So does the reliability.} On equal-weight-decay grokking runs, weight decay $1.0$ across every arm, vanilla Muon groks $0$ to $1$ of $3$ seeds while every gauge arm groks $3/3$, at $0.07$ to $0.11$ higher final accuracy. Holding the weight decay equal removes the speed gap and leaves a reliability gap the projection alone explains.

\begin{figure}[t]
\centering
  \includegraphics[width=0.70\linewidth]{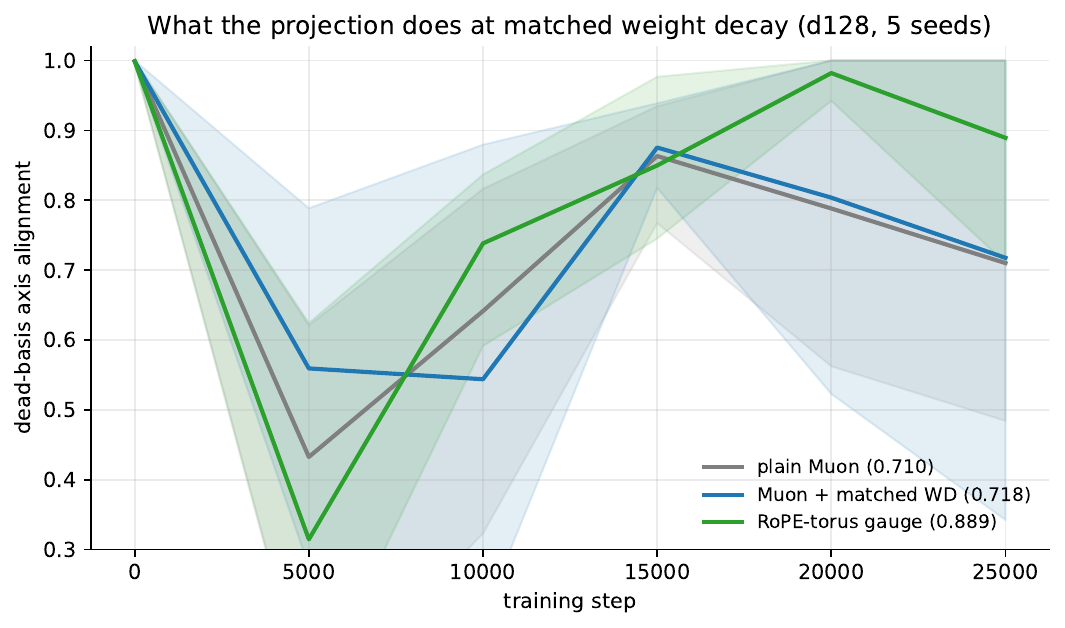}

\caption{Dead-basis axis alignment over training at matched weight decay ($d_{128}$, Muon base, 5 seeds). The gauge aligns the dead subspace with the coordinate axes to a final mean of $0.889$, above the weight-decay-matched Muon at $0.718$ and plain Muon at $0.710$. Matching the weight decay closes the grokking-speed gap; the projection drives the sharper axis alignment.}
\label{fig:overparam_basis}
\end{figure}
 
\subsection{Compression on a from-scratch vision transformer}\label{app:vit}
\paragraph{Testbed and arms.} A from-scratch vision transformer ($d_\mathrm{model} = 256$, depth 6, 8 heads, MLP width 1024 over-parametrised so weight decay prunes the spare units, patch 16, $112 \times 112$ images, 100 classes drawn from ImageNet-1k) trains into the compression regime over 50000 steps at batch 256 on an AdamW base. The DDC arm replaces AdamW with \eqadam{} plus the ViT multiplicative gauges (squared-ReLU rescale, LayerNorm scale), everything else matched to the vanilla arm at weight decay $0.1$, over 3 seeds.

\paragraph{Lower loss at matched weight decay.} At weight decay $0.1$ the gauge reaches a lower validation loss, $1.712 \pm 0.025$ against $2.116 \pm 0.028$ ($-0.40$ nats), with validation accuracy tied ($0.556$ against $0.548$).

\paragraph{Harder compression at the same weight decay.} The gauge compresses far harder than the matched vanilla arm: dead-ReLU count $\sim$3850 against $\sim$11, activation $\sigma_{\min}$ $4.24$ against $16.09$, and gate-dead units in the hundreds against zero. The vanilla arm needs weight decay $0.2$, twice the gauge's, to compress at all.

\paragraph{Which compression metrics are robust.} The dead-ReLU count and $\sigma_{\min}$ hold up across seeds; the total dead-unit count is seed-variable ($295 / 503 / 53$), so the read rests on the former two.

\begin{figure}[t]
\centering
  \includegraphics[width=0.94\linewidth]{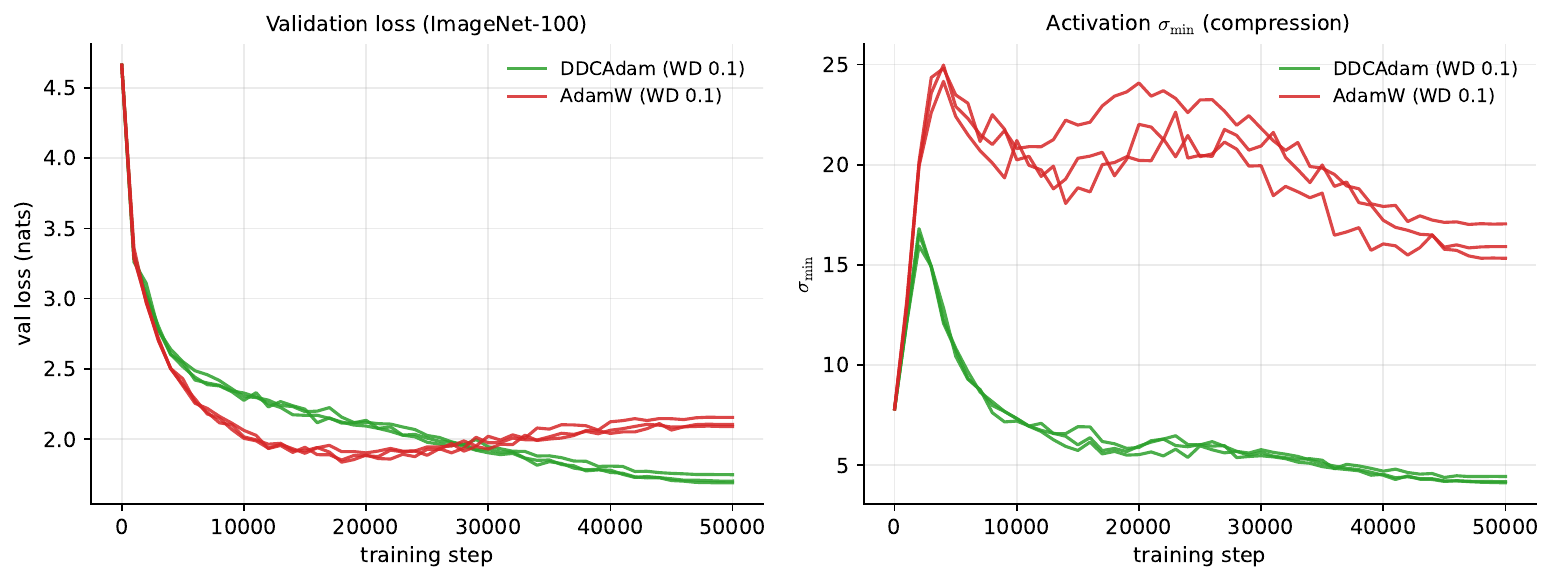}

\caption{ViT compression dynamics over training (ImageNet-100 from scratch, matched weight decay $0.1$, 3 seeds). Left: validation loss; \eqadam{} with the ViT multiplicative gauges reaches a lower minimum and resists the over-training climb AdamW falls into. Right: the activation $\sigma_{\min}$; the gauge compresses it far harder at the same weight decay, where AdamW leaves the spare capacity intact.}
\label{fig:vit_curves}
\end{figure}
 
\subsection{The grokking-transformer testbed}\label{app:grok_setup}
The mechanism studies below share one small, fully controlled network, the single-block residual transformer of \citet{NandaChanLieberum23}. A token embedding (vocabulary $p + 1 = 114$, $d_\mathrm{model} = 128$) and a length-2 positional embedding feed one attention block (4 heads, $d_\mathrm{head} = 32$) and then one feed-forward block $W_2\,\phi(W_1 x + b_1) + b_2$ of hidden width $4 d_\mathrm{model} = 512$ with activation $\phi$, each block preceded by a LayerNorm carrying a per-channel scale $\gamma$ and a bias, read out by an unembedding $\mathtt{Linear}(d_\mathrm{model}, p, \mathrm{bias}{=}\mathrm{False})$. The rows of $W_1$ index the hidden units: row $j$ holds the input weights of unit $j$, and the columns of $W_1$ index the $d_\mathrm{model}$ residual channels the LayerNorm-scale gauge acts on. The task is to predict $(a + b) \bmod p$ with $p = 113$ from the token sequence $[a, b]$, on a fixed $30\%$ training split, where the network memorises for thousands of steps before grokking. The rotation-gauge cohorts (\S\ref{app:rope_ab}, \S\ref{app:curvature}, \S\ref{app:d24}) keep this modular-addition task and grokking protocol but run deeper RoPE transformers (8 or 24 blocks) at a fixed sequence length of $128$; a reference there to ``the testbed of \S\ref{app:grok_setup}'' means that shared task and protocol, not the single-block length-2 network the abelian-gauge mechanism studies use.

Two optimizers are compared throughout. \eqadam{} carries the full gauge set of \S\ref{ssec:gauges}: the per-head QK and VO rotation gauges on the attention block, the readout row-shift on the unembedding, the LayerNorm-scale gauge $(\gamma, W_1)$ on the feed-forward input, and the per-channel ReLU-rescale gauge $(W_1, W_2)$ on the feed-forward block, with the rotation gauges using a body-frame second moment, the abelian gauges (LayerNorm-scale, rescale) each using a per-channel scalar second moment on their joint-scale coordinate, a frozen gauge-mode coordinate, and a per-coordinate second moment on the within-column tangential of the paired weight, and weight decay applied once through each gauge's own coordinate so the decay matches plain AdamW by construction. The body-frame second moment is a rotation-gauge knob only; the feed-forward gauges take no such knob. The matched baseline is that plain AdamW at the same learning rate and weight decay with no gauge. Sections that change the gauge set, the base optimizer, the activation, or a hyperparameter state the change against this pair.

Each section reads from dense checkpoints saved every 400 steps. The standard reads are the residual-free true-MC Fisher at $n / d \ge 100$ for the curvature spectra; the activation effective rank $\exp\!\left(-\sum_i p_i \log p_i\right)$ with $p_i = \sigma_i / \sum_j \sigma_j$, the entropy of the singular-value spectrum $\{\sigma_i\}$ of the $n_{\rm inputs} \times d$ node-activation matrix, together with that matrix's smallest singular value $\sigma_{\min}$, both read at each residual-stream node; the dead-ReLU census, which calls a hidden unit dead when its pre-activation $W_1\,\mathrm{LN}(h; \gamma) + b_1$ never crosses zero over the training inputs; the per-row and per-column Frobenius norms of $W_1$; and, where the mechanism needs the activations, the feed-forward pre-activations from a forward pass on the full input set. Seeds are 42, 142, 242 unless a section states otherwise. Figure~\ref{fig:grok_curve} shows the memorise-then-grok transition this testbed produces.

\begin{figure}[t]
\centering
  \includegraphics[width=0.82\linewidth]{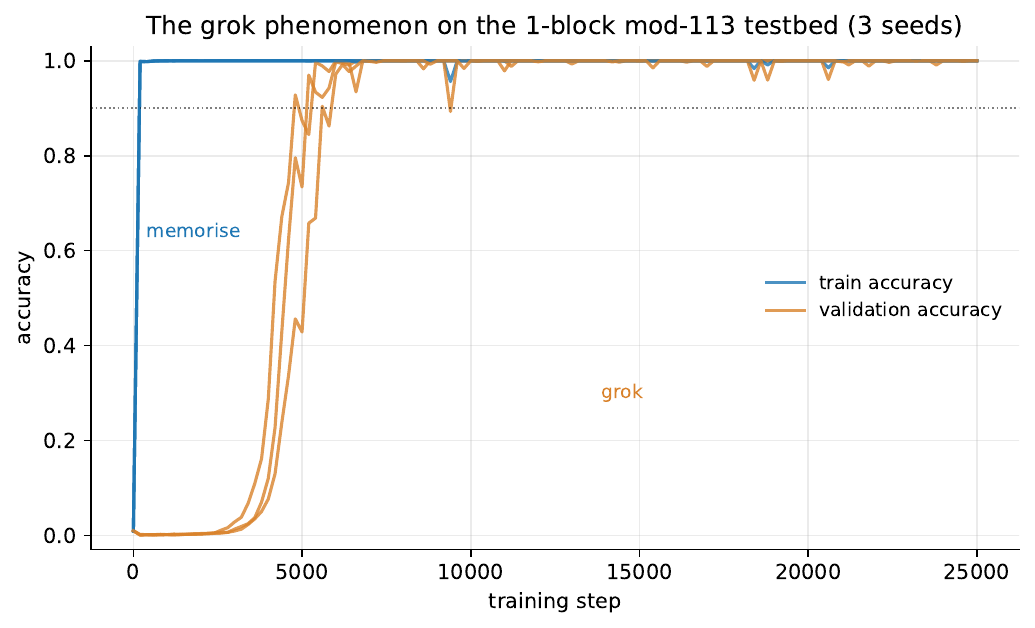}

\caption{The grok phenomenon on the one-block mod-113 testbed (3 seeds). Train accuracy memorises by step 200 while validation accuracy stays at chance for thousands of steps, then groks between steps 4800 and 5600. This is the transition every mechanism study in this appendix reads from.}
\label{fig:grok_curve}
\end{figure}
 
\subsection{Per-optimizer baseline tuning}\label{app:baseline_sweep}
A cross-optimizer grok-speed comparison is fair only when each base optimizer runs at its best, so we sweep learning rate against weight decay for AdamW, Muon, and SGD on this testbed (relu$^2$; AdamW, Muon, and SGD over 3 seeds (42, 142, 242); AdamW 15000 steps sampled every 100, Muon 5000 sampled every 25 for its faster grok, SGD 40000), reading the median step over the seeds at which validation accuracy first reaches $0.9$ (Figure~\ref{fig:baseline_sweep}). AdamW groks across the upper grid, fastest at learning rate $3\times10^{-3}$, weight decay $2$ (median 650 steps, two seeds of three), and fails in the low-learning-rate corner; its learning-rate $10^{-3}$, weight-decay $2$ cell groks in a median 2200. Muon, run as textbook Newton-Schulz-5, groks at every grid point, fastest at learning rate $0.04$ (875 steps). Plain SGD groks across a low-weight-decay band: at weight decay $10^{-3}$ and $2.5\times10^{-3}$ every learning rate groks three seeds of three, fastest at learning rate $0.3$ (median 800 steps at weight decay $2.5\times10^{-3}$, a tighter 1000 at $10^{-3}$); weight decay $5\times10^{-3}$ groks only at the two lower learning rates, the band narrowing as the rising learning rate pushes the effective shrinkage past SGD's grokking ceiling. SGD's grid sits two orders of magnitude below AdamW's because that shrinkage scales as learning rate times weight decay and SGD's raw-gradient steps run at a $10$ to $100$ times higher learning rate; at this small weight decay SGD groks modular addition readily, at a tuned speed within the range of the Adam and Muon bases. Each base and its DDC arm share one cohort at matched learning rate and weight decay, so the gauge reads directly against its tuned base (Figure~\ref{fig:gauge_sweep}): DDCMuon tracks Muon cell for cell (900 against 875 at the fast corner, mixed within 125 steps elsewhere), and DDCAdam runs faster than AdamW in the high-weight-decay column at low-to-mid learning rate (1900 against 2200 at learning rate $10^{-3}$, weight decay $2$; 6800 against 8100 at learning rate $3\times10^{-4}$, weight decay $2$), groks in two seeds of three the learning-rate $3\times10^{-4}$, weight-decay $1$ cell where AdamW manages one, and trails AdamW by at most 200 steps along the top learning-rate row (700 against 650 at the fast corner).

\begin{figure}[t]
\centering
  \includegraphics[width=0.92\linewidth]{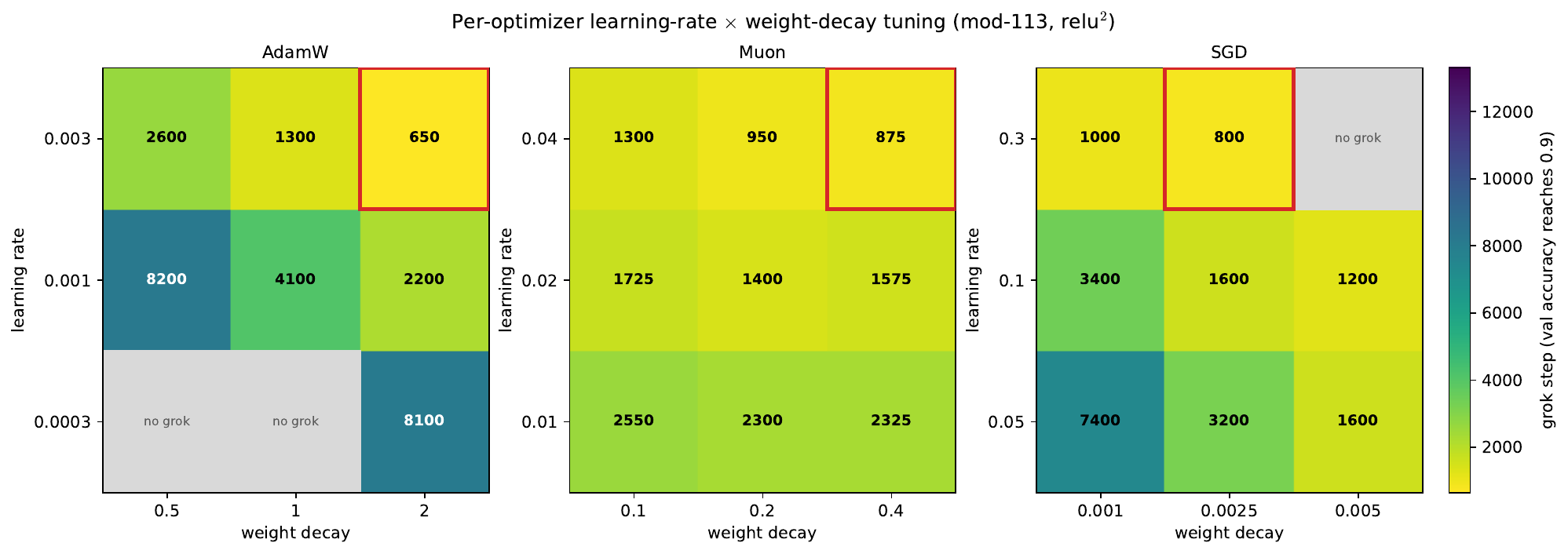}

\caption{Per-optimizer learning-rate by weight-decay tuning on the mod-113 grokking testbed (relu$^2$; AdamW, Muon, and SGD over 3 seeds; AdamW 15000 steps, Muon 5000, SGD 40000, with Muon sampled every 25 steps to resolve its faster grok). Muon is textbook Newton-Schulz-5. Colour is the median over the seeds of the step at which validation accuracy first reaches $0.9$ (lower is faster); grey cells grok in fewer than half the seeds. AdamW and Muon both grok across most of their grid and accelerate toward the high-learning-rate, high-weight-decay corner; SGD groks across its low-weight-decay band (weight decay $\le 2.5\times10^{-3}$ at every learning rate), fastest at learning rate $0.3$, weight decay $2.5\times10^{-3}$; the band narrows as the learning rate rises. SGD's weight-decay grid sits about $100\times$ below the others because effective decay scales as learning rate times weight decay and SGD runs at a far higher learning rate. At high learning rate SGD settles its validation accuracy right on the $0.9$ grok line and keeps oscillating, so its per-seed grok steps are noisier than AdamW's or Muon's (whose plateaus sit well above $0.9$) and the per-cell median is the stable read. The red box marks the fastest-grokking cell in each panel. The matched-weight-decay grok comparisons in the main text read each base in its tuned region.}
\label{fig:baseline_sweep}
\end{figure}
 
\begin{figure}[t]
\centering
  \includegraphics[width=0.92\linewidth]{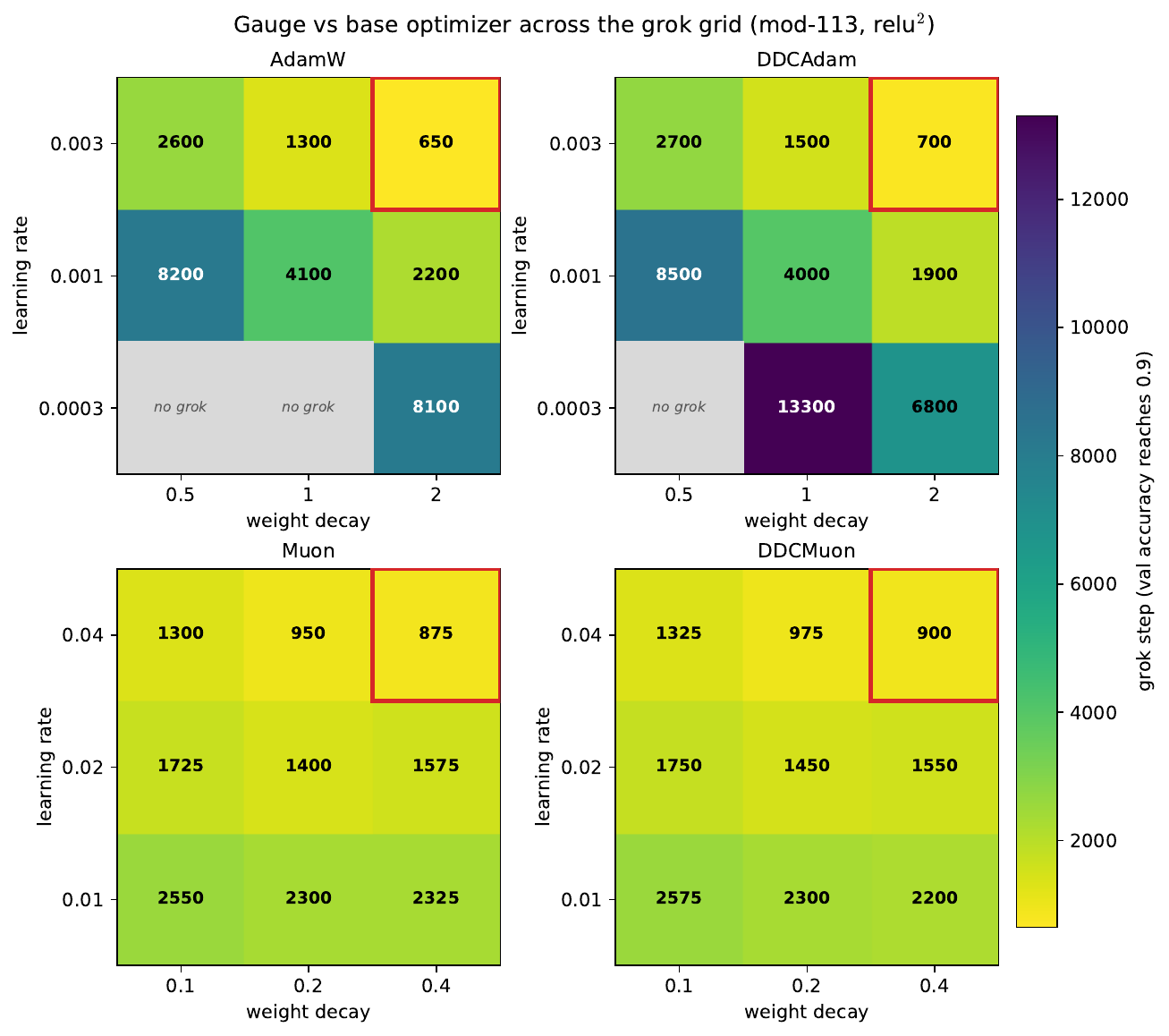}

\caption{Gauge versus base optimizer across the grok grid (mod-113, relu$^2$, seeds 42/142/242; Adam grid 15000 steps, Muon grid 5000 sampled every 25), in matched pairs at matched nominal weight decay: AdamW and DDCAdam on the base-Adam grid, textbook Newton-Schulz-5 Muon and DDCMuon on a shared Muon grid. Each pair is one cohort with the vanilla and gauge arms run together. Colour is the median over the three seeds of the step at which validation accuracy first reaches $0.9$ (lower is faster); a cell that groks in fewer than half the seeds is grey, and the red box marks the fastest-grokking cell. Each gauge groks over the same region as its tuned base. DDCMuon tracks Muon cell for cell across the whole grid (learning rate $0.04$, weight decay $0.4$: 900 against 875, mixed within 125 steps elsewhere). DDCAdam runs faster than AdamW in the high-weight-decay column at low-to-mid learning rate (1900 against 2200 at learning rate $10^{-3}$, weight decay $2$; 6800 against 8100 at learning rate $3\times10^{-4}$, weight decay $2$), groks the learning-rate $3\times10^{-4}$, weight-decay $1$ cell in two of three seeds where AdamW manages one, and trails AdamW by at most 200 steps along the top learning-rate row (700 against 650 at the fast corner). DDCMuon runs the scaled-polar orthogonaliser against textbook Muon's Newton-Schulz-5; \S\ref{ssec:muon_boundary} isolates that orthogonaliser from the gauge. The gauge groks where its tuned base groks; its gains, reliability and a less degenerate minimum, show in the main text.}
\label{fig:gauge_sweep}
\end{figure}
 
\subsection{Tracing the effective-rank change to the feed-forward prune}\label{app:mlp_pruning}
This study uses the testbed of \S\ref{app:grok_setup} with squared-ReLU activation, learning rate $10^{-3}$, weight decay $2.0$, 10000 steps, seed 42, comparing \eqadam{} with the full gauge set against the matched AdamW baseline.

\begin{figure}[t]
\centering
  \includegraphics[width=0.78\linewidth]{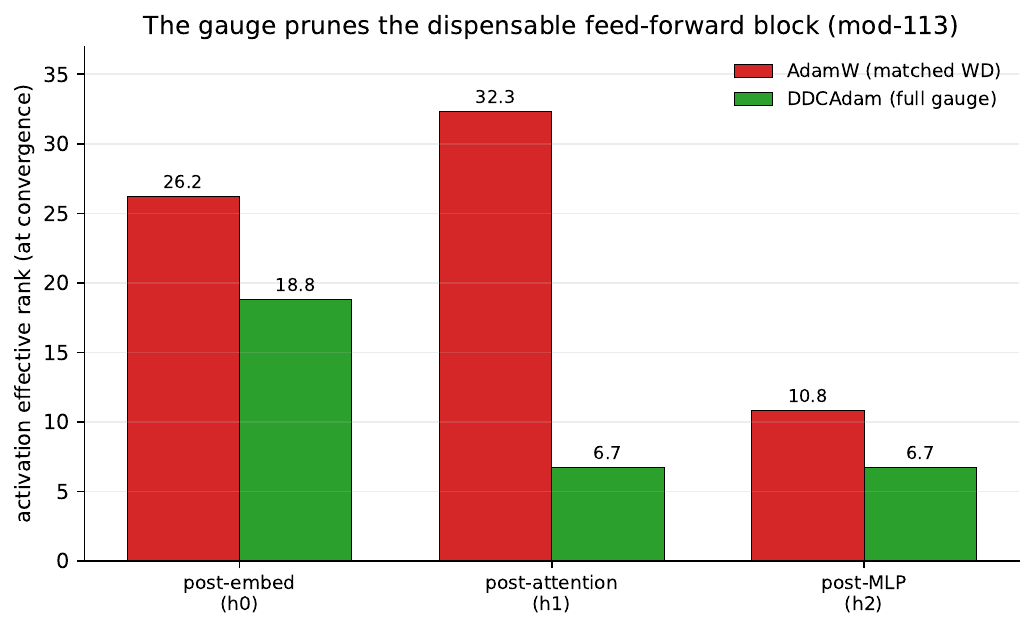}

\caption{Activation effective rank at each residual node on $(a+b)\bmod 113$, the full gauge against the matched AdamW (squared-ReLU, weight decay $2.0$, 10000 steps, 3-seed convergence means). The gauge collapses the dispensable feed-forward block: the post-attention rank drops from $32.3$ to $6.7$, and the post-MLP rank from $10.8$ to $6.7$. The activation rank is a shadow of the prune the decoupled weight decay performs on the block's undefended gauge-invariant scale.}
\label{fig:mlp_prune}
\end{figure}
 
\paragraph{The observable is a shadow of the prune.} At convergence the full gauge leaves a lower activation effective rank than the matched AdamW at every residual node ($6.7$ against $32.3$ at the post-attention node, $6.7$ against $10.8$ at the post-MLP node; Figure~\ref{fig:mlp_prune}). This is downstream of the gauge collapsing the dispensable feed-forward block: on $(a+b)\bmod p$ the circuit routes around the block, so its gauge-invariant scale goes undefended and the decoupled weight decay collapses it (\S\ref{app:radial}), pruning the block and lowering the activation rank. The sign confirms it is the block and not a rank effect of the optimizer: on a plain digit-sequence addition task, where the network memorises and forms no clean circuit, the same gauge sits \emph{above} AdamW in effective rank. How the prune distributes across the block's rows, and how much of that the early task gradient sets versus the gauge, we do not separate cleanly on the log-radial runs available. Whether the prune persists when the block is load-bearing is the question \S\ref{app:radial} takes up.

\subsection{Keeping or shedding a load-bearing block}\label{app:radial}
\paragraph{Testbed and arms.} Answering the open question of \S\ref{app:mlp_pruning} needs a network where the feed-forward block cannot be routed around. A one-hidden-layer block with an internal LayerNorm, $f(x) = W_2\,\mathrm{LN}(\mathrm{ReLU}(W_1 x); \gamma)$ at hidden width $512$, learns sparse parity (parity of the first $3$ of $10$ bits, full enumeration), where zeroing the hidden layer drops the network to chance ($0.49$), so the block is the whole non-linear path. The LayerNorm-scale gauge binds $(\gamma, W_2)$; the gate clears it (equivariance exact, AdamW leaks the mode at $0.75$, \eqadam{} holds it at $6.5 \times 10^{-2}$) and rejects the ReLU rescale, which the internal LayerNorm breaks, leaving LayerNorm-scale as the only valid abelian gauge here. Six arms run at weight decay $\{1, 2, 5\}$ over 3 seeds: the no-gauge baselines AdamW and Muon, and \eqadam{} and \textsc{DDCMuon} each on the log and the linear radial coordinate.
\paragraph{The radial coordinate decides whether the block survives.} At weight decay $2$ the two coordinates split cleanly and base-independently. The log coordinate collapses the block to chance on both bases, \eqadam{} to final accuracy $0.485$ at gain $\sim\!10^{-7}$ and \textsc{DDCMuon} to $0.485$ with $\sim\!500$ of $512$ units dead, three of three seeds each, where the no-gauge AdamW and Muon keep the block (final $1.000$). The linear coordinate keeps the block on both bases, final $1.000$ with the gain held near $0.2$ on the Adam base and $20$ on the Muon base, three of three. The collapse is weight-decay-dependent: it fires at weight decay $2$ and above, and a weight decay of $1$ holds the block through $6000$ steps.
\paragraph{The linear coordinate inherits the AdamW equilibrium.} The split is the difference between decaying a scale additively and multiplicatively. On the log of the joint scale $s = \rho_\ell\,\rho_{\ell+1}$ the decoupled weight decay is a constant subtraction every step, $\sqrt{q^2+1}\,\eta\lambda$ in log-space (here $q$ is the activation's positive-homogeneity degree, $q = 1$ for this ReLU block and $q = 2$ for squared-ReLU; $\eta$ the learning rate, $\lambda$ the weight-decay coefficient), which the Adam-normalised loss gradient, bounded by $\sim\!\eta$, cannot oppose once $\sqrt{q^2+1}\,\lambda > 1$, so the scale falls without limit. On the linear scale the same decay is multiplicative and balances the loss gradient at the equilibrium $s^\star = 1/((q^2+1)\lambda) = 1/(2\lambda)$ here, the per-neuron weight-decay fixed point of decoupled AdamW \citep{LoshchilovHutter19, vanLaarhoven17, Kosson24, Wan21}, specialised to the gauge-invariant joint scale. The construction runs the radial update on the linear scale by default; the log coordinate is the legacy variant that sheds the block.
\paragraph{A simpler fix keeps the block but leaks the gauge mode.} Keeping a LayerNorm scale alive under weight decay has a standard fix, reparametrising $\gamma$ as $1 + \delta$ and decaying $\delta$ to zero so the scale rests at $1$. It holds the block at final $1.000$ with $\gamma \approx 1$ and no gauge at all, but it leaves the gauge mode free. A head-to-head at weight decay $\{1, 2, 5\}$ over 3 seeds reads the gauge mode $r_i = \log|\gamma_i| - \log\|W_2[:, i]\|$, its drift from initialisation to convergence: the reparametrisation drifts it at $2.3$ to $4.4$ and a plain AdamW at $1.7$ to $4.9$, against the linear radial's $1.1$ to $4.9 \times 10^{-4}$, four orders of magnitude tighter. The linear radial is the only arm that keeps the block and freezes the mode together (Figure~\ref{fig:radial_quadrant}): the log radial freezes the mode but collapses the block, and the reparametrisation keeps the block but leaks the mode.
\paragraph{The collapse hides from the global gain.} On a Muon base the global feed-forward gain $\|W_1\|_F \|W_2\|_F$ does not read the collapse: the orthogonalised step refills $\|W_2\|$ along its own direction, holding the gain near $10$ while $\sim\!500$ units are dead. The collapse reads cleanly on the per-channel joint norm $|\gamma_i|\,\|W_2[:, i]\|$ against an absolute floor, the read this paper's compression measurements use; a median-relative census misses it because the collapse is near-uniform across channels.

\begin{figure}[t]
\centering
  \includegraphics[width=0.86\linewidth]{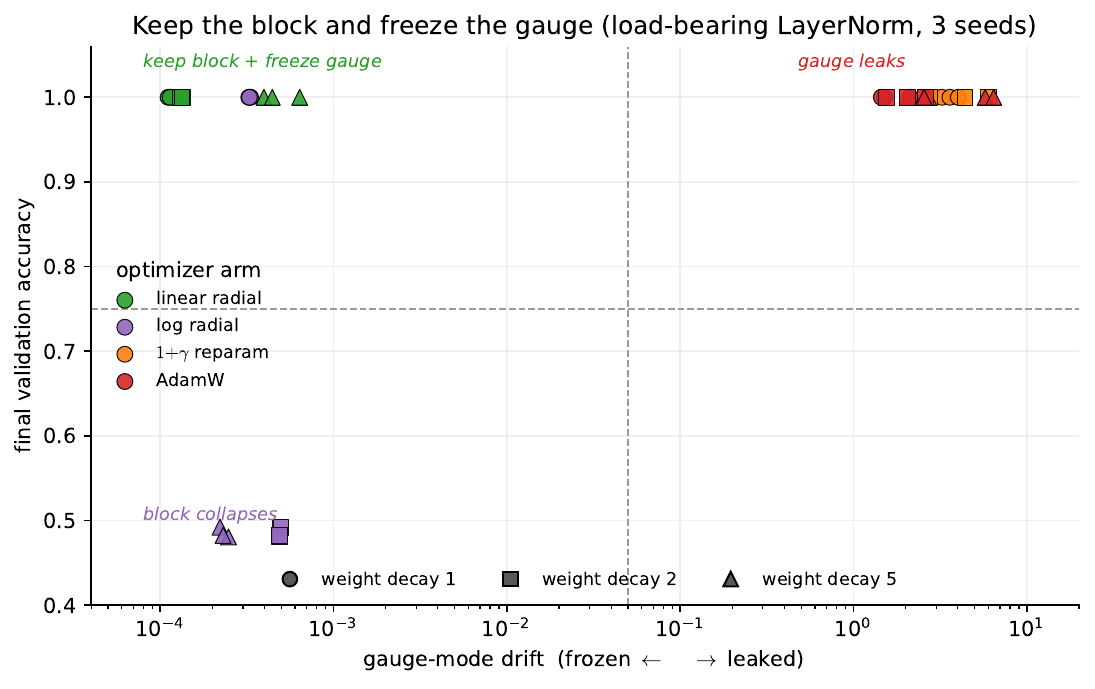}

\caption{Keeping the load-bearing block and freezing the gauge mode are separate properties, and the linear radial is the only arm with both. Each point is one (arm, weight decay, seed) run on sparse parity, placed by its gauge-mode drift (horizontal, log) and its final validation accuracy (vertical); marker shape is the weight decay. The linear radial holds the upper-left corner across all weight decays: drift $\sim\!10^{-4}$, accuracy $1.0$. The log radial freezes the mode equally well but collapses the block to chance at weight decay $\ge 2$. The $1{+}\gamma$ reparam and plain AdamW keep the block but leave the mode free, drifting it by order one.}
\label{fig:radial_quadrant}
\end{figure}
 
\subsection{$G$-equivariance self-tests}\label{app:equiv_tests} For each of the five constructions (CE-shift bias, CE-shift row-shift, ReLU rescale, LN scale, chained ReLU rescale), the test initialises two parameter copies related by a randomly-chosen $G$-action, feeds bit-identical $G$-related gradients to two \eqadam{} instances for 50 steps, and verifies the parameters remain $G$-related to within fp32 precision. The per-step relative deviation from the gauge image stays at machine precision and below $10^{-6}$ across all 50 steps for every construction (Figure~\ref{fig:equiv_selftest}).

\begin{figure}[t]
\centering
  \includegraphics[width=0.82\linewidth]{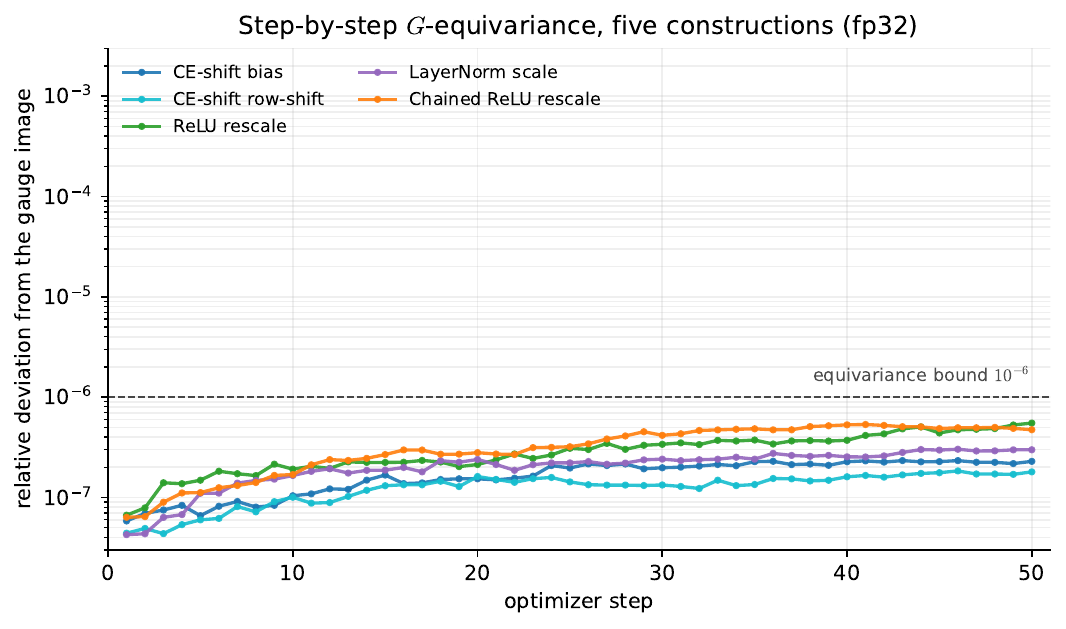}

\caption{Step-by-step $G$-equivariance for the five gauge constructions. Each line is the relative deviation of a gauge-shifted parameter copy from the gauge image of the reference copy, recorded at every step of 50 \eqadam{} steps fed $G$-related gradients (fp32). All five stay at machine precision and accumulate only the expected fp32 round-off, holding below the $10^{-6}$ bound throughout.}
\label{fig:equiv_selftest}
\end{figure}

\end{document}